\documentclass[12pt,onecolumn]{IEEEtran}

\usepackage[dvips]{color}
\usepackage{cite}
\usepackage{amssymb}
\usepackage{amsmath}
\usepackage{graphicx}
\usepackage{algorithmic}
\usepackage{boxedminipage}
\newtheorem{theorem}{Theorem}
\newtheorem{lemma}{Lemma}

\newtheorem{corollary}{Corollary}

\newtheorem{algorithm}{Algorithm}

\newtheorem{remark}{Remark}

\newcommand{\bydef}{\stackrel{\bigtriangleup}{=}}

\begin{document}
\title{Principal Component Analysis with Contaminated Data: The High Dimensional Case}
\author{Huan~Xu\IEEEmembership{},~Constantine~Caramanis,~\IEEEmembership{Member},~and~
Shie~Mannor,~\IEEEmembership{Senior Member} \thanks{Preliminary
versions of these results have appeared in part, in The Proceedings
of the 46th Annual Allerton Conference on Control, Communication,
and Computing, and at the 23rd international Conference on Learning
Theory (COLT).}
 \thanks{H. Xu and C. Caramanis are with the Department of Electrical and Computer
Engineering, The University of Texas at Austin, Austin, TX 78712 USA
email: (huan.xu@mail.utexas.edu; caramanis@mail.utexas.edu).}
\thanks{S. Mannor is with the Department of Electrical Engineering, Technion, Israel.
 email:
(shie@ee.technion.ac.il)}
 \markboth{~Oct.~2008}{Shell
\MakeLowercase{\textit{et al.}}: Bare Demo of IEEEtran.cls for
Journals}} \maketitle

\maketitle

\begin{abstract} We consider the dimensionality-reduction problem
(finding a subspace approximation of observed data) for contaminated
data in the high dimensional regime, where the number of {\em
observations} is of the same magnitude as the number of {\em
variables} of each observation, and the data set contains some
(arbitrarily) corrupted observations. We propose a High-dimensional
Robust Principal Component Analysis (HR-PCA) algorithm that is
tractable, robust to contaminated points,  and easily kernelizable.
The resulting subspace has a bounded deviation from the desired one,
achieves maximal robustness -- a breakdown point of $50\%$ while all
existing algorithms have a breakdown point of {\em zero}, and unlike
ordinary PCA algorithms, achieves optimality in the limit case where
the proportion of corrupted points goes to zero.
\end{abstract}

\begin{keywords}
Statistical Learning, Dimension Reduction, Principal Component
Analysis, Robustness, Outlier
\end{keywords}

\section{Introduction}

The analysis of very high dimensional data -- data sets where the dimensionality of each observation is comparable to or even larger than the number of observations -- has drawn increasing attention
in the last few decades \cite{Donoho00,Johnstone01}. For example, observations on individual instances can be
curves, spectra, images or even movies, where a single observation
has dimensionality ranging from thousands to billions. Practical
high dimensional data examples include DNA Microarray data,
financial data, climate data, web search engine, and consumer data.
In addition, the nowadays standard ``Kernel Trick''
\cite{Scholkopf02}, a pre-processing routine which non-linearly maps
the observations into a (possibly infinite dimensional) Hilbert space,
transforms virtually every data set to a high dimensional one.
Efforts of extending traditional statistical tools (designed for the
 low dimensional case) into this high-dimensional
regime are generally unsuccessful. This fact has stimulated research on
formulating fresh data-analysis techniques able to cope with such a
``dimensionality explosion.''

Principal Component Analysis (PCA) is perhaps one of the most widely
used statistical techniques for dimensionality reduction. Work on
PCA dates back  as early as \cite{Pearson1901}, and has become one
of the most important techniques for data compression and feature
extraction. It is widely used in statistical data analysis,
communication theory, pattern recognition, and image processing
\cite{Jolliffe86}. The standard PCA algorithm constructs the optimal
(in a least-square sense) subspace approximation to observations by
computing the eigenvectors or Principal Components (PCs) of the
sample covariance or correlation matrix. Its broad application can
be attributed to primarily two features: its success in the
classical regime for recovering a low-dimensional subspace even in
the presence of noise, and also the existence of efficient
algorithms for computation. Indeed, PCA is nominally a non-convex
problem, which we can, nevertheless, solve, thanks to the magic of
the SVD which allows us to {\it maximize} a convex function. It is
well-known, however, that precisely because of the quadratic error
criterion, standard PCA is exceptionally fragile, and the quality of
its output can suffer dramatically in the face of only a few
(even a vanishingly small fraction) grossly corrupted
points. Such non-probabilistic errors may be present due to data
corruption stemming from sensor failures, malicious tampering, or
other reasons. Attempts to use other error functions growing more
slowly than the quadratic that might be more robust to outliers,
result in non-convex (and intractable) optimization problems.

In this paper, we consider a high-dimensional  counterpart of
Principal Component Analysis (PCA) that is robust to the existence
of {\it arbitrarily corrupted} or contaminated data. We start with
the standard statistical setup: a low dimensional signal is
(linearly) mapped to a very high dimensional space, after which
point high-dimensional Gaussian noise is added, to produce points
that no longer lie on a low dimensional subspace. At this point, we
deviate from the standard setting in two important ways: (1) {\it a
constant fraction of the points are arbitrarily corrupted} in a
perhaps non-probabilistic manner. We emphasize that these
``outliers'' can be entirely arbitrary, rather than from the tails
of any particular distribution, e.g., the noise distribution; we
call the remaining points ``authentic''; (2) the
number of data points is of the same order as (or perhaps
considerably smaller than) the dimensionality. As we discuss below,
these two points confound (to the best of our knowledge) all
tractable existing Robust PCA algorithms.

A fundamental feature of the high dimensionality is that the noise
is large in some direction, with very high probability, and
therefore definitions of ``outliers'' from classical statistics are
of limited use in this setting. Another important property of this
setup is that the signal-to-noise ratio (SNR) can go to zero, as the
$\ell_2$ norm of the high-dimensional Gaussian noise scales as the
square root of the dimensionality. In the standard (i.e.,
low-dimensional case), a low SNR generally implies that the signal
cannot be recovered, even without any corrupted points.

\subsection*{The Main Result}
In this paper, we give a surprisingly optimistic message: contrary
to what one might expect given the brittle nature of classical PCA,
and in stark contrast to previous algorithms, it is possible to
recover such low SNR signals, in the high-dimensional regime, even
in the face of a {\it constant fraction of arbitrarily corrupted
data.} Moreover, we show that this can be accomplished with an
efficient (polynomial time) algorithm, which we call
High-Dimensional Robust PCA (HR-PCA). The algorithm we
propose here is tractable, provably robust to corrupted points, and
asymptotically optimal, recovering the {\it exact} low-dimensional
subspace when the number of corrupted points scales more slowly than
the number of ``authentic'' samples (i.e., when the fraction of corrupted
points tends to zero).
To the best of our knowledge, this is the only algorithm of this kind.
Moreover, it is easily kernelizable.

The proposed algorithm performs a PCA and a random removal
alternately. Therefore, in each iteration a candidate subspace is
found. The random removal process guarantees that with high
probability, one of candidate solutions found by the algorithm is
``close'' to the optimal one. Thus, comparing all solutions using a
(computational efficient) one-dimensional robust variance estimator
leads to a ``sufficiently good'' output. We will make this argument
rigorous in the following sections.

\subsection*{Organization and Notation}
The paper is organized as follows: In Section~\ref{sec:background}
we discuss past work and the reasons that classical robust PCA
algorithms fail to extend to the high dimensional regime. In
Section~\ref{sec.algorithm} we present the setup of the problem, and
the HR-PCA algorithm. We also provide finite sample and asymptotic
performance guarantees. Section~\ref{sec.kernelization} is devoted
to the kernelization of HR-PCA. The performance
guarantee are proved in Section~\ref{sec.proof}. We provide some
numerical experiment results in Section~\ref{sec.numerical}.
Some technical details in the derivation of the
performance guarantees are postponed to the appendix.

Capital letters and boldface letters are used to denote matrices and
vectors, respectively.  A $k \times k$ unit matrix is denoted by
$I_k$.  For $c\in \mathbb{R}$, $[c]^+\triangleq \max(0, c)$.We let
$\mathcal{B}_d \triangleq \{\mathbf{w}\in
\mathbb{R}^d|\|\mathbf{w}\|\leq 1\}$, and $\mathcal{S}_d$ be its
boundary. We use a subscript $(\cdot)$ to represent order statistics
of a random variable. For example, let $v_1,\cdots, v_n \in
\mathbb{R}$. Then $v_{(1)}, \cdots, v_{(n)}$ is a permutation of
$v_1,\cdots,v_n$, in a non-decreasing order.

\section{Relation to Past Work}
\label{sec:background}
In this section, we discuss past work and the reasons that classical robust
PCA algorithms fail to extend to the high dimensional regime.




Much previous robust PCA work focuses on the traditional robustness
measurement known as the  ``breakdown point'' \cite{Huber81}, i.e.,
the percentage of corrupted points that can make the output of the
algorithm {\em arbitrarily} bad. To the best of our knowledge, no
other algorithm can handle {\it any constant fraction of outliers}
with a lower bound on the error in the high-dimensional regime. That
is, the best-known breakdown point for this problem is zero. We show
that the algorithm we provide has breakdown point of $50\%$, which
is the best possible for any algorithm. In addition to this, we
focus on providing explicit lower bounds on the performance, for all
corruption levels up to the breakdown point.
%

In the low-dimensional regime where the observations significantly
outnumber the variables of each observation, several robust PCA
algorithms have been proposed (e.g.,
\cite{DevlinGnanadesikanKettenring81,XuYuille95,YangWang99,CrousHaesbroeck00,TorreBlack01,TorreBlack03,CrouxFilzmoserOliveira07,Brubaker09}).
These algorithms can be roughly divided into two classes: (i)
performing a standard PCA on a robust estimation of the covariance
or correlation matrix; (ii) maximizing (over all unit-norm
$\mathbf{w}$) some $r(\mathbf{w})$ that is a robust estimate of the
variance of univariate data obtained by projecting the observations
onto direction $\mathbf{w}$. Both approaches encounter serious
difficulties when applied to high-dimensional data-sets:
\begin{itemize}
\item There are not enough observations to robustly estimate the
covariance or correlations matrix. For example, the widely-used MVE
estimator \cite{Rousseeuw85}, which treats the Minimum Volume
Ellipsoid that covers half of the observations as the covariance
estimation, is ill-posed in the high-dimensional case. Indeed, to
the best of our knowledge, the assumption that observations far
outnumber dimensionality seems crucial for those robust variance
estimators to achieve statistical consistency.
\item Algorithms that subsample the points, and in the spirit of
leave-one-out approaches, attempt in this way to compute the correct
principal components, also run into trouble. The constant fraction
of corrupted points means the sampling rate must be very low (in
particular, leave-one-out accomplishes nothing). But then, due to
the high dimensionality of the problem, principal components from
one sub-sample to the next, can vary greatly.
\item Unlike standard PCA that has a polynomial computation time, the maximization of
$r(\mathbf{w})$ is generally a non-convex problem, and becomes
extremely hard to solve or approximate as the dimensionality of
$\mathbf{w}$ increases. In fact, the number of the local maxima
grows so fast that it is effectively impossible to find a
sufficiently good solution using gradient-based algorithms with
random re-initialization.
\end{itemize}

We now discuss in greater detail three pitfalls some existing algorithms face in high dimensions.

\underline{Diminishing Breakdown Point}: The breakdown point measures the fraction of outliers
required to change the output of a statistics algorithm arbitrarily. If an algorithm's breakdown point has an inverse dependence on the dimensionality, then it is unsuitable in our regime. Many algorithms fall into this category. In \cite{Donoho82}, several covariance estimators including M-estimator \cite{Maronna76}, Convex Peeling \cite{Barnett76,Bebbington78}, Ellipsoidal Peeling \cite{Titteringto78,Helbling83}, Classical Outlier Rejection \cite{BarnettLewis78,David81}, Iterative Deletion \cite{DempsterGasko-Green81} and Iterative Trimming \cite{GnanadesikanKettenring72,DevlinGnanadesikanKettenring75} are all shown to have breakdown points upper-bounded by the inverse of the dimensionality, hence not useful in the regime of interest.

\underline{Noise Explosion}: As we define in greater detail below, the model we consider is the standard PCA setup: we observe samples $\mathbf{y} = A \mathbf{x} + \mathbf{n}$,
where $A$ is an $n \times d$ matrix, $\mathbf{n}\sim\mathcal{N}(0, I_m)$, and $n \approx m >> d$.
Thus, $n$ is the number of samples, $m$ the dimension, and $d$ the dimension of $\mathbf{x}$ and thus the number of principal components.
Let $\sigma_1$ denote the largest singular value of $A$. Then, $\mathbb{E}(\|\mathbf{n}\|_2)=\sqrt{m}$, (in fact, the magnitude
sharply concentrates around $\sqrt{m}$), while $\mathbb{E}
(\|A\mathbf{x}\|_2)=\sqrt{\mathrm{trace}(A^\top A)}\leq
\sqrt{d}\sigma_1$. Unless $\sigma_1$ grows very quickly
(namely, at least as fast as $\sqrt{m}$) the magnitude of the noise
quickly becomes the dominating component of each {\it authentic}
point we obtain. Because of this, several perhaps counter-intuitive
properties hold in this regime. First, any given authentic point is
with overwhelming probability very close to orthogonal to the signal
space (i.e., to the true principal components). Second, it is
possible for a constant fraction of corrupted points all with a
small Mahalanobis distance to significantly change the output of
PCA. Indeed, by aligning $\lambda n$ points of magnitude some
constant multiple of $\sigma_1$, it is easy to see that the output
of PCA can be strongly manipulated -- on the other hand, since the
noise magnitude is $\sqrt{m} \approx \sqrt{n}$ in a direction
perpendicular to the principal components, the Mahalanobis distance
of each corrupted point will be very small. Third, and similarly, it
is possible for a constant fraction of corrupted points all with
small Stahel-Donoho outlyingness to significantly change the output
of PCA. Stahel-Donoho outlyingness is defined as:
$$
u_i\triangleq \sup_{\|\mathbf{w}\|=1}
\frac{|\mathbf{w}^\top\mathbf{y}_i
-\mathrm{med}_j(\mathbf{w}^\top\mathbf{y}_j)|}{\mathrm{med}_k
|\mathbf{w}^\top\mathbf{y}_k
-\mathrm{med}_j(\mathbf{w}^\top\mathbf{y}_j)|}.
$$
To see that this can be small, consider the same setup as for the Mahalanobis example: small magnitude outliers, all aligned along one direction. Then the Stahel-Donoho outlyingness of such a corrupted point is $O(\sigma_1/\lambda)$. For a given authentic sample $\mathbf{y}_i$, take $\mathbf{v}=\mathbf{y}_i/\|\mathbf{y}_i\|$. On the projection of $\mathbf{v}$, all samples except $\mathbf{y}_i$ follow a Gaussian distribution with a variance roughly $1$, because $\mathbf{v}$ only depends on $\mathbf{y}_i$ (recall that $\mathbf{v}$ is nearly orthogonal to $A$). Hence the S-D outlyingness of a sample is of $\Theta(\sqrt{m})$, which is much larger than that of a corrupted point.

The Mahalanobis distance and the S-D outlyingness are extensively used in existing robust PCA algorithms. For example,
Classical Outlier Rejection, Iterative Deletion and various alternatives of Iterative Trimmings all use the Mahalanobis distance to identify possible outliers. Depth Trimming \cite{Donoho82} weights the contribution of observations based on their S-D outlyingness. More recently, the ROBPCA algorithm proposed in \cite{HubertRousseeuwBranden05} selects a subset of observations with least S-D outlyingness to compute the $d$-dimensional signal
space. Thus, in the high-dimensional case, these algorithms may run
into problems since neither Mahalanobis distance nor S-D outlyingness
are valid indicator of outliers. Indeed, as shown in the
simulations, the empirical performance of such algorithms can
be worse than standard PCA, because they remove the authentic
samples.

\underline{Algorithmic Tractability}: There are algorithms that do
not rely on Mahalanobis distance or S-D outlyingness, and have a
non-diminishing breakdown point, namely Minimum Volume Ellipsoid
(MVE), Minimum Covariance Determinant (MCD) \cite{Rousseeuw84} and
Projection-Pursuit \cite{LiChen85}. MVE finds the minimum volume
ellipsoid that covers a certain fraction of observations. MCD finds
a fraction of observations whose covariance matrix has a minimal
determinant. Projection Pursuit maximizes a certain robust
univariate variance estimator over all directions.

MCD and MVE are combinatorial, and hence (as far as we know)
computationally intractable as the size of the problem scales. More
difficult yet, MCD and MVE are ill-posed in the high-dimensional
setting where the number of points (roughly) equals the dimension,
since there exist infinitely many zero-volume (determinant)
ellipsoids satisfying the covering requirement. Nevertheless, we
note that such algorithms work well in the low-dimensional case, and
hence can potentially be used as a post-processing procedure of our
algorithm by projecting all observations to the output subspace to
fine tune the eigenvalues and eigenvectors we produce.

Maximizing a robust univariate variance estimator as in Projection
Pursuit, is also non-convex, and thus to the best of our knowledge,
computationally intractable. In \cite{CrouxRuiz-Gazen05}, the
authors propose a fast Projection-Pursuit algorithm, avoiding the
non-convex optimization problem of finding the optimal direction, by
only examining the directions of each sample. While this is suitable
in the classical regime, in the high-dimensional setting this
algorithm fails, since as discussed above, the direction of each
sample is almost orthogonal to the direction of true principal
components. Such an approach would therefore only be examining
candidate directions nearly orthogonal to the true maximizing
directions.

\underline{Low Rank Techniques}: Finally, we discuss the recent
 paper \cite{ZhouLiWrightCandesMa10}. In this work, the authors adapt
techniques from low-rank matrix approximation, and in particular,
results similar to the matrix decomposition results of
\cite{ChandrasekaranSanghaviParriloWillsky09}, in order to recover a
low-rank matrix $L_0$ from highly corrupted measurements $M = L_0 +
S_0$, where the noise term, $S_0$, is assumed to have a sparse
structure. This models the scenario where we have perfect
measurement of most of the entries of $L_0$, and a small (but
constant) fraction of the {\em random} entries are arbitrarily
corrupted. This work is much closer in spirit, in motivation, and in
terms of techniques, to the low-rank matrix completion and matrix
recovery problems in
\cite{CandesRecht09,Recht09,RechtFazelParrilo2010} than the setting
we consider and the work presented herein. In particular, in our
setting, each corrupted point can change {\em every} element of a
column of $M$, and hence render the low rank approach invalid.

\section{HR-PCA: The Algorithm}\label{sec.algorithm}
The algorithm of HR-PCA is presented in this section. We start with
the mathematical setup of the problem in
Section~\ref{sss.problemsetup}.  The HR-PCA algorithm as well as its
performance guarantee are then given in
Section~\ref{sss.mainalgorithm}.

\subsection{Problem Setup}\label{sss.problemsetup}
We now define in detail the problem described above.
\begin{itemize}
\item The ``authentic samples'' $\mathbf{z}_1,\dots, \mathbf{z}_t \in \mathbb{R}^m$
are generated by $\mathbf{z}_i=A\mathbf{x}_i+\mathbf{n}_i$, where
$\mathbf{x}_i\in \mathbb{R}^d$ (the ``signal'') are i.i.d. samples
of a random variable $\mathbf{x}$, and $\mathbf{n}_i$ (the
``noise'') are independent realizations of $\mathbf{n}\sim
\mathcal{N}(\mathbf{0}, I_m)$. The matrix $A\in \mathbb{R}^{m\times
d}$ and the distribution of $\mathbf{x}$ (denoted by $\mu$) are
unknown. We do assume, however, that the distribution $\mu$ is
absolutely continuous with respect to the Borel measure, it is
spherically symmetric (and in particular, $\mathbf{x}$ has mean zero
and variance $I_d$) and it has light tails, specifically, there
exist constants $K, C > 0$ such  that
$\mathrm{Pr}(\|\mathbf{x}\|\geq x)\leq K\exp(-C x)$ for all
$x\geq0$.  Since the distribution $\mu$ and the dimension d are both
fixed, as m,n scale, the assumption that mu is spherically symmetric
can be easily relaxed, and the expense of potentially significant
{\it notational complexity}.
\item The outliers (the corrupted data) are denoted $\mathbf{o}_1,\dots,\mathbf{o}_{n-t} \in \mathbb{R}^m$ and as emphasized above, they are arbitrary (perhaps even maliciously chosen). We denote the fraction of corrupted points by $\lambda \bydef (n-t)/{n}$.

\item We only observe the contaminated data set
$$
\mathcal{Y}\triangleq \{\mathbf{y}_1\dots,
\mathbf{y}_n\}=\{\mathbf{z}_1,\dots,
\mathbf{z}_t\}\bigcup\{\mathbf{o}_1,\dots,\mathbf{o}_{n-t}\}.
$$
An element of $\mathcal{Y}$ is called a ``point''.
\end{itemize}
Given these contaminated observations, we want to recover the
principal components of $A$, equivalently, the top eigenvectors,
$\overline{\mathbf{w}}_1,\dots,\overline{\mathbf{w}}_d$ of
$AA^{\top}$. That is, we seek a collection of orthogonal vectors
$\mathbf{w}_1,\dots,\mathbf{w}_d$, that maximize the performance
metric called the {\em Expressed Variance}:
$$
{\rm E.V.} \triangleq \frac{\sum_{j=1}^d \mathbf{w}_j^\top A A^\top
\mathbf{w}_j}{\sum_{j=1}^d \overline{\mathbf{w}}_j^{\top} A A^\top
\overline{\mathbf{w}}_j}=\frac{\sum_{j=1}^d \mathbf{w}_j^\top A
A^\top \mathbf{w}_j}{\mathrm{trace}(AA^\top)}.
$$
The E.V. is always less than one, with equality achieved exactly
when the vectors $\mathbf{w}_1,\dots,\mathbf{w}_d$ have the span of
the true principal components
$\{\overline{\mathbf{w}}_1,\dots,\overline{\mathbf{w}}_d\}$. When
$d=1$, the Expressed Variance relates to another natural performance
metric --- the angle between $\mathbf{w}_1$ and
$\overline{\mathbf{w}}_1$ --- since by definition
$E.V.(\mathbf{w}_1)=\cos^2(\angle(\mathbf{w}_1,\,\overline{\mathbf{w}}_1))$.\footnote{This
geometric interpretation does not extend to the case where $d>1$,
since the angle between two subspaces is not well defined.} The
Expressed Variance represents the portion of signal $A\mathbf{x}$
being expressed by $\mathbf{w}_1,\dots,\mathbf{w}_d$. Equivalently,
$1-E.V.$ is the reconstruction error of the signal.

It is natural to expect that the ability to recover vectors with a
high expressed variance depends on $\lambda$, the fraction of
corrupted points --- in addition, it depends on the distribution,
$\mu$ generating the (low-dimensional) points $\mathbf{x}$, through
its tails. If $\mu$ has longer tails, outliers that affect the
variance (and hence are far from the origin) and authentic samples
in the tail of the distribution, become more difficult to
distinguish. To quantify this effect, we define the following ``tail
weight'' function $\mathcal{V}: [0,\,1]\rightarrow [0, \, 1]$:
$$
\mathcal{V}(\alpha)\triangleq \int_{-c_{\alpha}}^{c_{\alpha}} x^2  \overline{\mu}(dx);
$$
where $\overline{\mu}$ is the one-dimensional margin of $\mu$
(recall that $\mu$ is spherically symmetric), and $c_{\alpha}$ is
such that $\overline{\mu}([-c_{\alpha}, c_{\alpha}] =\alpha)$. Since
$\mu$ has a density function, $c_{\alpha}$ is well defined.  Thus,
$\mathcal{V}(\cdot)$ represents how the tail of $\overline{\mu}$
contributes to its variance. Notice that $\mathcal{V}(0)=0$,
$\mathcal{V}(1)=1$, and $\mathcal{V}(\cdot)$ is continuous since
$\mu$ has a density function. For notational convenience, we simply
let $\mathcal{V}(x)=0$ for $x<0$, and $\mathcal{V}(x)=\infty$ for
$x>1$.

The bounds on the quality of recovery, given in Theorems
\ref{thm.main} and \ref{thm.asymptotic} below, are functions of
$\eta$ and the function $\mathcal{V}(\cdot)$.

\subsubsection*{High Dimensional Setting and Asymptotic Scaling}
In this paper, we focus on the case where $n\sim m\gg d$ and
$\mathrm{trace}(A^\top A)\gg 1$. That is, the number of observations
and the dimensionality are of the same magnitude, and much larger
than the dimensionality of $\mathbf{x}$; the trace of $A^\top A$ is
significantly larger than $1$, but may be much smaller than $n$ and
$m$. In our asymptotic scaling, $n$ and $m$ scale together to
infinity, while $d$ remains fixed. The value of $\sigma_1$ also scales
to infinity, but there is no lower bound on the rate at which this
happens (and in particular, the scaling of $\sigma_1$ can be much
slower than the scaling of $m$ and $n$).

While we give finite-sample results, we are particularly interested in the asymptotic performance of HR-PCA when {\it the dimension and the number of observations grow together} to infinity. Our asymptotic setting is as follows. Suppose there exists a sequence of sample sets $\{\mathcal{Y}(j)\} = \{\mathcal{Y}(1),\mathcal{Y}(2), \dots\}$, where for $\mathcal{Y}(j)$, $n(j)$, $m(j)$, $A(j)$, $d(j)$, etc., denote the corresponding values of the quantities defined above. Then the following must hold for some positive constants $c_1, c_2$:
\begin{equation}\label{equ.condforasympo1}
\begin{split}&\lim_{j\rightarrow \infty} \frac{n(j)}{m(j)}=
c_1;\quad
 d(j) \leq c_2;\quad m(j)\uparrow +\infty;\\
 &\mathrm{trace}(A(j)^\top A(j)) \uparrow +\infty.
\end{split}\end{equation}
While $\mathrm{trace}(A(j)^\top A(j)) \uparrow +\infty$, if it scales more slowly than $\sqrt{m(j)}$, the SNR will asymptotically decrease to zero.

\subsection{Key Idea and Main Algorithm}\label{sss.mainalgorithm}
For $\mathbf{w}\in \mathcal{S}_m$, we define the Robust Variance
Estimator (RVE) as $\overline{V}_{\hat{t}}(\mathbf{w})\triangleq
\frac{1}{n}\sum_{i=1}^{\hat{t}} |\mathbf{w}^\top
\mathbf{y}|^2_{(i)}$. This stands for the following statistics:
project $\mathbf{y}_i$ onto the direction $\mathbf{w}$, replace the
furthest (from original) $n-\hat{t}$ samples by $0$, and then
compute the variance. Notice that the RVE is always performed on the
original observed set $\mathcal{Y}$.

The main algorithm of HR-PCA is as given below.
\medskip

\begin{center}
\begin{boxedminipage}[H]{16cm}
\begin{algorithm}\label{alg.main}
{\rm {\bf HR-PCA} \begin{list}\usecount{}
\item[{\bf Input:}] Contaminated sample-set $\mathcal{Y}=\{\mathbf{y}_1,\dots,\mathbf{y}_n\}
\subset \mathbb{R}^m$, $d$, $\overline{T}$, $\hat{t}$.
\item[{\bf Output:}] $\mathbf{w}^*_1,\dots, \mathbf{w}^*_d$.
\item[{\bf Algorithm:} ] \begin{enumerate}
\item[]
\item Let $\hat{\mathbf{y}}_i:=\mathbf{y}_i$ for $i=1,\dots
n$; $s:=0$; $\mathrm{Opt}:=0$.
\item While $s\leq \overline{T}$, do
\begin{enumerate}
\item\label{ste.one-iteration} Compute the empirical variance matrix \[\hat{\Sigma}:=\frac{1}{n-s}\sum_{i=1}^{n-s}
\hat{\mathbf{y}}_i\hat{\mathbf{y}}_i^\top.\]
\item Perform PCA on $\hat{\Sigma}$. Let $\mathbf{w}_1,\dots, \mathbf{w}_d$ be the $d$
principal components of $\hat{\Sigma}$.
\item If $\sum_{j=1}^d\overline{V}_{\hat{t}}(\mathbf{w}_j) >
\mathrm{Opt}$, then let
$\mathrm{Opt}:=\sum_{j=1}^d\overline{V}_{\hat{t}}(\mathbf{w}_j)$ and
let $\mathbf{w}^*_j:=\mathbf{w}_j$ for $j=1, \cdots, d$.
\item Randomly remove a point from
$\{\hat{\mathbf{y}}_i\}_{i=1}^{n-s}$ according to
\[\mathrm{Pr}(\hat{\mathbf{y}}_i\mbox{ is removed})\propto
\sum_{j=1}^d (\mathbf{w}_j^\top\hat{\mathbf{y}}_i)^2;\]
\item Denote
the remaining points by $\{\hat{\mathbf{y}}_i\}_{i=1}^{n-s-1}$;
\item $s:=s+1$.
\end{enumerate}
\item
Output $\mathbf{w}^*_1, \dots,\mathbf{w}^*_d$. End.
\end{enumerate}
\end{list}}
\end{algorithm}
\end{boxedminipage}
\end{center}

\subsubsection*{Intuition on Why The Algorithm Works}
On any given iteration, we select candidate directions based on
standard PCA -- thus directions chosen are those with largest
empirical variance. Now, given a candidate direction, $\mathbf{w}$,
our robust variance estimator measures the variance of the
$(n-\hat{t})$-smallest points projected in that direction. If this
is large, it means that many of the points have a large variance in
this direction -- the points contributing to the robust variance
estimator, and the points that led to this direction being selected
by PCA. If the robust variance estimator is small, it is likely that
a number of the largest variance points are corrupted, and thus
removing one of them randomly, in proportion to their distance in
the direction $\mathbf{w}$, will remove a corrupted point.

Thus in summary, the algorithm works for the following intuitive reason.
If the corrupted points have a very high variance along a direction with large angle from the span of the principal components, then with some probability, our algorithm removes them. If they have a high variance in a direction ``close to'' the span of the principal components, then this can only help in finding the principal components. Finally, if the corrupted points do not have a large variance, then the distortion they can cause in the output of PCA is necessarily limited.

The remainder of the paper makes this intuition precise, providing lower bounds on the probability of removing corrupted points, and subsequently upper bounds on the maximum distortion the corrupted points can cause, i.e., lower bounds on the Expressed Variance of the principal components our algorithm recovers.



There are two parameters to tune for HR-PCA, namely $\hat{t}$ and
$\overline{T}$. Basically, $\hat{t}$ affects the performance of
HR-PCA through Inequality~\ref{equ.asympo}, and as a rule of thumb
we can set $\hat{t}=t$ when no {\it a priori} information of $\mu$ exists.
$\overline{T}$ does not affect the performance as long as it is
large enough, hence we can simply set $T=n-1$, although when
$\lambda$ is small, a smaller $T$ leads to the same solution with
less computational cost.

The correctness of HR-PCA is shown in the following theorems for both the finite-sample bound, and the asymptotic performance.

\begin{theorem}[Finite Sample Performance]\label{thm.main}
Let the algorithm above output
$\{\mathbf{w}_1,\dots,\mathbf{w}_d\}$. Fix a $\kappa>0$, and let
$\tau=\max(m/n, 1)$. There exists a universal constant $c_0$ and a
constant $C$ which can possible depend on $\hat{t}/t$, $\lambda$,
$d$, $\mu$ and $\kappa$, such that for any $\gamma <1$, if $n/\log^4
n\geq \log^6(1/\gamma)$, then with probability $1-\gamma$ the
following holds

\begin{eqnarray*}
\mathrm{E.V.\{\mathbf{w}_1,\dots,\mathbf{w}_d\}} &\geq&
\left[\frac{\mathcal{V}\left(1-\frac{\lambda(1+\kappa)}{(1-\lambda)\kappa}\right)}{(1+\kappa)}\right] \times
\left[\frac{\mathcal{V}\left(\frac{\hat{t}}{t}-\frac{\lambda}{1-\lambda}\right)}{\mathcal{V}\left(\frac{\hat{t}}{t}\right)}\right]
\\
&& -\left[\frac{8\sqrt{c_0\tau
d}}{\mathcal{V}\left(\frac{\hat{t}}{t}\right)}\right](\mathrm{trace}(AA^\top))^{-1/2}
-\left[\frac{2c_0\tau}{\mathcal{V}\left(\frac{\hat{t}}{t}\right)}\right](\mathrm{trace}(AA^\top))^{-1}
-C \frac{\log^2 n \log^3(1/\gamma)}{\sqrt{n}}.
\end{eqnarray*}
\end{theorem}

The last three terms go to zero as the dimension and number of points scale to infinity, i.e., as $n.m \rightarrow \infty$. Therefore, we immediately obtain:
\begin{theorem}[Asymptotic Performance]\label{thm.asymptotic} Given
a sequence of $\{\mathcal{Y}(j)\}$, if the asymptotic scaling in
Expression~(\ref{equ.condforasympo1}) holds, and
$\limsup\lambda(j)\leq \lambda^*$,  then the following holds in
probability when $j\uparrow \infty$ (i.e., when $n,m \uparrow
\infty$),
\begin{equation}\label{equ.asympo}
\lim\inf_j\mathrm{E.V.\{\mathbf{w}_1(j),\dots,\mathbf{w}_d(j)\}
}\geq \max_{\kappa}
\left[\frac{\mathcal{V}\left(1-\frac{\lambda^*(1+\kappa)}{(1-\lambda^*)\kappa}\right)}{(1+\kappa)}\right]
\times
\left[\frac{\mathcal{V}\left(\frac{\hat{t}}{t}-\frac{\lambda^*}{1-\lambda^*}\right)}{\mathcal{V}\left(\frac{\hat{t}}{t}\right)}\right].
\end{equation}
\end{theorem}
\begin{remark}{\rm The bounds in the two bracketed terms in the asymptotic bound may be, roughly, explained as follows.  The first term is due to the fact that the removal procedure may well not remove all large-magnitude corrupted points, while at the same time, some authentic points may be removed. The second term accounts for the fact that not all the outliers may have large magnitude. These will likely not be removed, and will have some (small) effect on the principal component directions reported in the output.}
\end{remark}
\begin{remark}{\rm The terms in the second line of Theorem~\ref{thm.main} go to zero as $n,m \rightarrow \infty$, and therefore the proving Theorem~\ref{thm.main} immediately implies Theorem~\ref{thm.asymptotic}.}
\end{remark}
\begin{remark}{\rm If $\lambda(j)\downarrow 0$, i.e., the {\em number} of corrupted points scales sublinearly (in particular, this holds when there are a fixed number of corrupted points), then the
right-hand-side of Inequality~(\ref{equ.asympo}) equals $1$,\footnote{We can take $\kappa(j)=\sqrt{\lambda(j)}$ and note that since $\mu$ has a density, $\mathcal{V}(\cdot)$ is continuous.} i.e.,
HR-PCA is asymptotically optimal. This is in contrast to PCA, where
the existence of {\em even a single} corrupted point is sufficient
to bound the output {\em arbitrarily} away from the optimum.}
\end{remark}
\begin{remark}{\rm The breakdown point of HR-PCA  converges to
$50\%$. Note that since $\mu$ has a density function,
$\mathcal{V}(\alpha)>0$ for any $\alpha\in (0,1]$. Therefore, for
any $\lambda<1/2$, if we set $\hat{t}$ to any value in $(\lambda n,
t]$, then there exists $\kappa$ large enough such that the
right-hand-side is strictly positive (recall that $t=(1-\lambda)n$).
The breakdown point hence converges to $50\%$. Thus, HR-PCA achieves
the maximal possible break-down point (note that a breakdown point
greater than $50\%$ is never possible, since then there are more
outliers than samples. } \end{remark}

The graphs in Figure~\ref{fig.lowerbound} illustrate the
lower-bounds of asymptotic performance if the 1-dimension marginal
of $\mu$ is the Gaussian distribution or the Uniform distribution.
\begin{figure}[h]
\begin{center}
\begin{tabular}{cc}
  \includegraphics[height=5cm, width=0.45\linewidth]{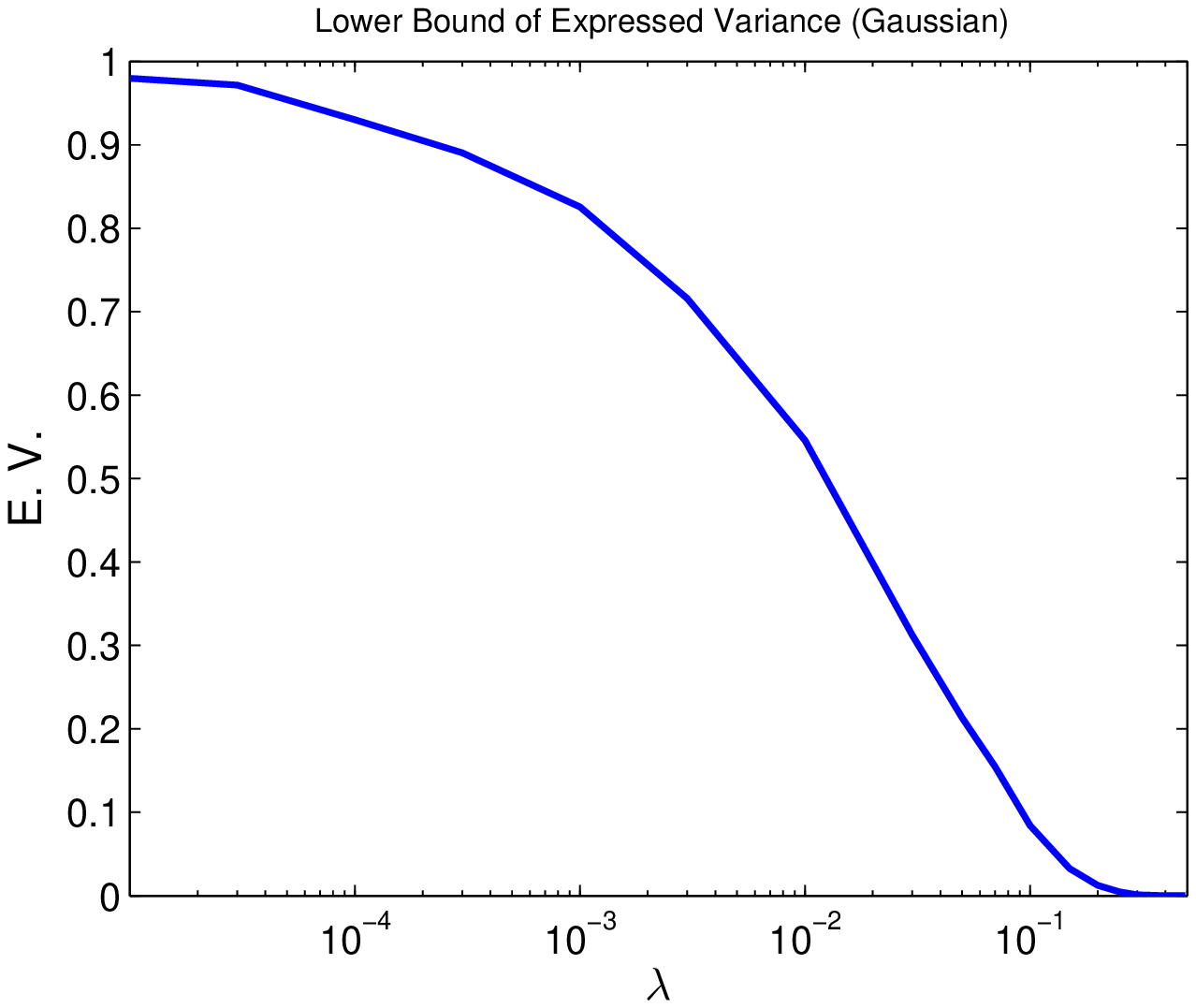}
&  \includegraphics[height=5cm,
width=0.45\linewidth]{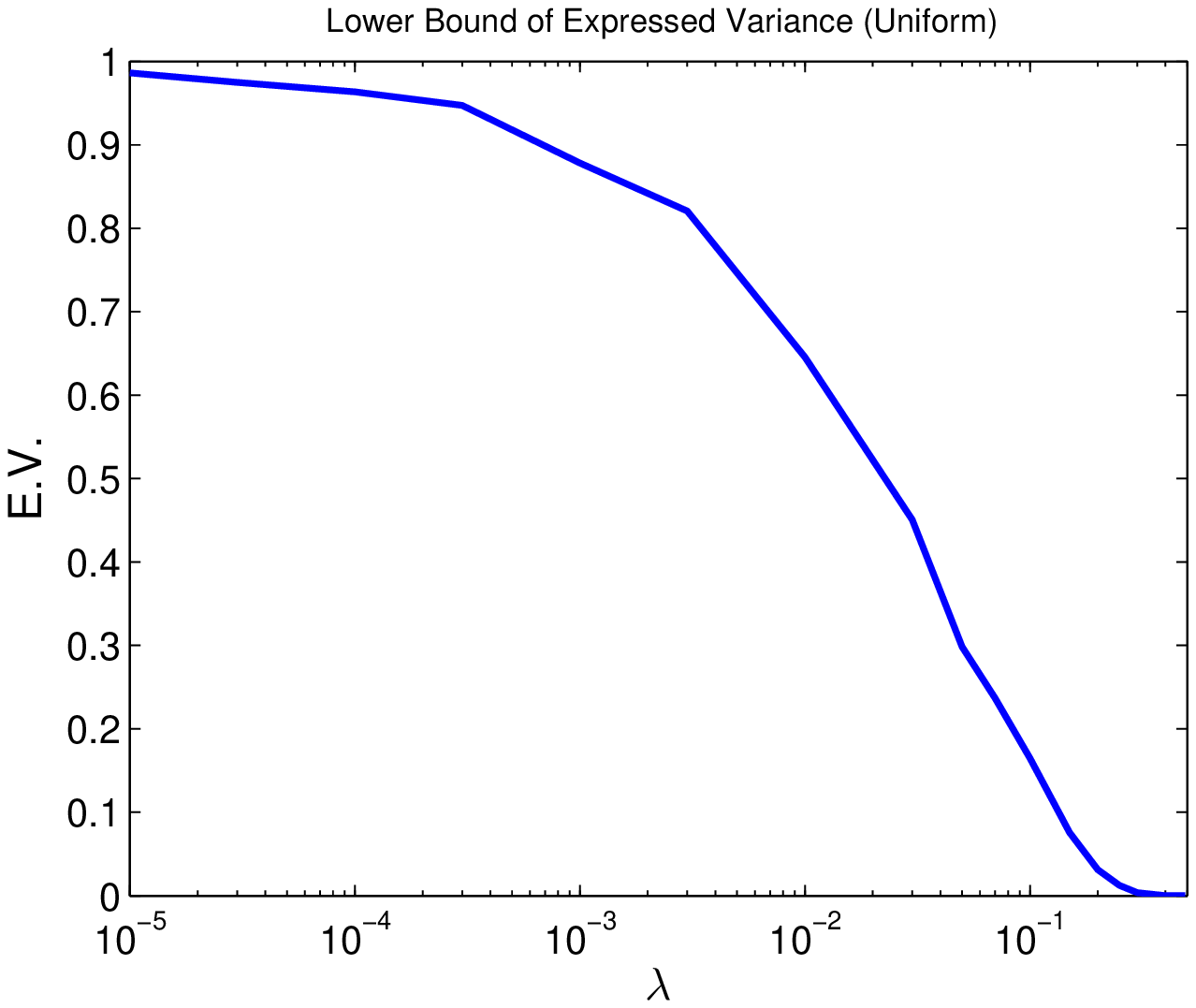} \\ (a) Gaussian
distribution & (b) Uniform distribution
\end{tabular}
 \caption{Lower Bounds of Asymptotic Performance.}
\label{fig.lowerbound}
\end{center}
\end{figure}
\section{Kernelization}\label{sec.kernelization}
We consider kernelizing HR-PCA in this section: given a feature
mapping $\Upsilon(\cdot): \mathbb{R}^m\rightarrow \mathcal{H}$
equipped with a kernel function $k(\cdot, \cdot)$, i.e.,
$\langle\Upsilon(\mathbf{a}),\,\Upsilon(\mathbf{b})\rangle
=k(\mathbf{a},\mathbf{b})$ holds for all $\mathbf{a}, \mathbf{b}\in
\mathbb{R}^m$, we perform the dimensionality reduction in the
feature space $\mathcal{H}$ without knowing the explicit form of
$\Upsilon(\cdot)$.

We assume that $\{\Upsilon(\mathbf{y}_1), \cdots,
\Upsilon(\mathbf{y}_n)\}$ is centered at origin without loss of
generality, since we can center  any $\Upsilon(\cdot)$ with the
following feature mapping
\[\hat{\Upsilon}(\mathbf{x})\triangleq
\Upsilon(\mathbf{x})-\frac{1}{n}\sum_{i=1}^n
\Upsilon(\mathbf{y}_i),\] whose kernel function is
\[\hat{k}(\mathbf{a},\mathbf{b})= k(\mathbf{a},\mathbf{b})-\frac{1}{n}\sum_{j=1}^n  k(\mathbf{a},\mathbf{y}_j)- \frac{1}{n}\sum_{i=1}^n  k(\mathbf{y}_i,\mathbf{b})+\frac{1}{n^2}\sum_{i=1}^n\sum_{j=1}^n k(\mathbf{y}_i,\mathbf{y}_j).\]

Notice that HR-PCA involves finding a set of PCs $
\mathbf{w}_1,\dots,\mathbf{w}_d\in \mathcal{H}$, and evaluating
$\langle\mathbf{w}_q,\, \Upsilon(\cdot) \rangle$ (Note that RVE is a
function of $\langle\mathbf{w}_q,\, \Upsilon(\mathbf{y}_i) \rangle$,
and random removal depends on $\langle\mathbf{w}_q,\,
\Upsilon(\hat{\mathbf{y}}_i) \rangle$).  The former can be
kernelized by applying Kernel PCA introduced by
\cite{ScholkopfSmolaMuller99}, where each of the output PCs admits a
representation
\[\mathbf{w}_q=\sum_{j=1}^{n-s} \alpha_j(q) \Upsilon(\hat{\mathbf{y}}_j).\]
Thus, $\langle\mathbf{w}_q,\, \Upsilon(\cdot) \rangle$ is easily
evaluated by
\[\langle\mathbf{w}_q,\, \Upsilon(\mathbf{v}) \rangle= \sum_{j=1}^{n-s} \alpha_j(q)
k(\hat{\mathbf{y}}_j, \mathbf{v});\quad \forall \mathbf{v}\in
\mathbb{R}^m\] Therefore, HR-PCA is kernelizable since both steps
are easily kernelized and we have the following Kernel HR-PCA.
\medskip

\begin{center}
\begin{boxedminipage}[H]{16cm}
\begin{algorithm}\label{alg.kernel}
{\rm {\bf Kernel HR-PCA} \begin{list}\usecount{}
\item[{\bf Input:}] Contaminated sample-set $\mathcal{Y}=\{\mathbf{y}_1,\dots,\mathbf{y}_n\}
\subset \mathbb{R}^m$, $d$, $T$, $\hat{n}$.
\item[{\bf Output:}] $\boldsymbol{\alpha}^*(1),\dots, \boldsymbol{\alpha}^*(d)$.
\item[{\bf Algorithm:} ] \begin{enumerate}
\item[]
\item Let $\hat{\mathbf{y}}_i:=\mathbf{y}_i$ for $i=1,\dots
n$; $s:=0$; $\mathrm{Opt}:=0$.
\item While $s\leq T$, do
\begin{enumerate}
\item Compute the Gram matrix of $\{\hat{\mathbf{y}}_i\}$: \[K_{ij}:=k(\hat{\mathbf{y}}_i,\hat{\mathbf{y}}_j);\quad i,j=1,\cdots,n-s.\]
\item Let $\hat{\sigma}^2_1,\cdots, \hat{\sigma}^2_d$ and $\hat{\boldsymbol{\alpha}}(1),\cdots, \hat{\boldsymbol{\alpha}}(d)$ be the $d$
largest eigenvalues and the corresponding eigenvectors of $K$.
\item Normalize:
$\boldsymbol{\alpha}(q):=\hat{\boldsymbol{\alpha}}(q)/\hat{\sigma}_q$,
so that $\|\mathbf{w}_q\|=1$.
\item If $\sum_{q=1}^d\overline{V}_{\hat{t}}(\boldsymbol{\alpha}(q)) >
\mathrm{Opt}$, then let
$\mathrm{Opt}:=\sum_{q=1}^d\overline{V}_{\hat{t}}(\boldsymbol{\alpha}(q))$
and let $\boldsymbol{\alpha}^*(q):=\boldsymbol{\alpha}(q)$ for $q=1,
\cdots, d$.
\item Randomly remove a point from
$\{\hat{\mathbf{y}}_i\}_{i=1}^{n-s}$ according to
\[\mathrm{Pr}(\hat{\mathbf{y}}_i\mbox{ is removed})\propto
\sum_{q=1}^d (\sum_{j=1}^{n-s} \alpha_j(q) k(\hat{\mathbf{y}}_j,
\hat{\mathbf{y}}_i))^2;\]
\item Denote
the remaining points by $\{\hat{\mathbf{y}}_i\}_{i=1}^{n-s-1}$;
\item $s:=s+1$.
\end{enumerate}
\item
Output $\boldsymbol{\alpha}^*(1), \dots,\boldsymbol{\alpha}^*(d)$.
End.
\end{enumerate}
\end{list}}
\end{algorithm}
\end{boxedminipage}
\end{center}
\medskip

Here, the kernelized RVE is defined as
\[\overline{V}_{\hat{t}}(\boldsymbol\alpha)\triangleq
\frac{1}{n}\sum_{i=1}^{\hat{t}} \left[\big|\langle \sum_{j=1}^{n-s}
\alpha_j\Upsilon(\hat{\mathbf{y}}_j), \Upsilon(\mathbf{y})
\rangle\big|_{(i)}\right]^2=\frac{1}{n}\sum_{i=1}^{\hat{n}}
\left[\big|\sum_{j=1}^{n-s} \alpha_j k(\hat{\mathbf{y}}_j,
\mathbf{y})\big|_{(i)}\right]^2.\]

\section{Proof of the Main Result}
\label{sec.proof}
In this section we provide the main steps of the proof of the finite-sample and asymptotic performance bounds, including the precise statements and the key ideas in the proof, but deferring some of the more standard or tedious elements to the appendix.
The proof consists of three steps which we now outline. In what follows, we let $d$, $m/n$, $\lambda$, $\hat{t}/t$, and $\mu$ be fixed. We can fix a $\lambda\in (0, 0.5)$ without loss of generality, due to the fact that if a result is shown to hold for $\lambda$, then it holds for $\lambda'<\lambda$. The letter $c$ is used to represent a constant, and $\epsilon$ is a constant that decreases to zero as $n$ and $m$ increase to infinity. The values of $c$ and $\epsilon$ can change from line to line, and can possible depend on  $d$, $m/n$, $\lambda$, $\hat{t}/t$, and $\mu$.
\begin{enumerate}
\item The blessing of dimensionality, and laws of large numbers: The first step involves two ideas; the first is the (well-known, e.g., \cite{DavidsonSzarek01}) fact that even as $n$ and $m$ scale, the expectation of the covariance of the noise is bounded \emph{independently} of $m$. The second involves appealing to laws of large numbers to show that sample estimates of the covariance of the noise, $\mathbf{n}$, of the signal, $\mathbf{x}$, and then of the authentic points, $\mathbf{z} = A\mathbf{x} + \mathbf{n}$, are uniformly close to their expectation, with high probability. Specifically, we prove:
\begin{enumerate}
\item\label{item.noise} With high probability, the largest eigenvalue of  the variance of noise matrix is
bounded. That is,
$$
\sup_{\mathbf{w} \in \mathcal{S}_m}\frac{1}{n}\sum_{i=1}^t (\mathbf{w}^\top \mathbf{n}_i)^2\leq c.
$$
\item\label{item.signal} With high probability, both the largest and the smallest eigenvalue of the signals in the original space
 converge to $1$.  That is
\[\sup_{\mathbf{w}\in \mathcal{S}_d}|\frac{1}{t}\sum_{i=1}^t (\mathbf{w}^\top\mathbf{x}_i)^2-1|\leq \epsilon.\]
\item\label{item.RVE} Under~\ref{item.signal}, with high probability,
RVE is a valid variance estimator for the $d-$dimensional signals.
That is,
$$
\sup_{\mathbf{w}\in \mathcal{S}_d} \big|\frac{1}{t} \sum_{i=1}^{\hat{t}}|\mathbf{w}^\top \mathbf{x}|_{(i)}^2 -\mathcal{V}\left(\frac{\hat{t}}{t}\right)\big| \leq \epsilon.
$$
\item\label{item.RVEinm} Under~\ref{item.noise} and~\ref{item.RVE},
RVE is a valid estimator of the variance of the authentic samples.
That is, the following holds uniformly over all $\mathbf{w}\in
\mathcal{S}_m$,
\begin{equation*}\begin{split}
&(1-\epsilon)\|\mathbf{w}^\top A\|^2
\mathcal{V}\left(\frac{t'}{t}\right) -c\|\mathbf{w}^\top A\| \leq
\frac{1}{t}\sum_{i=1}^{t'} |\mathbf{w}^\top \mathbf{z}|_{(i)}^2 \leq
 (1+\epsilon)\|\mathbf{w}^\top A\|^2
\mathcal{V}\left(\frac{t'}{t}\right)+c\|\mathbf{w}^\top A\|.
\end{split}\end{equation*}
\end{enumerate}
\item\label{item.s} The next step shows that with high probability, the algorithm finds a ``good'' solution within a bounded number of steps. In particular, this involves showing that if in a given step the algorithm has not found a good solution, in the sense that the variance along a principal component is not mainly due to the authentic points, then the random removal scheme removes a corrupted point with probability bounded away from zero. We then use martingale arguments to show that as a consequence of this, there cannot be many steps with the algorithm finding at least one ``good'' solution, since in the absence of good solutions, most of the corrupted points are removed by the algorithm.
%

\item The previous step shows the existence of a ``good'' solution. The final step shows two things: first, that this good solution has performance that is close to that of the optimal solution, and second, that the final output of the algorithm is close to that of the ``good'' solution. Combining these two steps, we derive the finite-sample and asymptotic performance bounds for HR-PCA.
\end{enumerate}

\subsection{Step 1a}

 \begin{theorem}
 \label{thm.step1a}Let $\tau=\max(m/n, 1)$.
 There exist universal constants $c$ and $c'$ such that for any $\gamma>0$, with probability at least $1-\gamma$, the following holds:
$$
\sup_{\mathbf{w}\in \mathcal{S}_m}\frac{1}{t}\sum_{i=1}^t
(\mathbf{w}^\top \mathbf{n}_i)^2\leq
c\tau+\frac{c'\log{\frac{1}{\gamma}}}{n}.
$$
 \end{theorem}
 \begin{proof} The proof of the theorem depends on the following lemma, that is essentially Theorem II.13 in \cite{DavidsonSzarek01}.
\begin{lemma} Let $\Gamma$ be an $n \times p$ matrix with $n \leq p$,
whose entries are all i.i.d. $\mathcal{N}(0,1)$ Gaussian variables.
Let $s_1(\Gamma)$ be the largest singular value of $\Gamma$; then
$$
\mathrm{Pr}\big(s_1(\Gamma)>\sqrt{n}+\sqrt{p}+\sqrt{p}\epsilon\big)\leq \exp(-p\epsilon^2/2).
$$
 \end{lemma}
Our result now follows, since $sup_{\mathbf{w}\in
\mathcal{S}_m}\frac{1}{t}\sum_{i=1}^t (\mathbf{w}^\top
\mathbf{n}_i)^2$ is the largest eigenvalue of $W=(1/t)\Gamma_1^\top
\Gamma_1$, where $\Gamma_1$ is a $m\times t$ matrix whose entries
are all i.i.d. $\mathcal{N}(0,1)$ Gaussian variables; and, moreover,
the largest eigenvalue of $W$ is given by
$\lambda_W=[s_1(\Gamma_1)]^2/t$. Specifically, we have
\begin{equation*}\begin{split}
&\mathrm{Pr}\big(\lambda_W >\frac{ \tau(2n+n\epsilon^2 + 2n
+2\sqrt{n^2}\epsilon)}{(1-\lambda)n}
\big)\\\leq&\mathrm{Pr}\big(\lambda_W
>\frac{ m+t+\max(m,t)\epsilon^2 + 2\sqrt{mt}
+2\sqrt{(m+t)\max(m,t)}\epsilon}{t}
\big)\\
=&\mathrm{Pr}\big(s_1(\Gamma)> \sqrt{m} +\sqrt{t} +
\sqrt{\max(m,t)}\epsilon\big)\leq \exp(-\max(m,t)\epsilon^2/2)\leq
\exp(-(1-\lambda)n\tau\epsilon^2/2).\end{split}\end{equation*} Let
$\gamma$ equals the r.h.s. and note that $\lambda <1/2$, we have
that
 \[\sup_{\mathbf{w}\in
\mathcal{S}_m}\frac{1}{t}\sum_{i=1}^t (\mathbf{w}^\top
\mathbf{n}_i)^2\leq
8\tau+8\sqrt{\frac{\tau\log{\frac{1}{\gamma}}}{n}}+\frac{8\log{\frac{1}{\gamma}}}{n}.\]
 The theorem follows by letting $c=16$ and $c'=16$.

\end{proof}

\subsection{Step 1b}
\begin{theorem}\label{thm.uniform-signal} There exists a constant $c$ that only depends on $\mu$ and $d$, such that for any $\gamma>0$, with probability at least $1-\gamma$,
$$
\sup_{\mathbf{w}\in \mathcal{S}_d} \big|\frac{1}{t} \sum_{i=1}^{t}(\mathbf{w}^\top \mathbf{x}_i)^2 -1\big| \leq \frac{c \log^2 n \log^3\frac{1}{\gamma}}{\sqrt{n}}.
$$
\end{theorem}
\begin{proof}
The proof of Theorem~\ref{thm.uniform-signal} depends on the following matrix concentration inequality from \cite{MendelsonPajor06}.
\begin{theorem}\label{thm.convergencemendelson} There exists an absolute constant $c_0$ for which the
following holds. Let $X$ be a random vector in $\mathbb{R}^n$, and  set $Z=\|X\|$.
If $X$ satisfies
\begin{enumerate}
\item There is some $\rho>0$ such that $\sup_{\mathbf{w}\in \mathcal{S}_n}\big( (\mathbb{E}(\mathbf{w}^\top X)^4\big)^{1/4} \leq \rho$,
\item $\|Z\|_{\psi_{\alpha}} <\infty$ for some $\alpha \geq 1$,
\end{enumerate}
then for any $\epsilon>0$
$$
\mathrm{Pr}\left(\|\frac{1}{N}\sum_{i=1}^N X_i X_i^\top -\mathbb{E}(X_i X_i^\top)\| \geq \epsilon\right)\leq \exp\left[-\left(\frac{c_0\epsilon}{\max(B_{d,N}, A^2_{d,N})}\right)^{\beta}\right],
$$
where $X_i$ are i.i.d. copies of $X$, $d= \min(n, N)$,
$\beta=(1+2/\alpha)^{-1}$ and
$$
A_{d,N}=\|Z\|_{\psi_{\alpha}}\frac{\sqrt{\log d}(\log N)^{1/\alpha}}{\sqrt{N}},\quad B_{d,N}=\frac{\rho^2}{\sqrt{N}}+\|\mathbb{E}(XX^\top)\|^{1/2} A_{d,N}.
$$
\end{theorem}
We apply Theorem~\ref{thm.convergencemendelson} by observing that \begin{equation*}\begin{split}&\sup_{\mathbf{w}\in
\mathcal{S}_d} \big|\frac{1}{t} \sum_{i=1}^{t}(\mathbf{w}^\top
\mathbf{x}_i)^2 -1\big|\\
=&\sup_{\mathbf{w}\in \mathcal{S}_d} \left|\frac{1}{t}
\sum_{i=1}^{t} \mathbf{w}^\top \mathbf{x}_i \mathbf{x}_i^\top
\mathbf{w}- \mathbf{w}^\top \mathbb{E}(\mathbf{x} \mathbf{x}^\top)
\mathbf{w}\right|\\
=& \sup_{\mathbf{w}\in \mathcal{S}_d}
\left|\mathbf{w}^\top\left[\frac{1}{t} \sum_{i=1}^{t}
 \mathbf{x}_i \mathbf{x}_i^\top -
\mathbb{E}(\mathbf{x} \mathbf{x}^\top)
\right]\mathbf{w}\right|\\
\leq & \|\frac{1}{t} \sum_{i=1}^{t}
 \mathbf{x}_i \mathbf{x}_i^\top -
\mathbb{E}(\mathbf{x} \mathbf{x}^\top) \|.
\end{split}
\end{equation*}
One must still check that both conditions in Theorem~\ref{thm.convergencemendelson} are satisfied by $\mathbf{x}$. The first condition is satisfied because $\sup_{\mathbf{w}\in \mathcal{S}_m} \mathbb{E}(\mathbf{w}^\top \mathbf{x})^4 \leq \mathbb{E} \|\mathbf{x}\|^4 <\infty$, where the second inequality follows from
the assumption that $\|\mathbf{x}\|$ has an exponential decay which
guarantees the existence of all moments. The second condition is
satisfied thanks to Lemma 2.2.1. of \cite{Vaart2000}. 
\end{proof}

\subsection{Step 1c}
\begin{theorem}\label{thm.step1c} Fix $\eta<1$. There exists a constant $c$
 that depends on $d$, $\mu$ and $\eta$,
 such that for all $\gamma<1$, $t$,
 the following holds with probability at least
$1-\gamma$:
$$
\sup_{\mathbf{w}\in \mathcal{S}_d, \overline{t} \leq \eta t}\left|\frac{1}t
\sum_{i=1}^{\overline{t}}|\mathbf{w}^\top
\mathbf{x}|_{(i)}^2-\mathcal{V}\left(\frac{\overline{t}}{t}\right)\right| \leq
c\sqrt{\frac{\log n +\log{1/\gamma}}{n}}+c\frac{\log^{5/2}n
\log^{7/2}(1/\gamma)}{n}.
$$
\end{theorem}

We first prove a one-dimensional version of this result, and then use this to prove the general case. We show that if the empirical mean is bounded, then the truncated mean converges to its expectation, and more importantly, the convergence rate is distribution free. Since this is a general result, we abuse the notation $\mu$ and $m$.
\begin{lemma}\label{lem.onedirectionpartialvarbound}
Given $\delta\in [0,1]$, $\hat{c} \in \mathbb{R}^+$, $\hat{m}, m\in \mathbb{N}$ satisfying $\hat{m}<m$. Let $a_1,\cdots, a_m$ be i.i.d. samples drawn from a probability measure $\mu$ supported on $\mathbb{R}^+$ and has a density function. Assume that $\mathbb{E}(a)=1$ and $\frac{1}{m}\sum_{i=1}^ma_i\leq 1+\hat{c}$. Then with probability at least $1-\delta$ we have
$$
\sup_{\overline{m}\leq \hat{m}}|\frac{1}{m}\sum_{i=1}^{\overline{m}} a_{(i)}-
\int_{0}^{\mu^{-1}({\overline{m}}/{m})}a d\mu|\leq
\frac{(2+\hat{c})m}{m-\hat{m}}\sqrt{\frac{8(2\log
m+1+\log\frac{8}{\delta})}{m}},
$$
where $\mu^{-1}(x)\triangleq
\min\{z| \mu(a\leq z)\geq x\}$.
\end{lemma}
\begin{proof} To avoid heavy notation, let
$\epsilon_0=\sqrt{\frac{8(2\log m+1+\log\frac{8}{\delta})}{m}}$ and
$\epsilon= \frac{(1+\hat{c})m}{m-\hat{m}}\epsilon_0$. The key to
obtaining uniform convergence in this proof relies on a standard
Vapnik-Chervonenkis (VC) dimension argument. Consider two classes
of functions $\mathcal{F}=\{f_e(\cdot): \mathbb{R}^+\rightarrow
\mathbb{R}^+|e\in \mathbb{R}^+\}$ and $\mathcal{G}=\{g_e(\cdot):
\mathbb{R}^+\rightarrow \{0, +1\}| e\in \mathbb{R}^+\}$, as
$f_e(a)=a\cdot\mathbf{1}(a\leq e)$ and $g_e(a)=\mathbf{1}(a\leq e)$.
Note that for any $e_1\geq e_2$, the subgraphs of $f_{e_1}$ and
$g_{e_1}$ are contained in the subgraph of $f_{e_2}$ and $g_{e_2}$
respectively, which  guarantees that
$VC(\mathcal{F})=VC(\mathcal{G}) \leq 2$ (cf page 146 of
\cite{Vaart2000}). Since $g_e(\cdot)$ is bounded in $[0,1]$,
$f_e(\cdot)$ is bounded in $[0, e]$, standard VC-based
uniform-convergence analysis yields
\begin{equation*}
\mathrm{Pr}\big(\sup_{e \geq 0}|\frac{1}{m}\sum_{i=1}^m g_e(a_i)-\mathbb{E}g_e(a)|\geq
\epsilon_0\big)\leq 4\exp(2 \log
m+1-m\epsilon_0^2/8)=\frac{\delta}{2};
\end{equation*}
and
\begin{equation*}
\mathrm{Pr}\big(\sup_{e\in [0, (1+c)m/(m-\hat{m})]
}|\frac{1}{m}\sum_{i=1}^m f_e(a_i)-\mathbb{E}f_e(a)|\geq
\epsilon\big)\leq 4\exp\left(2 \log m+1-
\frac{\epsilon^2(m-\hat{m})^2}{8(1+\hat{c})^2m}
\right)=\frac{\delta}{2}.
\end{equation*}
With some additional work (see the appendix for the full details)
these inequalities provide the one-dimensional result of the lemma.
\end{proof}

Next, en route to proving the main result, we prove a uniform multi-dimensional version of the previous lemma.
\begin{theorem}\label{thm.incomvar-d}
If $\sup_{\mathbf{w}\in \mathcal{S}_d} \big|\frac{1}{t}\sum_{i=1}^t (\mathbf{w}^\top\mathbf{x}_i)^2-1\big|\leq
\hat{c}$, then
\begin{equation*}\begin{split} &\mathrm{Pr}\left\{\sup_{\mathbf{w}\in \mathcal{S}_d,
\overline{t}\leq \hat{t}} \big|\frac{1}t
\sum_{i=1}^{\overline{t}}|\mathbf{w}^\top
\mathbf{x}|_{(i)}^2-\mathcal{V}\left(\frac{\overline{t}}{t}\right)\big|\geq
\epsilon\right\}\\ &\leq \max\left[\frac{8e
t^26^d(1+\hat{c})^{d/2}}{\epsilon^{d/2}},\frac{8e
t^224^d(1+\hat{c})^d
t^{d/2}}{\epsilon^{d}(t-\hat{t})^{d/2}}\right]\exp\left(-\frac{\epsilon^2
(1-\hat{t}/t)^2 t}{32
(2+\hat{c})^2}\right).\end{split}\end{equation*}
\end{theorem}
\begin{proof}
To avoid heavy notation, let $\delta_1=\sqrt{\epsilon/(4+4\hat{c})}$, $\delta_2=\epsilon\sqrt{t-\hat{t}}/((8+8\hat{c})\sqrt{t})$, and $\delta=\min(\delta_1, \delta_2)$.

It is well known (cf. Chapter 13 of~\cite{LorentzGolitschekMakovoz96}) that we can construct a finite set $\hat{\mathcal{S}}_d\subset \mathcal{S}_d$ such that $|\hat{\mathcal{S}}_d|\leq (3/\delta)^d$, and
$\max_{\mathbf{w}\in\mathcal{S}}\min_{\mathbf{w}_1\in \hat{\mathcal{S}}_d} \|\mathbf{w}-\mathbf{w}_1\| \leq \delta$. For a
fixed $\mathbf{w}_1\in \hat{\mathcal{S}}_d$, note that $(\mathbf{w}_1^\top\mathbf{x}_1)^2,\cdots,
(\mathbf{w}_1^\top\mathbf{x}_t)^2$ are i.i.d. samples of a non-negative random variable satisfying the conditions of Lemma~\ref{lem.onedirectionpartialvarbound}. Thus by Lemma~\ref{lem.onedirectionpartialvarbound} we have
$$
\mathrm{Pr}\left\{\sup_{\overline{t}\leq \hat{t}}\left|\frac{1}{t} \sum_{i=1}^{\overline{t}}|\mathbf{w}_1^\top
\mathbf{x}|_{(i)}^2-\mathcal{V}\left(\frac{\overline{t}}{t}\right)\right|\geq
\epsilon/2\right\}\leq 8 e t^2 \exp\left(-\frac{ (1-\hat{t}/t)^2
\epsilon^2t}{32 (2+\hat{c})^2}\right).
$$
Thus by the union bound we have
$$
\mathrm{Pr}\left\{\sup_{\mathbf{w} \in \hat{\mathcal{S}}_d, \overline{t}\leq \hat{t}}\left|\frac{1}{t}
\sum_{i=1}^{\overline{t}}|\mathbf{w}^\top
\mathbf{x}|_{(i)}^2-\mathcal{V}\left(\frac{\overline{t}}{t}\right)\right|\geq
\epsilon/2\right\}\leq  \frac{8e t^2 3^d}{\delta^d}
\exp\left(-\frac{(1-\hat{t}/t)^2\epsilon^2  t}{32
(2+\hat{c})^2}\right).
$$
Next, we need to relate the uniform bound on $\mathcal{S}_d$ with
the uniform bound on this finite set. This requires a number of
steps, all of which we postpone to the appendix.
\end{proof}

\begin{corollary}\label{cor.step1c}
If $\sup_{\mathbf{w}\in \mathcal{S}_d} \big|\frac{1}{t}\sum_{i=1}^t
(\mathbf{w}^\top\mathbf{x}_i)^2-1\big|\leq \hat{c}$, then with
probability $1-\gamma$
\[\sup_{\mathbf{w}\in \mathcal{S}_d,
\overline{t}\leq \hat{t}} \big|\frac{1}t
\sum_{i=1}^{\overline{t}}|\mathbf{w}^\top
\mathbf{x}|_{(i)}^2-\mathcal{V}\left(\frac{\overline{t}}{t}\right)\big|\leq
\epsilon_0,\] where \begin{equation*}\begin{split}\epsilon_0 =&
\sqrt{\frac{32(2+\hat{c})^2\big\{\max[\frac{d+4}{2}\log t
+\log{\frac{1}{\gamma}}+ \log(16e
6^d)+\frac{d}{2}\log(1+\hat{c}),\,(1-\hat{t}/t)^2]\big\}}{t(1-\hat{t}/t)^2}}\\&\,+
\sqrt{\frac{32(2+\hat{c})^2\big\{\max[(d+2)\log t
+\log{\frac{1}{\gamma}}+ \log(16e^2
24^d)+d\log(1+\hat{c})-\frac{d}{2}\log(1-\hat{t}/t),\,(1-\hat{t}/t)^2]\big\}}{t(1-\hat{t}/t)^2}}.
\end{split}\end{equation*}
\end{corollary}
\begin{proof} The proof follows from Theorem~\ref{thm.incomvar-d} and from the following lemma, whose proof we leave to the appendix.
\begin{lemma}\label{lem.inproofofsimplysignal} For any $C_1, C_2, d', t \geq 0$, and $0<\gamma<1$,
let \[\epsilon=\sqrt{\frac{\max(d'\log
t-\log(\gamma/C_1),C_2)}{tC_2}},\] then
\[C_1\epsilon^{-d'}\exp(-C_2\epsilon^2 t) \leq \gamma.\]
\end{lemma}
\end{proof}

Now we prove Theorem~\ref{thm.step1c}, which is the main result of
this section.

\begin{proof} By Corollary~\ref{cor.step1c}, there exists a constant $c'$ which only depends on $d$,
such that if $$ \sup_{\mathbf{w}\in \mathcal{S}_d}
\big|\frac{1}{t}\sum_{i=1}^t
(\mathbf{w}^\top\mathbf{x}_i)^2-1\big|\leq \hat{c},$$ then with
probability $1-\gamma/2$
$$
\sup_{\mathbf{w}\in \mathcal{S}_d,
\overline{t}\leq \hat{t}} \big|\frac{1}t
\sum_{i=1}^{\overline{t}}|\mathbf{w}^\top
\mathbf{x}|_{(i)}^2-\mathcal{V}\left(\frac{\overline{t}}{t}\right)\big|\leq
c(2+\hat{c})\sqrt{\frac{\log t +\log{1/\gamma}+\log(1+\hat{c})-\log(1-\hat{t}/{t})}{t(1-\hat{t}/t)^2}}.
$$
Now apply Theorem~\ref{thm.uniform-signal}, to bound $\hat{c}$ by $O(\log^2 n \log^3(1/\gamma)/n)$, and note that $\log(1+\hat{c})$ is thus absorbed by $\log n$ and $\log(1+\gamma)$. The theorem then follows.
\end{proof}

\subsection{Step 1d}
Recall that $\mathbf{z}_i = A\mathbf{x}_i+\mathbf{n}_i$.

\begin{theorem}\label{thm.step1d}Let $t' \leq t$. If there exists $\epsilon_1, \epsilon_2, \overline{c}$ such that
\begin{equation*}\begin{split}
(I)\quad & \sup_{\mathbf{w}\in \mathcal{S}_d} \big|\frac{1}{t}
\sum_{i=1}^{t'}|\mathbf{w}^\top
\mathbf{x}|_{(i)}^2-\mathcal{V}(\frac{t'}{t})\big|\leq
\epsilon_1\\(II)\quad & \sup_{\mathbf{w}\in \mathcal{S}_d}
\big|\frac{1}{t} \sum_{i=1}^{t}|\mathbf{w}^\top
\mathbf{x}_i|^2-1\big|\leq \epsilon_2\\
(III)\quad &\sup_{\mathbf{w}\in\mathcal{S}_m} \frac{1}t
\sum_{i=1}^{t}|\mathbf{w}^\top \mathbf{n}_{i}|^2\leq
\overline{c},\end{split}\end{equation*}then for all $\mathbf{w}\in
\mathcal{S}_m$ the following holds:
\begin{equation*}\begin{split}
&(1-\epsilon_1)\|\mathbf{w}^\top A\|^2
\mathcal{V}\left(\frac{t'}{t}\right) -2\|\mathbf{w}^\top
A\|\sqrt{(1+\epsilon_2)\overline{c}}\\
\leq & \frac{1}{t}\sum_{i=1}^{t'} |\mathbf{w}^\top
\mathbf{z}|_{(i)}^2 \\
\leq & (1+\epsilon_1)\|\mathbf{w}^\top A\|^2
\mathcal{V}\left(\frac{t'}{t}\right)+2\|\mathbf{w}^\top
A\|\sqrt{(1+\epsilon_2)\overline{c}}+\overline{c}.
\end{split}\end{equation*}
\end{theorem}
\begin{proof}
Fix an arbitrary $\mathbf{w}\in \mathcal{S}_m$. Let
$\{\hat{j}_i\}_{i=1}^t$ and $\{\bar{j}_i\}_{i=1}^t$ be permutations
of $[1,\cdots,t]$ such that both
$|\mathbf{w}^\top\mathbf{z}_{\hat{j}_i}|$ and $|\mathbf{w}^\top A
\mathbf{x}_{\bar{j}_i}|$ are non-decreasing. Then we have:
\begin{equation*}\begin{split}&\frac{1}{t}\sum_{i=1}^{t'} |\mathbf{w}^\top
\mathbf{z}|_{(i)}^2 \stackrel{(a)}{=} \frac{1}{t}\sum_{i=1}^{t'}
|\mathbf{w}^\top A \mathbf{x}_{\hat{j}_i} +\mathbf{w}^\top
\mathbf{n}_{\hat{j}_i}|^2\\
\stackrel{(b)}{\leq} &\frac{1}{t}\sum_{i=1}^{t'} |\mathbf{w}^\top A
\mathbf{x}_{\bar{j}_i} +\mathbf{w}^\top
\mathbf{n}_{\bar{j}_i}|^2\\
= &\frac{1}{t}\left\{ \sum_{i=1}^{t'} (\mathbf{w}^\top A
\mathbf{x}_{\bar{j}_i})^2 +2\sum_{i=1}^{t'} (\mathbf{w}^\top A
\mathbf{x}_{\bar{j}_i})(\mathbf{w}^\top
\mathbf{n}_{\bar{j}_i})+\sum_{i=1}^{t'}(\mathbf{w}^\top
\mathbf{n}_{\bar{j}_i})^2\right\}\\
\leq &\frac{1}{t}\left\{ \sum_{i=1}^{t'} (\mathbf{w}^\top A
\mathbf{x}_{\bar{j}_i})^2 +2\sum_{i=1}^{t} (\mathbf{w}^\top A
\mathbf{x}_{\bar{j}_i})(\mathbf{w}^\top
\mathbf{n}_{\bar{j}_i})+\sum_{i=1}^{t}(\mathbf{w}^\top
\mathbf{n}_{\bar{j}_i})^2\right\}\\
\stackrel{(c)}{\leq} & \|\mathbf{w}^\top A\|^2 \sup_{\mathbf{v}\in
\mathcal{S}_d} \frac{1}{t} \sum_{i=1}^{t'} |\mathbf{v}^\top
\mathbf{x}|_{(i)}^2 +2 \sqrt{\frac{1}{t}\sum_{i=1}^{t}
|\mathbf{w}^\top A \mathbf{x}_i|^2}
\sqrt{\frac{1}{t}\sum_{i=1}^{t}|\mathbf{w}^\top
\mathbf{n}_i|^2}+\frac{1}{t}\sum_{i=1}^{t}(\mathbf{w}^\top
\mathbf{n}_i)^2\\
\leq & (1+\epsilon_1)\|\mathbf{w}^\top A\|^2
\mathcal{V}(\hat{t}/t)+2\|\mathbf{w}^\top
A\|\sqrt{(1+\epsilon_2)\overline{c}}
+\overline{c}.\end{split}\end{equation*} Here, $(a)$ and $(b)$ follow
from the definition of $\hat{j}_i$, and $(c)$ follows from the
definition of $\bar{j}_i$ and the well known inequality $(\sum_{i}
a_ib_i)^2 \leq (\sum_i a_i^2)(\sum_i b_i^2)$.

Similarly, we have
\begin{equation*}\begin{split}&\frac{1}{t}\sum_{i=1}^{t'} |\mathbf{w}^\top
\mathbf{z}|_{(i)}^2=\frac{1}{t}\sum_{i=1}^{t'} |\mathbf{w}^\top A
\mathbf{x}_{\hat{j}_i} +\mathbf{w}^\top \mathbf{n}_{\hat{j}_i}|^2\\=
&\frac{1}{t}\left\{ \sum_{i=1}^{t'} (\mathbf{w}^\top A
\mathbf{x}_{\hat{j}_i})^2 +2\sum_{i=1}^{t'} (\mathbf{w}^\top A
\mathbf{x}_{\hat{j}_i})(\mathbf{w}^\top
\mathbf{n}_{\hat{j}_i})+\sum_{i=1}^{t'}(\mathbf{w}^\top
\mathbf{x}_{\hat{j}_i})^2\right\}\\
\stackrel{(a)}{\geq} &\frac{1}{t}\left\{ \sum_{i=1}^{t'}
(\mathbf{w}^\top A \mathbf{x}_{\bar{j}_i})^2 +2\sum_{i=1}^{t'}
(\mathbf{w}^\top A \mathbf{x}_{\hat{j}_i})(\mathbf{w}^\top
\mathbf{n}_{\hat{j}_i})+\sum_{i=1}^{t'}(\mathbf{w}^\top
\mathbf{n}_{\hat{j}_i})^2\right\}\\
\geq  & \frac{1}{t} \sum_{i=1}^{t'} (\mathbf{w}^\top A
\mathbf{x}_{\bar{j}_i})^2 -\frac{2}{t}\sum_{i=1}^t |\mathbf{w}^\top
A \mathbf{x}_i||\mathbf{w}^\top \mathbf{n}_i|\\ \geq &
(1-\epsilon_1) \|\mathbf{w}^\top A\|^2
\mathcal{V}(t'/t)-2\|\mathbf{w}^\top
A\|\sqrt{(1+\epsilon_2)\overline{c}},
\end{split}\end{equation*}
where $(a)$ follows from the definition of $\bar{j}_i$.
\end{proof}

\begin{corollary}\label{cor.step1dgeneralt} Let $t' \leq t$. If there exists $\epsilon_1, \epsilon_2, \overline{c}$ such that
\begin{equation*}\begin{split}
(I)\quad & \sup_{\mathbf{w}\in \mathcal{S}_d} \big|\frac{1}{t}
\sum_{i=1}^{t'}|\mathbf{w}^\top
\mathbf{x}|_{(i)}^2-\mathcal{V}(\frac{t'}{t})\big|\leq
\epsilon_1\\(II)\quad & \sup_{\mathbf{w}\in \mathcal{S}_d}
\big|\frac{1}{t} \sum_{i=1}^{t}|\mathbf{w}^\top
\mathbf{x}_i|^2-1\big|\leq \epsilon_2\\
(III)\quad &\sup_{\mathbf{w}\in\mathcal{S}_m} \frac{1}t
\sum_{i=1}^{t}|\mathbf{w}^\top \mathbf{n}_{i}|^2\leq
\overline{c},\end{split}\end{equation*} then for any
$\mathbf{w}_1,\cdots, \mathbf{w}_d\in
\mathcal{S}_m$ the following holds
\begin{equation*}\begin{split}
&(1-\epsilon_1)\mathcal{V}\left(\frac{t'}{t}\right)H(\mathbf{w}_1,\cdots,
\mathbf{w}_d)
 -2\sqrt{(1+\epsilon_2)\overline{c}d H(\mathbf{w}_1,\cdots, \mathbf{w}_d)}\\
\leq & \sum_{j=1}^d\frac{1}{t}\sum_{i=1}^{t'} |\mathbf{w}_j^\top
\mathbf{z}|_{(i)}^2 \\
\leq &
(1+\epsilon_1)\mathcal{V}\left(\frac{t'}{t}\right)H(\mathbf{w}_1,\cdots,
\mathbf{w}_d)
+2\sqrt{(1+\epsilon_2)\overline{c}dH(\mathbf{w}_1,\cdots,
\mathbf{w}_d)}+\overline{c},
\end{split}\end{equation*}
where $H(\mathbf{w}_1,\cdots, \mathbf{w}_d)\triangleq\sum_{j=1}^d
\|\mathbf{w}_j^\top A\|^2$.
\end{corollary}
\begin{proof}From Theorem~\ref{thm.step1d}, we have that
\[\sum_{j=1}^d(1-\epsilon_1)\|\mathbf{w}_j^\top A\|^2
\mathcal{V}\left(\frac{t'}{t}\right)
-2\sum_{j=1}^d\|\mathbf{w}_j^\top
A\|\sqrt{(1+\epsilon_2)\overline{c}} \leq \sum_{j=1}^d
\frac{1}{t}\sum_{i=1}^{t'} |\mathbf{w}^\top \mathbf{z}|_{(i)}^2.\]
Note that $\sum_{j=1}^d a_j \leq \sqrt{d\sum_{j=1}^d a_j^2}$ holds
for any $a_1, \cdots, a_d$, we have
\begin{equation*}\begin{split}&(1-\epsilon_1)\mathcal{V}\left(\frac{t'}{t}\right)H(\mathbf{w}_1,\cdots,
\mathbf{w}_d)
 -2\sqrt{(1+\epsilon_2)\overline{c}d H(\mathbf{w}_1,\cdots,
\mathbf{w}_d)}\\ \leq &
\sum_{j=1}^d(1-\epsilon_1)\|\mathbf{w}_j^\top A\|^2
\mathcal{V}\left(\frac{t'}{t}\right)
-2\sum_{j=1}^d\|\mathbf{w}_j^\top
A\|\sqrt{(1+\epsilon_2)\overline{c}},\end{split}\end{equation*}
which proves the first inequality of the lemma. The second one
follows similarly.
\end{proof}

Letting $t'=t$ we immediately have the following corollary.
\begin{corollary}\label{cor.step1dtis1} If there exists $\epsilon, \overline{c}$ such that
\begin{equation*}\begin{split}
(I)\quad & \sup_{\mathbf{w}\in \mathcal{S}_d} \big|\frac{1}{t}
\sum_{i=1}^{t}|\mathbf{w}^\top
\mathbf{x}|^2-1\big|\leq \epsilon\\
(II)\quad &\sup_{\mathbf{w}\in\mathcal{S}_m} \frac{1}t
\sum_{i=1}^{t}|\mathbf{w}^\top \mathbf{n}_{i}|^2\leq
\overline{c},\end{split}\end{equation*} then for any
$\mathbf{w}_1,\cdots, \mathbf{w}_d\in \mathcal{S}_m$ the following
holds:
\begin{equation*}\begin{split}
&(1-\epsilon)H(\mathbf{w}_1,\cdots, \mathbf{w}_d)
 -2\sqrt{(1+\epsilon)\overline{c}d H(\mathbf{w}_1,\cdots, \mathbf{w}_d)}\\
\leq & \sum_{j=1}^d\frac{1}{t}\sum_{i=1}^{t} |\mathbf{w}_j^\top
\mathbf{z}_i|^2 \\
\leq & (1+\epsilon)H(\mathbf{w}_1,\cdots, \mathbf{w}_d)
+2\sqrt{(1+\epsilon)\overline{c}dH(\mathbf{w}_1,\cdots,
\mathbf{w}_d)}+\overline{c}.
\end{split}\end{equation*}
\end{corollary}

\subsection{Step 2}
The next step shows that the algorithm finds a good solution in a small number of steps. Proving this involves showing that at any given step, either the algorithm finds a good solution, or the random removal eliminates one of the corrupted points with high probability (i.e., probability bounded away from zero). The intuition then, is that there cannot be too many steps without finding a good solution, since too many of the corrupted points will have been removed. This section makes this intuition precise.

Let us fix a $\kappa>0$. Let $\mathcal{Z}(s)$ and $\mathcal{O}(s)$ be the set of remaining authentic samples and the set of remaining corrupted points after the $s^{th}$ stage, respectively. Then with this notation,
$\mathcal{Y}(s)=\mathcal{Z}(s)\bigcup \mathcal{O}(s)$. Observe that $|\mathcal{Y}(s)|=n-s$. Let $\overline{r}(s)
=\mathcal{Y}(s-1)\backslash \mathcal{Y}(s)$, i.e., the point removed at stage $s$. Let $\mathbf{w}_1(s),\dots, \mathbf{w}_d(s)$ be the $d$ PCs found in the $s^{th}$ stage --- these points are the output of standard PCA on $\mathcal{Y}(s-1)$. These points are a good solution if the variance of the points projected onto their span is mainly due to the authentic samples rather than the corrupted points. We denote this ``good output event at step $s$'' by $\mathcal{E}(s)$, defined as follows:
$$
\mathcal{E}(s)=\{\sum_{j=1}^d \sum_{\mathbf{z}_i\in
\mathcal{Z}(s-1)} (\mathbf{w}_j(s)^{\top} \mathbf{z}_i)^2 \geq
\frac{1}{\kappa} \sum_{j=1}^d \sum_{\mathbf{o}_i\in
\mathcal{O}(s-1)} (\mathbf{w}_j(s)^{\top} \mathbf{o}_i)^2\}.
$$
We show in the next theorem that with high probability, $\mathcal{E}(s)$ is true for at
least one ``small'' $s$, by showing that at every $s$ where it is not true, the random removal procedure removes a corrupted point with probability at least $\kappa/(1 + \kappa)$.

\begin{theorem}\label{thm.Es} With probability at least $1-\gamma$, event $\mathcal{E}(s)$ is true for some $1 \leq s \leq s_0$, where
$$
s_0\triangleq(1+\epsilon)\frac{(1+\kappa)\lambda n}{\kappa};\quad\epsilon=\frac{16(1+\kappa)\log(1/\gamma)}{\kappa \lambda n}+4\sqrt{\frac{(1+\kappa)\log(1/\gamma)}{\kappa \lambda n}}.
$$
\end{theorem}
\begin{remark} {\rm When $\kappa$ and $\lambda$ are fixed,  we have $s_0/n
\rightarrow (1+\kappa)\lambda/\kappa$. Therefore, $s_0\leq t$ for
$(1+\kappa)\lambda<\kappa(1-\lambda)$ and $n$ large.}\end{remark}

When $s_0\geq n$, Theorem~\ref{thm.Es} holds trivially. Hence we
focus on the case where $s_0< n$.  En route to proving this theorem, we first prove that when $\mathcal{E}(s)$ is not true, our procedure removes a corrupted point with high probability. To this end, let $\mathcal{F}_s$ be the filtration generated by the set of events
until stage $s$. Observe that $\mathcal{O}(s), \mathcal{Z}(s),
\mathcal{Y}(s) \in \mathcal{F}_s$. Furthermore, since given
$\mathcal{Y}(s)$, performing a PCA is deterministic,
$\mathcal{E}(s+1)\in \mathcal{F}_s$.
\begin{theorem}\label{thm.probabilitytoremoveanoutlier} If $\mathcal{E}^c(s)$ is true, then
$$
\mathrm{Pr}(\{\overline{r}(s)\in \mathcal{O}(s-1)\}|\mathcal{F}_{s-1})>\frac{\kappa}{1+\kappa}.
$$
\end{theorem}
\begin{proof}If $\mathcal{E}^c(s)$ is true, then
\[\sum_{j=1}^d \sum_{\mathbf{z}_i\in
\mathcal{Z}(s-1)} (\mathbf{w}_j(s)^{\top} \mathbf{z}_i)^2 <
\frac{1}{\kappa} \sum_{j=1}^d \sum_{\mathbf{o}_i\in
\mathcal{O}(s-1)} (\mathbf{w}_j(s)^{\top} \mathbf{o}_i)^2,\] which
is equivalent to
\[ \frac{\kappa}{1+\kappa}\big[\sum_{\mathbf{z}_i\in
\mathcal{Z}(s-1)} \sum_{j=1}^d(\mathbf{w}_j(s)^{\top}
\mathbf{z}_i)^2 + \sum_{\mathbf{o}_i\in \mathcal{O}(s-1)}
\sum_{j=1}^d(\mathbf{w}_j(s)^{\top} \mathbf{o}_i)^2\big] <
\sum_{\mathbf{o}_i\in \mathcal{O}(s-1)}
\sum_{j=1}^d(\mathbf{w}_j(s)^{\top} \mathbf{o}_i)^2.\] Note that
\begin{equation*}\begin{split}&\mathrm{Pr}(\{\overline{r}(s)\in
\mathcal{O}(s-1)\}|\mathcal{F}_{s-1})\\ = &\sum_{\mathbf{o}_i\in
\mathcal{O}(s-1)}
\mathrm{Pr}(\overline{r}(s)=\mathbf{o}_i|\mathcal{F}_{s-1})\\
=&\sum_{\mathbf{o}_i\in \mathcal{O}(s-1)}
\frac{\sum_{j=1}^d(\mathbf{w}_j(s)^{\top}
\mathbf{o}_i)^2}{\sum_{\mathbf{z}_i\in \mathcal{Z}(s-1)}
\sum_{j=1}^d(\mathbf{w}_j(s)^{\top} \mathbf{z}_i)^2 +
\sum_{\mathbf{o}_i\in \mathcal{O}(s-1)}
\sum_{j=1}^d(\mathbf{w}_j(s)^{\top}
\mathbf{o}_i)^2}\\
>&\frac{\kappa}{1+\kappa}.\end{split}\end{equation*} Here, the second
equality follows from the definition of the algorithm, and in
particular, that in stage $s$, we remove a point $\mathbf{y}$ with
 probability proportional to
$\sum_{j=1}^d(\mathbf{w}_j(s)^{\top} \mathbf{y})^2$, and independent
to other events.
\end{proof}

As a consequence of this theorem, we can now prove
Theorem~\ref{thm.Es}. The intuition is rather straightforward: if
the events were independent from one step to the next, then since
``expected corrupted points removed''   is at least
$\kappa/(1+\kappa)$, then after $s_0 = (1+\epsilon)(1+\kappa)\lambda
n/\kappa$ steps, with exponentially high probability all the
outliers would be removed, and hence we would have a good event with
high probability, for some $s \leq s_0$. Since subsequent steps are
not independent, we have to rely on martingale arguments.

Let $T=\min\{s|\mathcal{E}(s)\mbox{ is true}\}$. Note that since $\mathcal{E}(s)\in \mathcal{F}_{s-1}$, we have $\{T > s\}\in \mathcal{F}_{s-1}$. Define the following random variable
$$
X_s=\left\{\begin{array}{ll}  |\mathcal{O}(T-1)|+\frac{\kappa(T-1)}{1+\kappa},&\mbox{if} \,\, T \leq s; \\
|\mathcal{O}(s)|+\frac{\kappa s}{1+\kappa}, & \mbox{if} \,\, T>s.
\end{array}\right.
$$
\begin{lemma}
\label{le.supermartingale} $\{X_s, \mathcal{F}_s\}$ is a supermartingale.
\end{lemma}
\begin{proof} The proof essentially follows from the definition of $X_s$, and the fact that if $\mathcal{E}(s)$ is true, then $|\mathcal{O}(s)|$ decreases by one with probability $\kappa/(1 + \kappa)$. The full details are deferred to the appendix.




\end{proof}

From here, the proof of Theorem~\ref{thm.Es} follows fairly quickly.

\begin{proof}
Note that
\begin{equation}\label{equ.proofofE}\mathrm{Pr}\left(\bigcap_{s=1}^{s_0}
\mathcal{E}(s)^c\right)=\mathrm{Pr}\left(T>s_0\right)\leq
\mathrm{Pr}\left(X_{s_0}\geq \frac{\kappa
s_0}{1+\kappa}\right)=\mathrm{Pr}\left(X_{s_0}\geq
(1+\epsilon)\lambda n\right),\end{equation} where the inequality is
due to $|\mathcal{O}(s)|$ being non-negative. Recall that $X_0=\lambda n$. Thus the probability that no good events occur before step $s_0$ is at most the probability that a supermartingale with bounded incremements increases in value by a constant factor of $(1 + \epsilon)$, from $\lambda n$ to $(1 + \epsilon)\lambda n$. An appeal to Azuma's inequality shows that this is exponentially unlikely. The details are left to the appendix.
\end{proof}

\subsection{Step 3}
\label{ssec.step3}
Let $\overline{\mathbf{w}}_1,\dots, \overline{\mathbf{w}}_d$ be the
eigenvectors corresponding to the $d$ largest eigenvalues of
$AA^\top$, i.e., the optimal solution. Let $\mathbf{w}^*_1,\dots,
\mathbf{w}^*_d$ be the output of the algorithm. Let
$\mathbf{w}_1(s), \dots, \mathbf{w}_d(s)$ be the candidate solution
at stage $s$. Recall that $H(\mathbf{w}_1,\cdots, \mathbf{w}_d)\triangleq\sum_{j=1}^d
\|\mathbf{w}_j^\top A\|^2$, and for notational simplification, let $\overline{H}\triangleq
H(\overline{\mathbf{w}}_1,\cdots, \overline{\mathbf{w}}_d)$,
$H_s\triangleq H(\mathbf{w}_1(s),\dots, \mathbf{w}_d(s))$, and
$H^*\triangleq H(\mathbf{w}^*_1,\dots, \mathbf{w}^*_d)$.

The statement of the finite-sample and asymptotic theorems (Theorems \ref{thm.main} and \ref{thm.asymptotic}, respectively) lower bound the expressed variance, E.V., which is the ratio $H^*/\overline{H}$. The final part of the proof accomplishes this in two main steps. First, Lemma~\ref{lem.step6hsvsoverlineh} lower bounds $H_s$ in terms of $\overline{H}$, where $s$ is some step for which $\mathcal{E}(s)$ is true, i.e., the principal components found by the $s^{th}$ step of the algorithm are ``good.'' By Theorem~\ref{thm.Es}, we know that there is a ``small'' such $s$, with high probability. The final output of the algorithm, however, is only guaranteed to have a high value of the robust variance estimator, $\overline{V}$ --- that is, even if there is a ``good'' solution at some intermediate step $s$, we do not necessarily have a way of identifying it. Thus, the next step, Lemma~\ref{lem.step6hsvshstar}, lower bounds the value of $H^*$ in terms of the value $H$ of {\it any} output $\mathbf{w}_1',\dots,\mathbf{w}_d'$ that has a smaller value of the robust variance estimator.

We give the statement of all the intermediate results, leaving the details of the proof to the appendix.
\begin{lemma}\label{lem.step6hsvsoverlineh}If $\mathcal{E}(s)$ is true for some $s\leq
s_0$, and there exists $\epsilon_1, \epsilon_2, \overline{c}$ such
that
\begin{equation*}\begin{split}
(I)\quad & \sup_{\mathbf{w}\in \mathcal{S}_d} \big|\frac{1}{t}
\sum_{i=1}^{t-s_0}|\mathbf{w}^\top
\mathbf{x}|_{(i)}^2-\mathcal{V}\left(\frac{t-s_0}{t}\right)\big|\leq
\epsilon_1\\(II)\quad & \sup_{\mathbf{w}\in \mathcal{S}_d}
\big|\frac{1}{t} \sum_{i=1}^{t}|\mathbf{w}^\top
\mathbf{x}_i|^2-1\big|\leq \epsilon_2\\
(III)\quad &\sup_{\mathbf{w}\in\mathcal{S}_m} \frac{1}t
\sum_{i=1}^{t}|\mathbf{w}^\top \mathbf{n}_{i}|^2\leq
\overline{c},\end{split}\end{equation*} then
\[\frac{1}{1+\kappa}\left[(1-\epsilon_1)\mathcal{V}\left(\frac{t-s_0}{t}\right)\overline{H}
 -2\sqrt{(1+\epsilon_2)\overline{c}d \overline{H}}\right]\leq
(1+\epsilon_2)H_s
+2\sqrt{(1+\epsilon_2)\overline{c}dH_s}+\overline{c}.\]
\end{lemma}

\begin{lemma}\label{lem.step6hsvshstar}Fix a $\hat{t} \leq t$. If $\sum_{j=1}^d\overline{V}_{\hat{t}}(\mathbf{w}_j)
\geq \sum_{j=1}^d\overline{V}_{\hat{t}}(\mathbf{w}'_j)$,  and there
exists $\epsilon_1, \epsilon_2, \overline{c}$ such that
\begin{equation*}\begin{split}
(I)\quad & \sup_{\mathbf{w}\in \mathcal{S}_d} \big|\frac{1}{t}
\sum_{i=1}^{\hat{t}}|\mathbf{w}^\top
\mathbf{x}|_{(i)}^2-\mathcal{V}(\frac{\hat{t}}{t})\big|\leq
\epsilon_1,\\(II)\quad & \sup_{\mathbf{w}\in \mathcal{S}_d}
\big|\frac{1}{t} \sum_{i=1}^{\hat{t}-\frac{\lambda
t}{1-\lambda}}|\mathbf{w}^\top
\mathbf{x}|_{(i)}^2-\mathcal{V}\left(\frac{\hat{t}}{t}-\frac{\lambda}{1-\lambda}\right)\big|\leq
\epsilon_1,\\(III)\quad & \sup_{\mathbf{w}\in \mathcal{S}_d}
\big|\frac{1}{t} \sum_{i=1}^{t}|\mathbf{w}^\top
\mathbf{x}_i|^2-1\big|\leq \epsilon_2,\\
(IV)\quad &\sup_{\mathbf{w}\in\mathcal{S}_m} \frac{1}t
\sum_{i=1}^{t}|\mathbf{w}^\top \mathbf{n}_{i}|^2\leq
\overline{c},\end{split}\end{equation*} then
\begin{equation*}\begin{split}&(1-\epsilon_1)\mathcal{V}\left(\frac{\hat{t}}{t}-\frac{\lambda}{1-\lambda}\right)H(\mathbf{w}_1'\cdots, \mathbf{w}_d')
 -2\sqrt{(1+\epsilon_2)\overline{c}d H(\mathbf{w}_1'\cdots,
\mathbf{w}_d')}\\\leq & (1+\epsilon_1)H(\mathbf{w}_1\cdots,
\mathbf{w}_d)\mathcal{V}\left(\frac{\hat{t}}{t}\right)
+2\sqrt{(1+\epsilon_2)\overline{c}dH(\mathbf{w}_1\cdots,
\mathbf{w}_d)}+\overline{c}.\end{split}\end{equation*}
\end{lemma}

\begin{theorem}
\label{thm.step3ugly} If $\bigcup_{s=1}^{s_0}\mathcal{E}(s)$ is true, and there
exists $\epsilon_1<1, \epsilon_2, \overline{c}$ such that
\begin{equation*}\begin{split}
(I)\quad & \sup_{\mathbf{w}\in \mathcal{S}_d} \big|\frac{1}{t}
\sum_{i=1}^{t-s_0}|\mathbf{w}^\top
\mathbf{x}|_{(i)}^2-\mathcal{V}(\frac{t-s_0}{t})\big|\leq
\epsilon_1\\
(II)\quad & \sup_{\mathbf{w}\in \mathcal{S}_d} \big|\frac{1}{t}
\sum_{i=1}^{\hat{t}}|\mathbf{w}^\top
\mathbf{x}|_{(i)}^2-\mathcal{V}(\frac{\hat{t}}{t})\big|\leq
\epsilon_1\\(III)\quad & \sup_{\mathbf{w}\in \mathcal{S}_d}
\big|\frac{1}{t} \sum_{i=1}^{\hat{t}-\frac{\lambda
t}{1-\lambda}}|\mathbf{w}^\top
\mathbf{x}|_{(i)}^2-\mathcal{V}\left(\frac{\hat{t}}{t}-\frac{\lambda
}{1-\lambda}\right)\big|\leq \epsilon_1\\(IV)\quad &
\sup_{\mathbf{w}\in \mathcal{S}_d} \big|\frac{1}{t}
\sum_{i=1}^{t}|\mathbf{w}^\top
\mathbf{x}_i|^2-1\big|\leq \epsilon_2\\
(V)\quad &\sup_{\mathbf{w}\in\mathcal{S}_m} \frac{1}t
\sum_{i=1}^{t}|\mathbf{w}^\top \mathbf{n}_{i}|^2\leq
\overline{c},\end{split}\end{equation*} then
\begin{equation}\begin{split}\label{equ.step6noprobabilityintext}&\frac{H^*}{\overline{H}} \geq \frac{(1-\epsilon_1)^2 \mathcal{V}\left(\frac{\hat{t}}{t}-\frac{\lambda}{1-\lambda}\right)\mathcal{V}\left(\frac{t-s_0}{t}\right)}{(1+\epsilon_1)(1+\epsilon_2)(1+\kappa)\mathcal{V}\left(\frac{\hat{t}}{t}\right)}\\
&\quad-\left[\frac{(2\kappa+4)(1-\epsilon_1)\mathcal{V}\left(\frac{\hat{t}}{t}-\frac{\lambda}{1-\lambda}\right)\sqrt{(1+\epsilon_2)\overline{c}d}+4(1+\kappa)(1+\epsilon_2)\sqrt{(1+\epsilon_2)\overline{c}d}}{(1+\epsilon_1)(1+\epsilon_2)(1+\kappa)\mathcal{V}\left(\frac{\hat{t}}{t}\right)}\right](\overline{H})^{-1/2}\\
&\quad -\left[\frac{(1-\epsilon_1)\mathcal{V}\left(\frac{\hat{t}}{t}-\frac{\lambda}{1-\lambda}\right)\overline{c}+(1+\epsilon_2)\overline{c}}{(1+\epsilon_1)(1+\epsilon_2)\mathcal{V}\left(\frac{\hat{t}}{t}\right)}\right](\overline{H})^{-1}.
\end{split}\end{equation}
\end{theorem}

By bounding all diminishing terms in the r.h.s.
of~(\ref{equ.step6noprobabilityintext}), we can reformulate the
above theorem in a slightly more palatable form, as stated in
Theorem~\ref{thm.main}:

\emph{Theorem~\ref{thm.main}} Let $\tau=\max(m/n,1)$. There exists a
universal constant $c_0$ and a constant $C$ which can possible
depend on $\hat{t}/t$, $\lambda$, $d$, $\mu$ and $\kappa$, such that
for any $\gamma <1$, if $n/\log^4 n\geq \log^6(1/\gamma)$, then with
probability $1-\gamma$ the following holds
\begin{equation*}\begin{split}\frac{H^*}{\overline{H}}\geq \frac{
\mathcal{V}\left(\frac{\hat{t}}{t}
-\frac{\lambda}{1-\lambda}\right)\mathcal{V}\left(1-\frac{\lambda(1+\kappa)}{(1-\lambda)\kappa}\right)}
{(1+\kappa)\mathcal{V}\left(\frac{\hat{t}}{t}\right)}-\left[\frac{8\sqrt{c_0\tau
d}}{\mathcal{V}\left(\frac{\hat{t}}{t}\right)}\right](\overline{H})^{-1/2}
-\left[\frac{2c_0\tau}{\mathcal{V}\left(\frac{\hat{t}}{t}\right)}\right](\overline{H})^{-1}
-C \frac{\log^2 n \log^3(1/\gamma)}{\sqrt{n}}.
\end{split}\end{equation*}

We immediately get the asymptotic bound of Theorem~\ref{thm.asymptotic} as a corollary:

{\emph{Theorem \ref{thm.asymptotic}}} The asymptotical performance of HR-PCA is given by
$$
\frac{H^*}{\overline{H}}\geq \max_{\kappa} \frac{\mathcal{V}\left(\frac{\hat{t}}{t}-\frac{\lambda}{1-\lambda}\right)\mathcal{V}\left(1-\frac{\lambda(1+\kappa)}{(1-\lambda)\kappa}\right)}{(1+\kappa)\mathcal{V}\left(\frac{\hat{t}}{t}\right)}.
$$

\section{Numerical Illustrations}\label{sec.numerical}
We report in this section some numerical results on synthetic data
of the proposed algorithm. We compare its performance with standard
PCA, and several robust PCA algorithms, namely Multi-Variate
iterative Trimming (MVT), ROBPCA proposed in
\cite{HubertRousseeuwBranden05}, and the (approximate)
Project-Pursuit (PP) algorithm proposed in
\cite{CrouxRuiz-Gazen05}. One objective of this numerical study is
to illustrate how the special properties of the high-dimensional
regime discussed in Section~\ref{sec:background} can degrade the
performance of available robust PCA algorithms, and make some of
them completely invalid.

We report the $d=1$ case first. We randomly generate an  $m\times 1$
matrix and scale it so that its leading eigenvalue has magnitude
equal to a given $\sigma$.  A $\lambda$ fraction of outliers are
generated on a line with a uniform distribution over $[-\sigma
\cdot \rm{mag}, \sigma \cdot \rm{mag}]$. Thus, $\rm{mag}$ represents the ratio
between the magnitude of the outliers and that of the signal
$A\mathbf{x}_i$. For each parameter setup, we report the average
result of $20$ tests. The MVT algorithm breaks down in the $n=m$
case since it involves taking the inverse of the covariance matrix
which is ill-conditioned. Hence we do not report MVT results in any
of the experiments with $n=m$, as shown in Figure~\ref{fig.1d} and
perform a separate test for MVT, HR-PCA and PCA under the case that
$m\ll n$ reported in Figure~\ref{fig.mvt}.

\begin{figure}[htbp!]
\begin{center}
\begin{tabular}{cc}
  \includegraphics[height=4.5cm,
  width=0.48\linewidth]{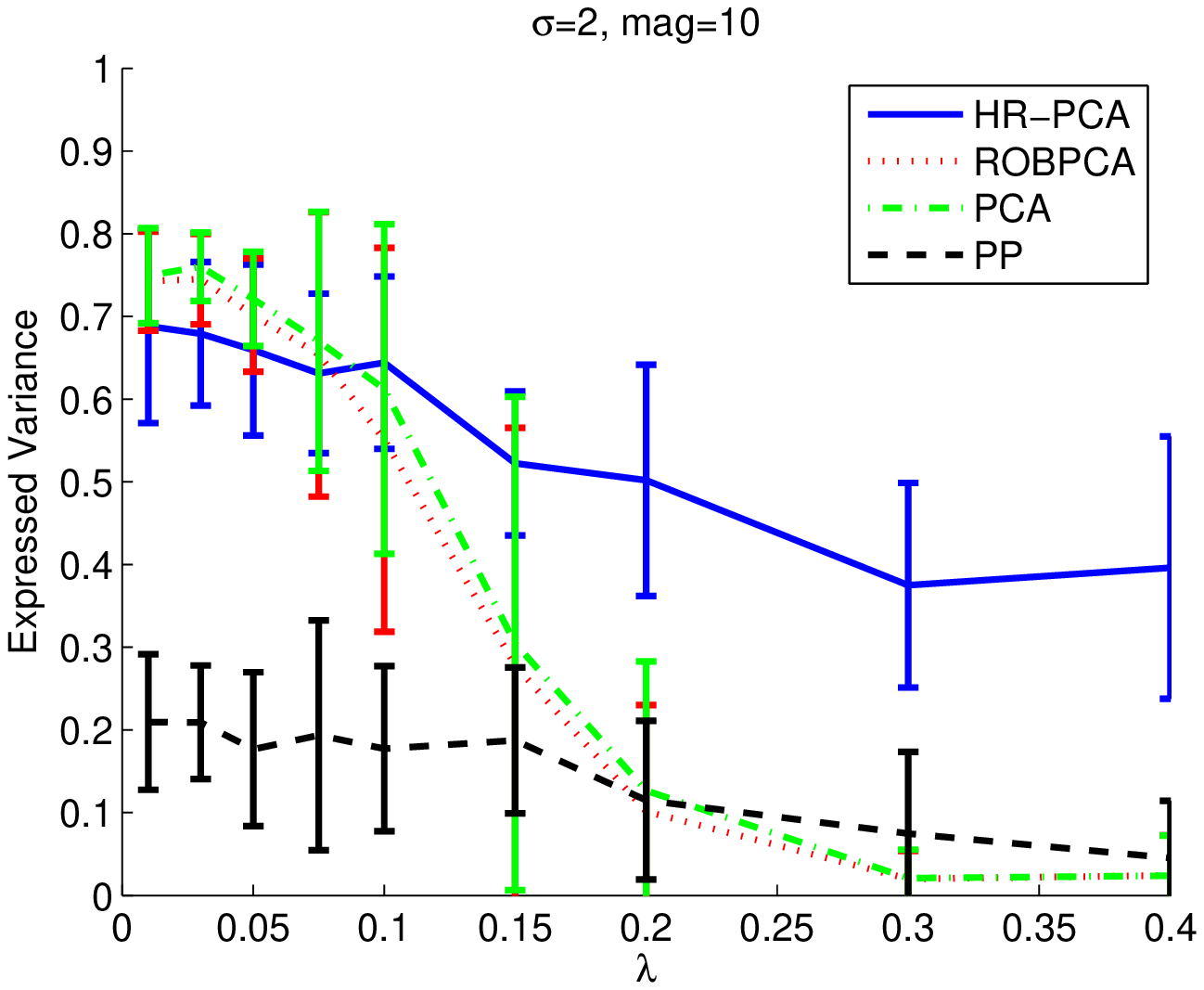}
  & \includegraphics[height=4.5cm,
  width=0.48\linewidth]{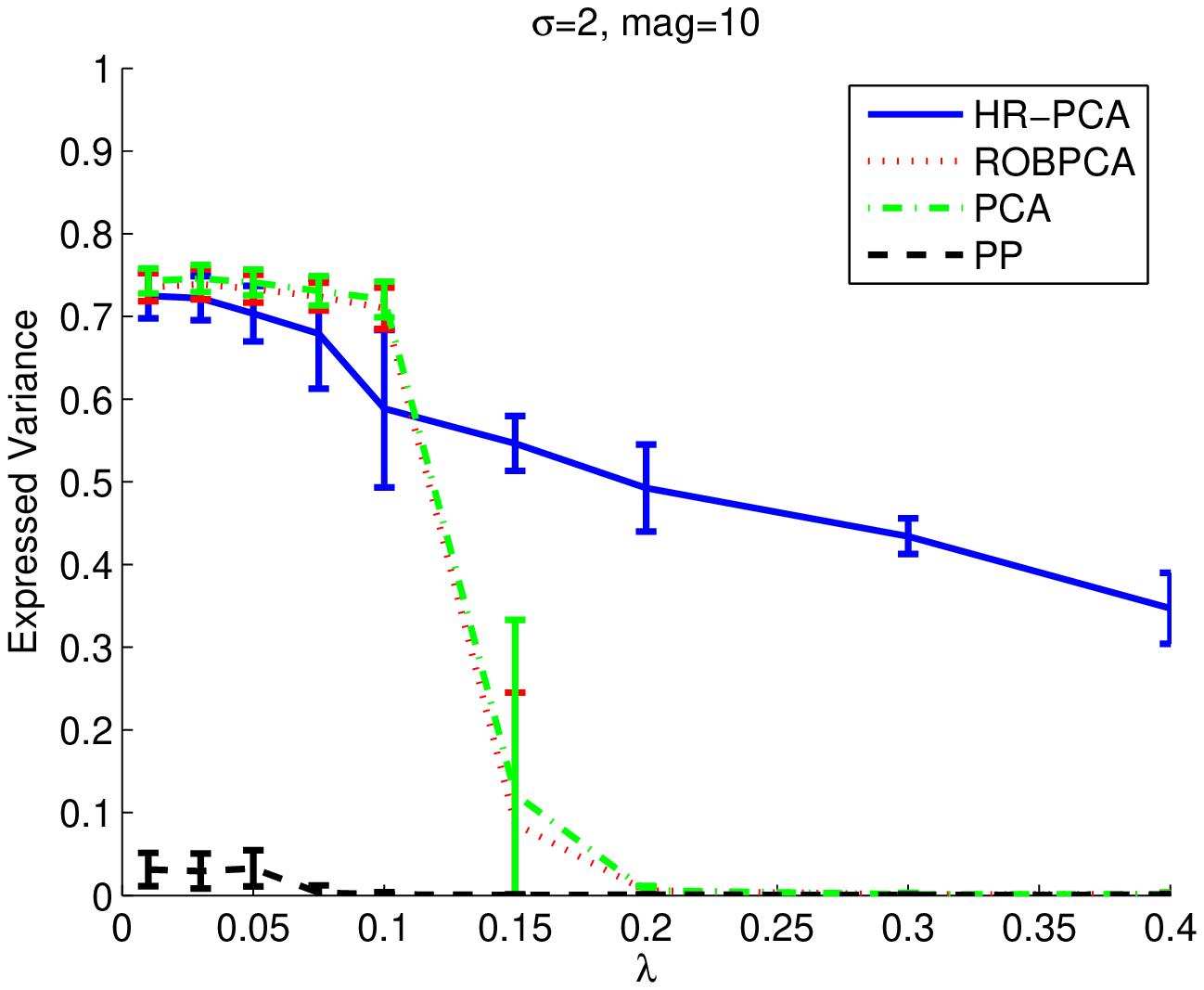}\\
  \includegraphics[height=4.5cm,
  width=0.48\linewidth]{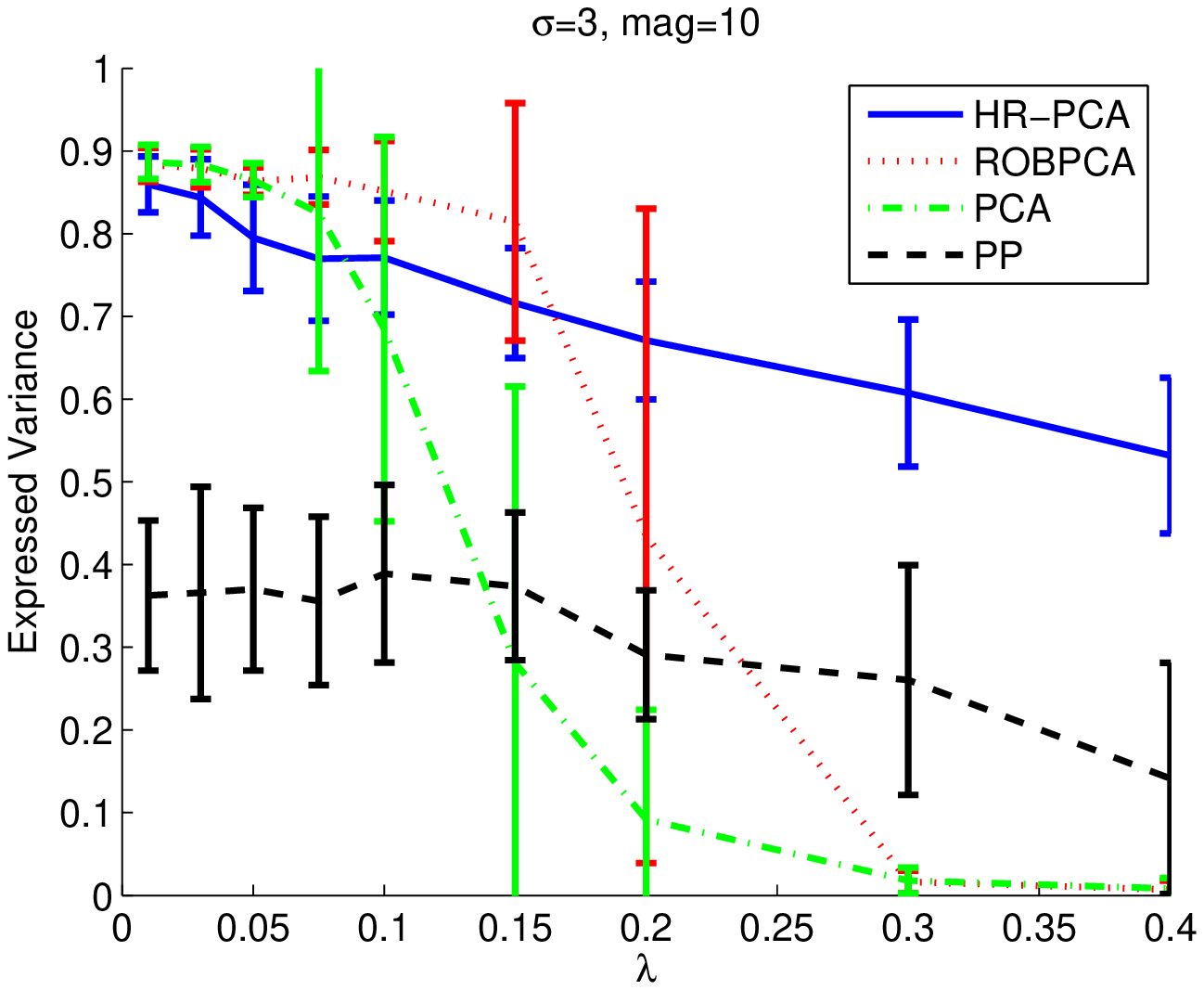} &
  \includegraphics[height=4.5cm,
  width=0.48\linewidth]{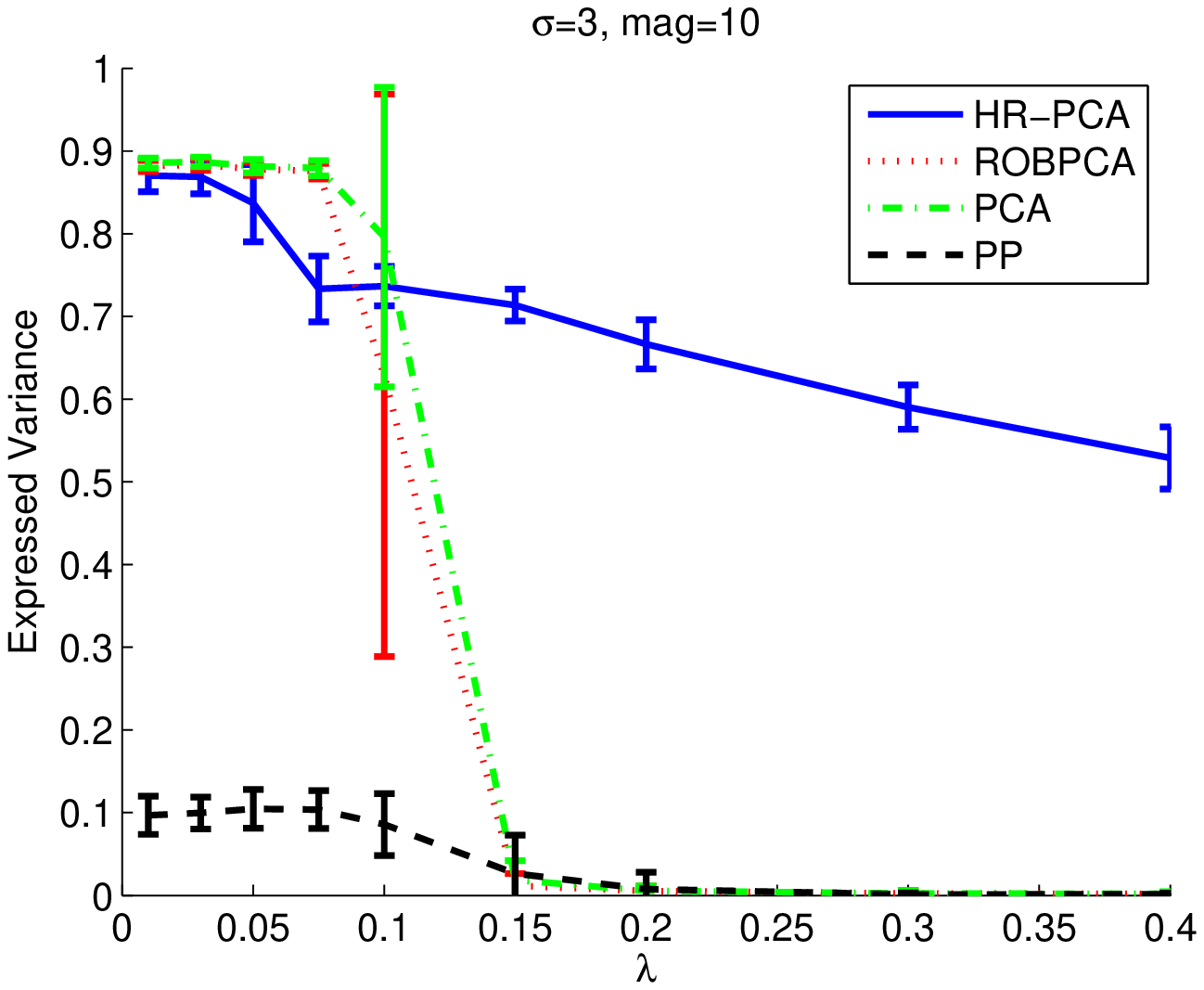} \\
    \includegraphics[height=4.5cm,
  width=0.48\linewidth]{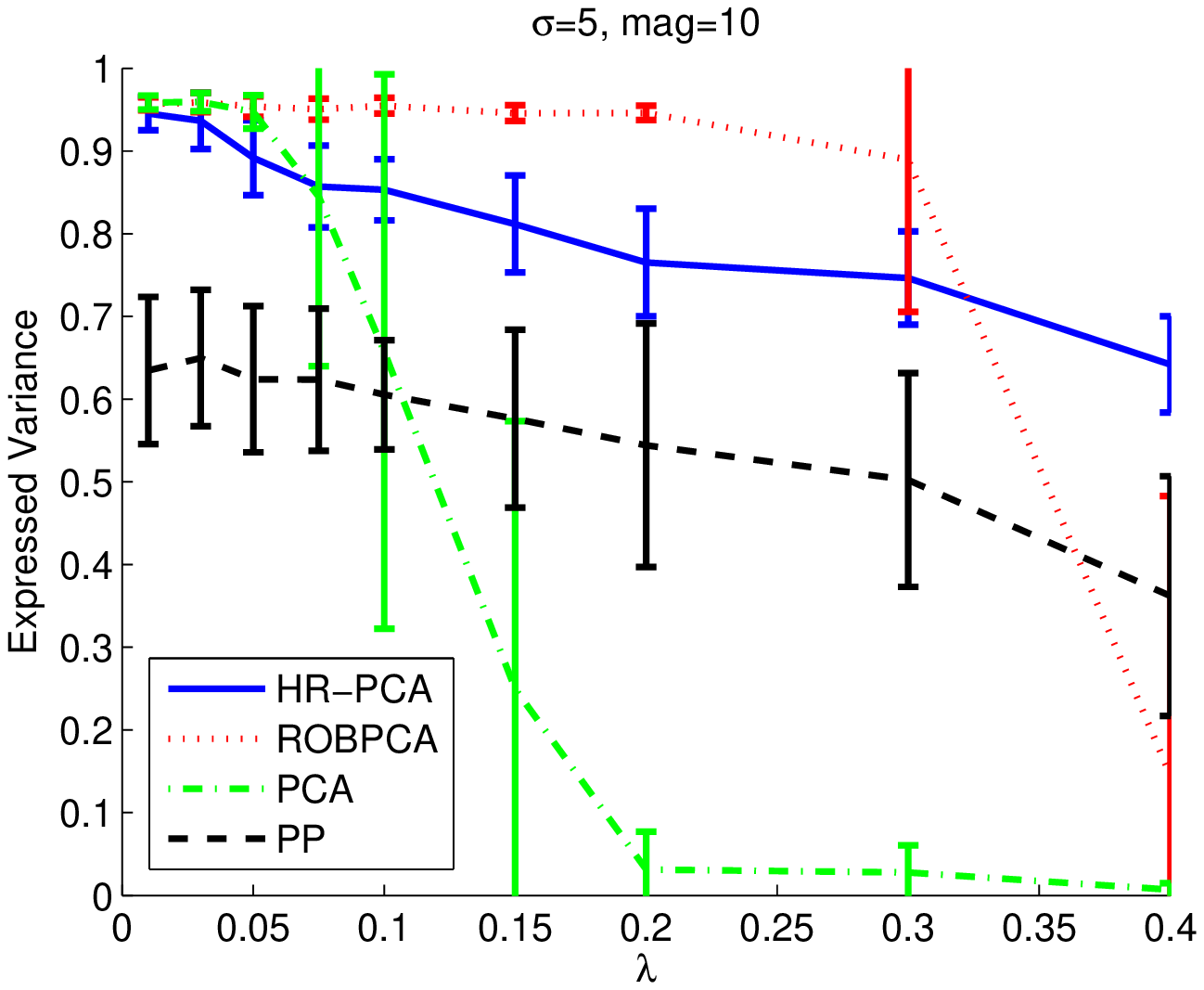}
  &      \includegraphics[height=4.5cm,
  width=0.48\linewidth]{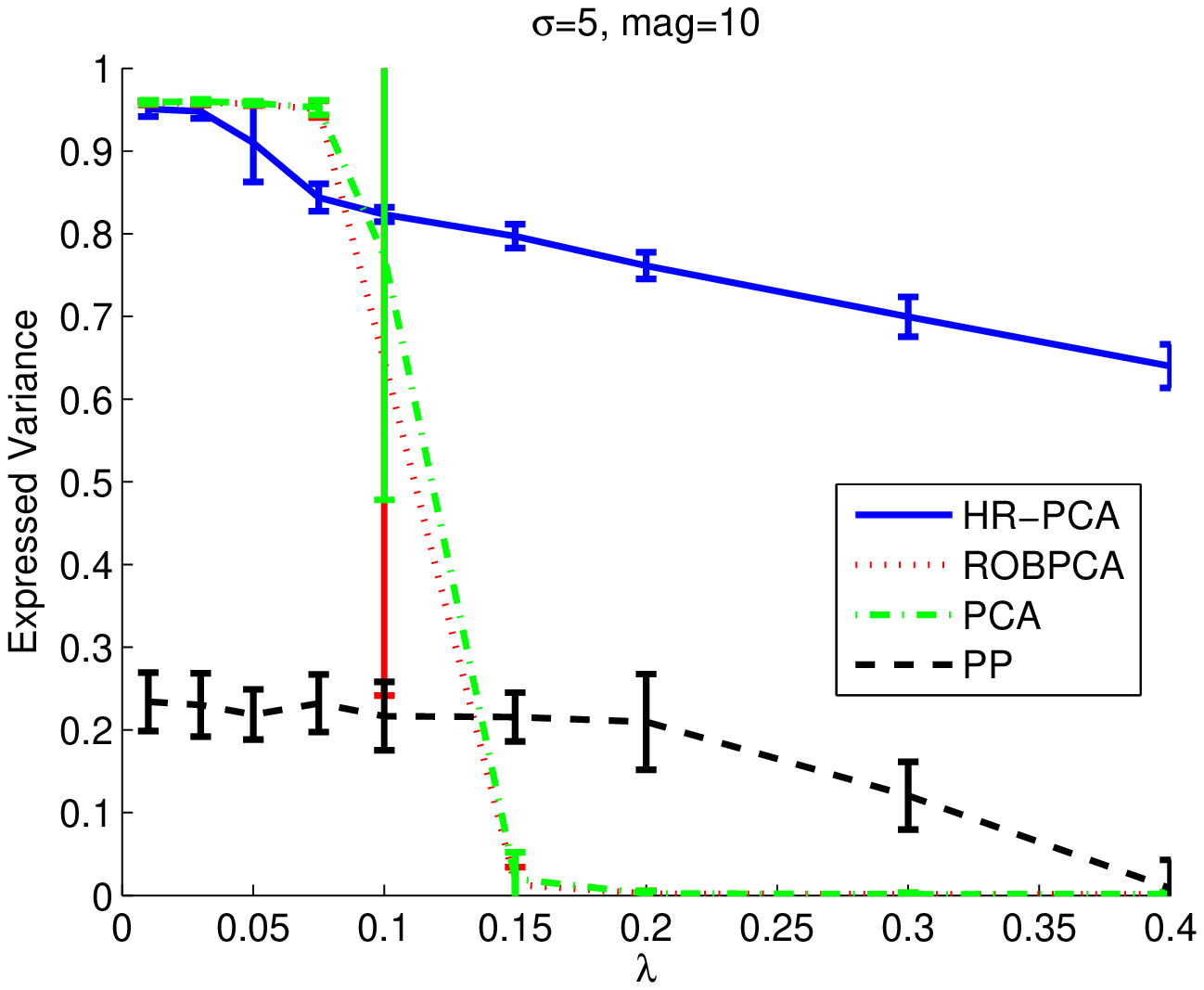}\\
  \includegraphics[height=4.5cm,
  width=0.48\linewidth]{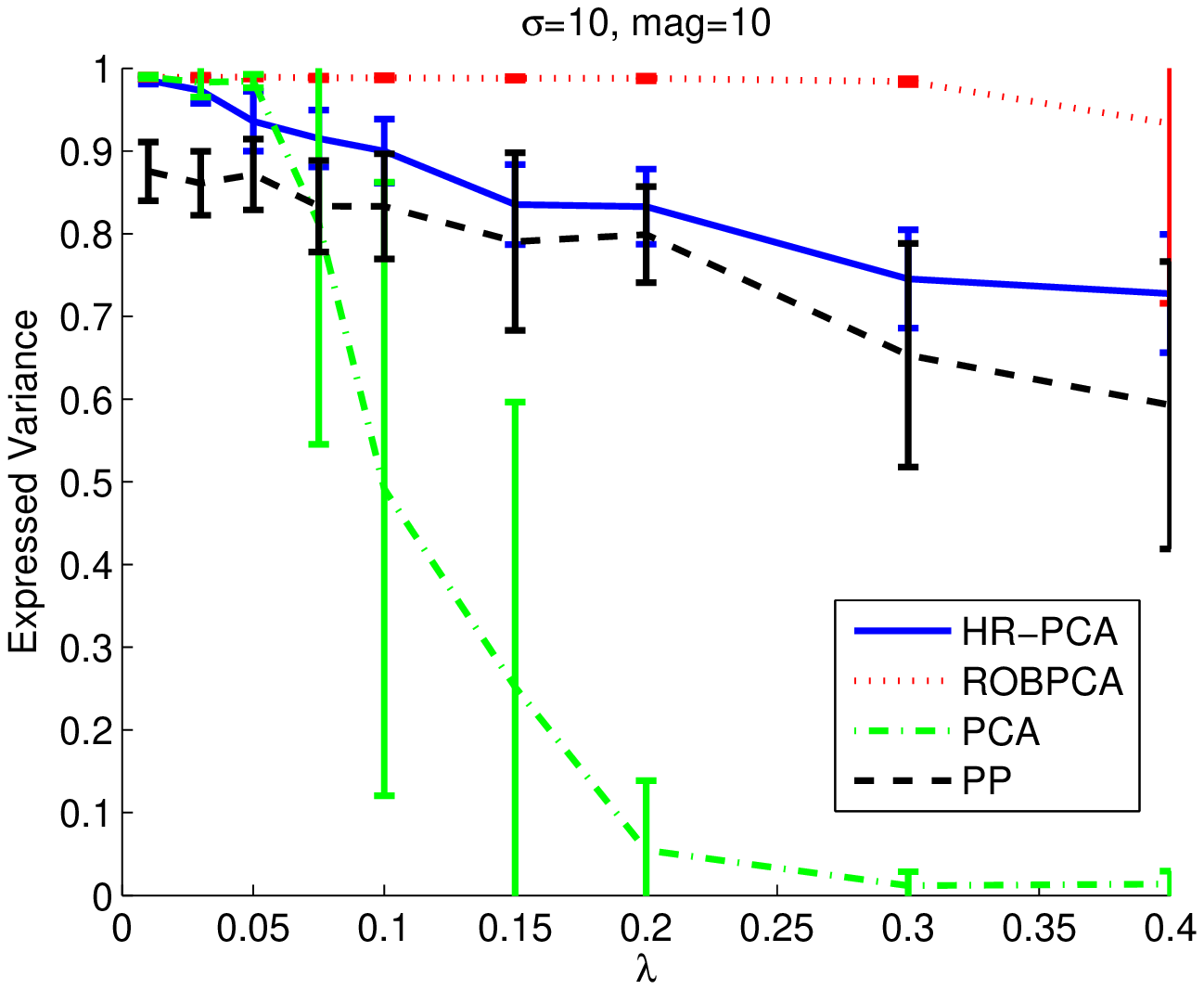}
  & \includegraphics[height=4.5cm,
  width=0.48\linewidth]{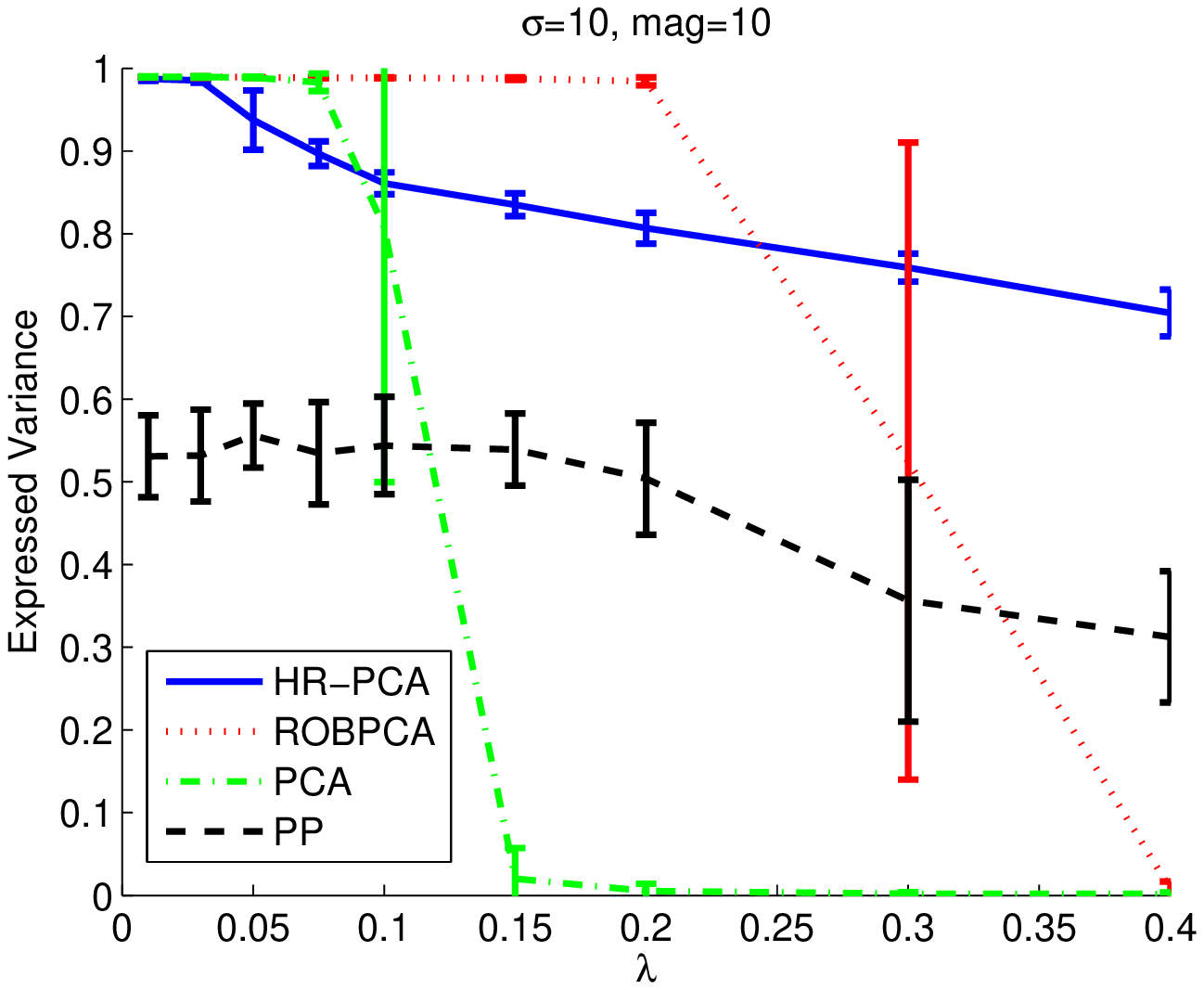}\\
  \includegraphics[height=4.5cm,
  width=0.48\linewidth]{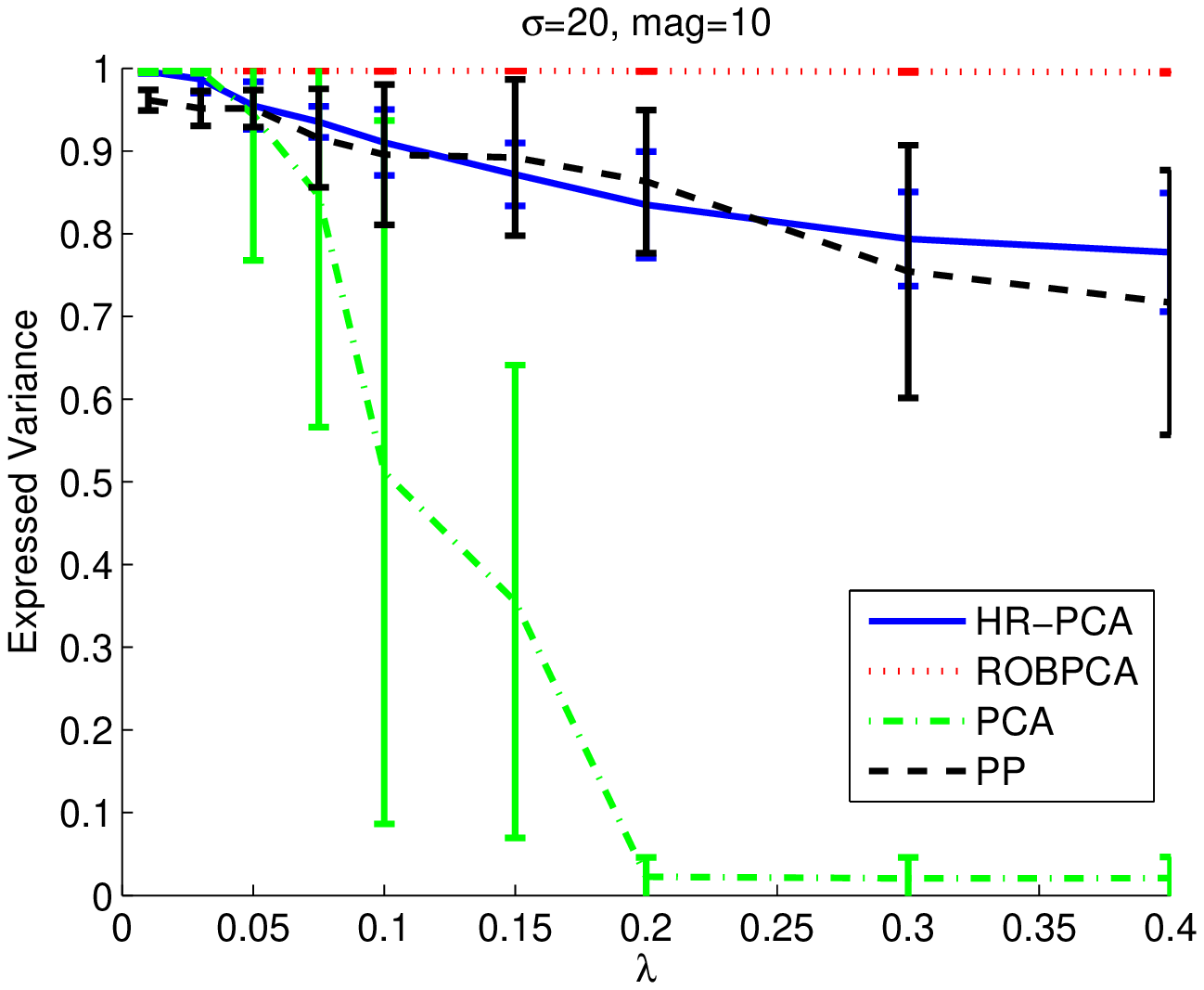}
  & \includegraphics[height=4.5cm,
  width=0.48\linewidth]{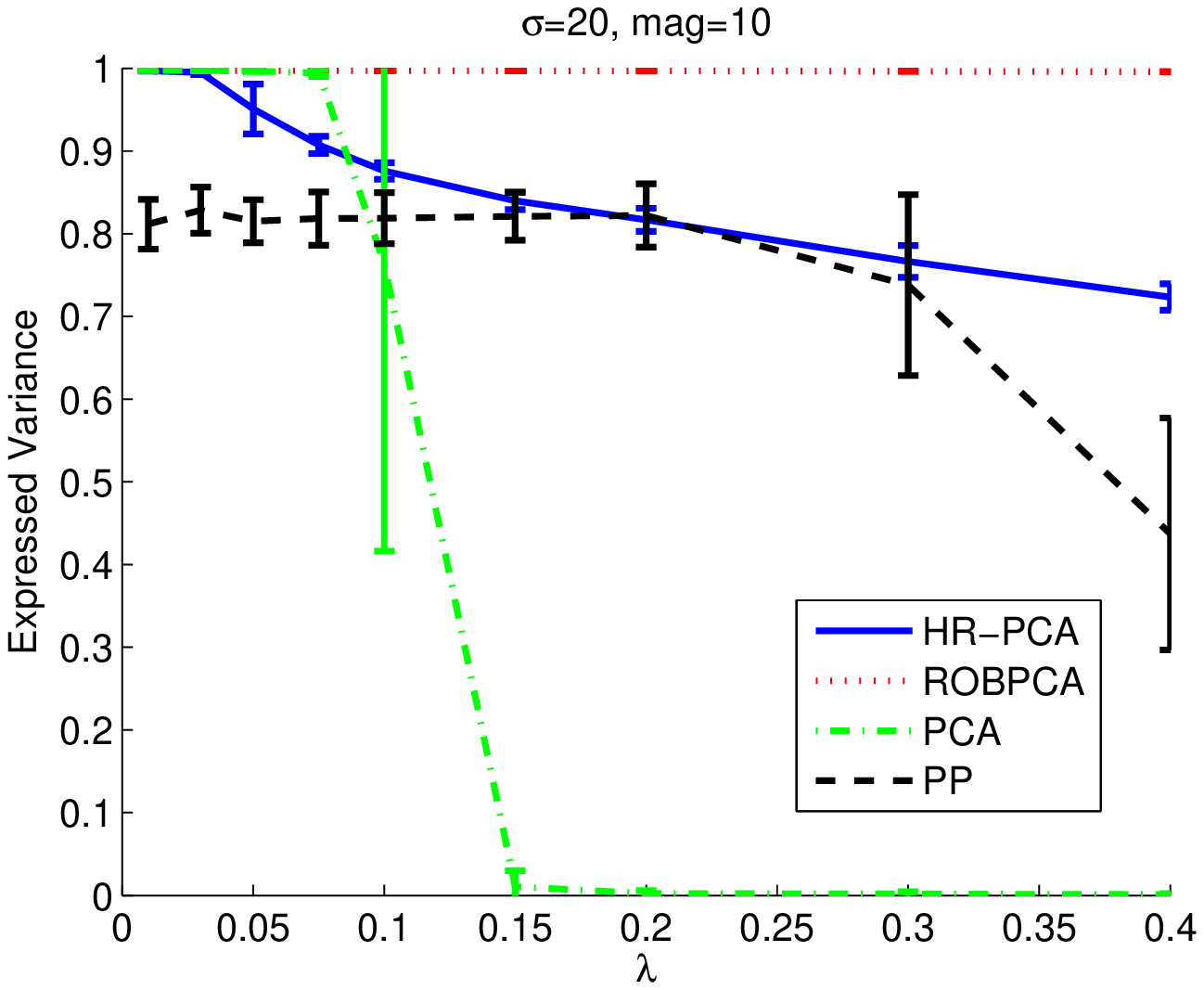}\\
  (a) $n=m=100$ & (b) $n=m=1000$  \end{tabular}
  \caption{Performance of HR-PCA vs ROBPCA, PP, PCA ($d=1$).}\label{fig.1d}
\end{center}
\end{figure}
We make the following three observations from Figure~\ref{fig.1d}.
 First, PP and ROBPCA can breakdown when $\lambda$ is
large, while on the other hand, the performance of HR-PCA is rather
robust even when $\lambda$ is as large as $40\%$. Second, the
performance of PP and ROBPCA  depends strongly on $\sigma$, i.e.,
the signal magnitude (and hence the magnitude of the corrupted
points). Indeed, when $\sigma$ is very large, ROBPCA achieves
effectively optimal recovery of the $A$ subspace. However, the
performance of both algorithms is not satisfactory when $\sigma$ is
small, and sometimes even worse than the performance of standard
PCA. Finally, and perhaps most importantly, the performance of
PP and ROBPCA degrades as the dimensionality increases, which makes
them essentially not suitable for the high-dimensional regime we
consider here. This is more explicitly shown in Figure~\ref{fig.dim}
where the performance of different algorithms versus dimensionality
is reported. We notice that the performance of ROBPCA (and similarly
other algorithms based on Stahel-Donoho outlyingness) has a sharp
decrease at a certain threshold that corresponds to the
dimensionality where S-D outlyingness becomes invalid in identifying
outliers.

\begin{figure}[htbp!]
\begin{center}
\begin{tabular}{cc}
  \includegraphics[height=4.5cm,
  width=0.48\linewidth]{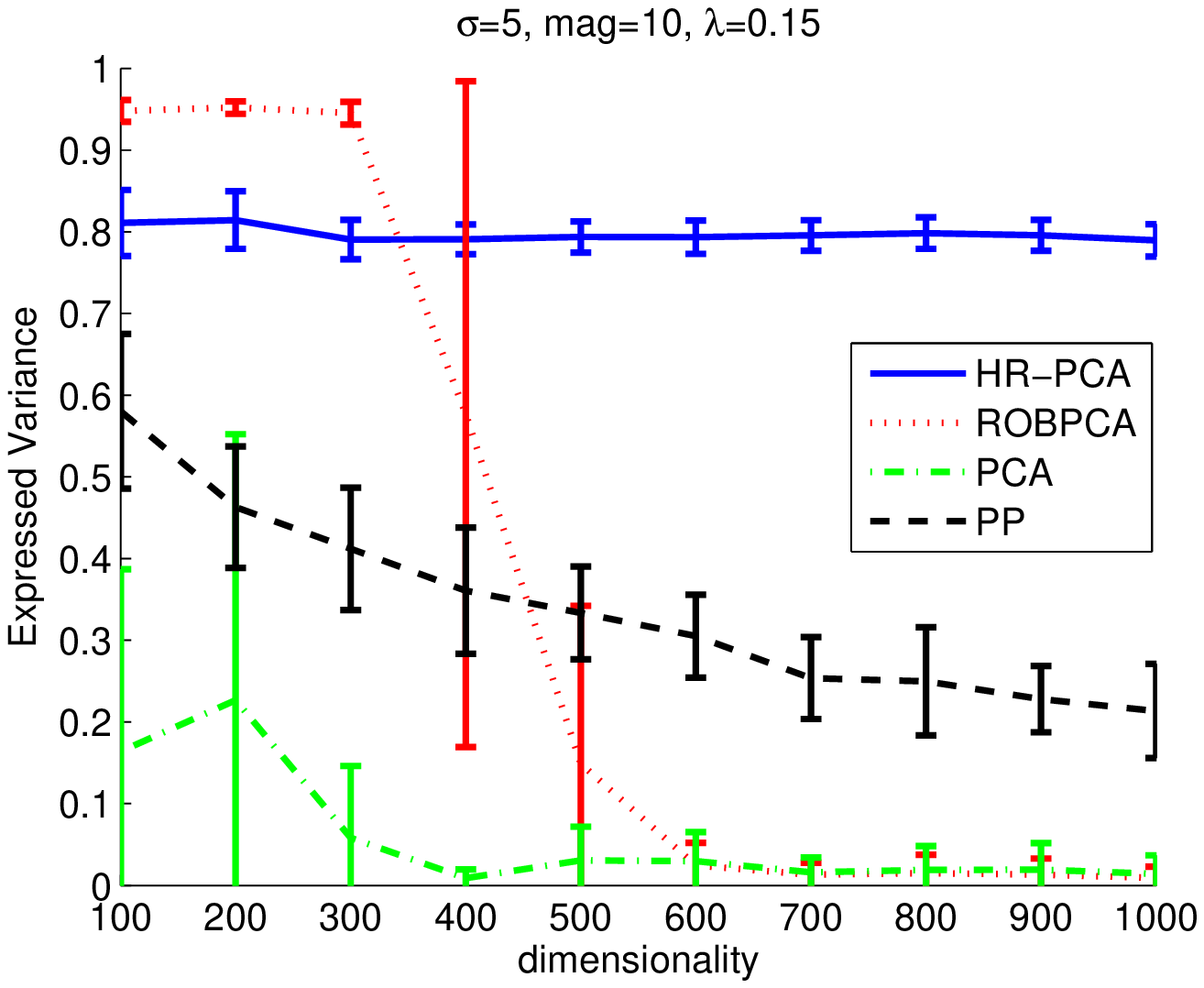}
  & \includegraphics[height=4.5cm,
  width=0.48\linewidth]{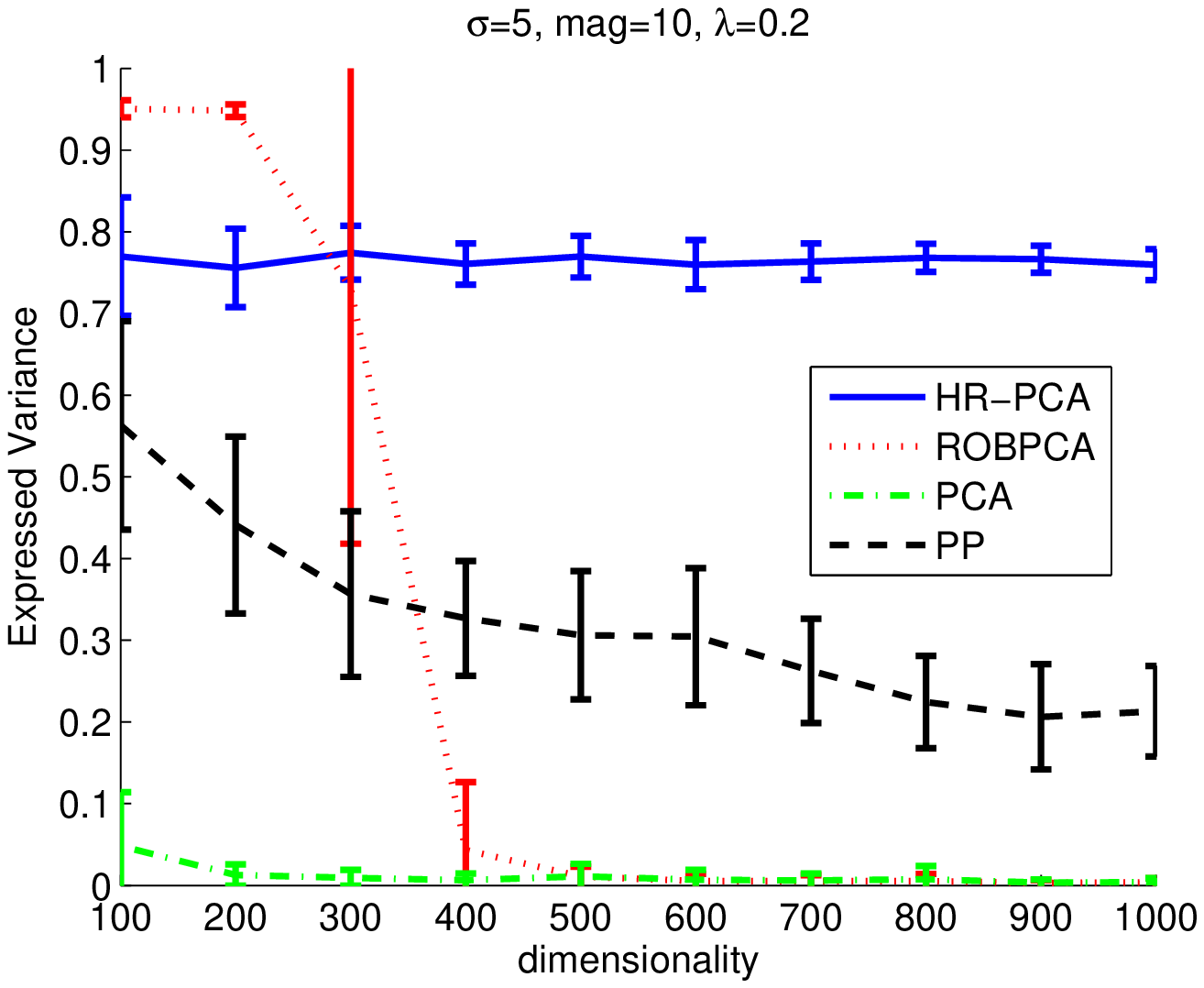}\\
  (a) $\lambda=0.15$ & (b) $\lambda=0.2$\\
  \includegraphics[height=4.5cm,
  width=0.48\linewidth]{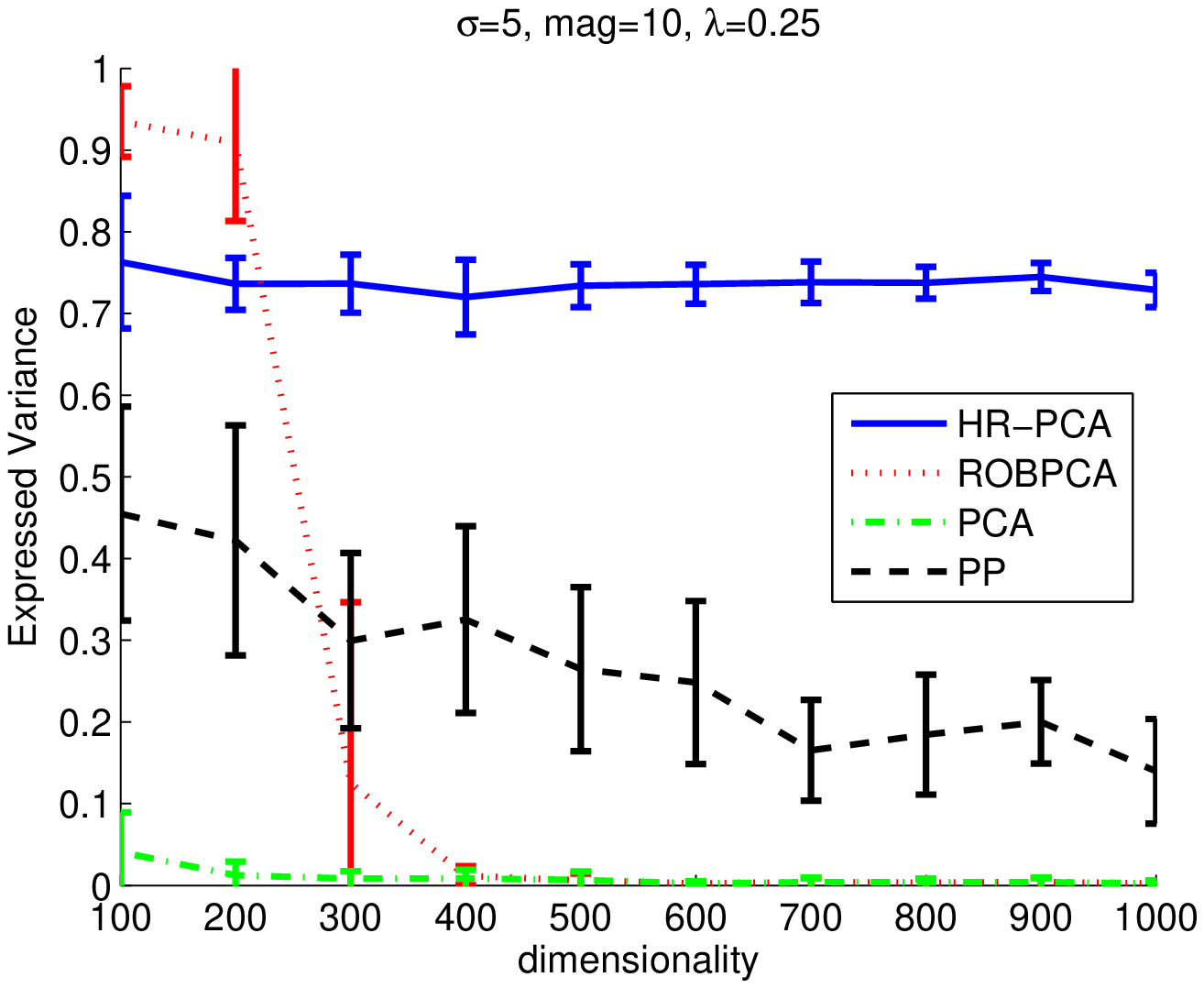}
  & \includegraphics[height=4.5cm,
  width=0.48\linewidth]{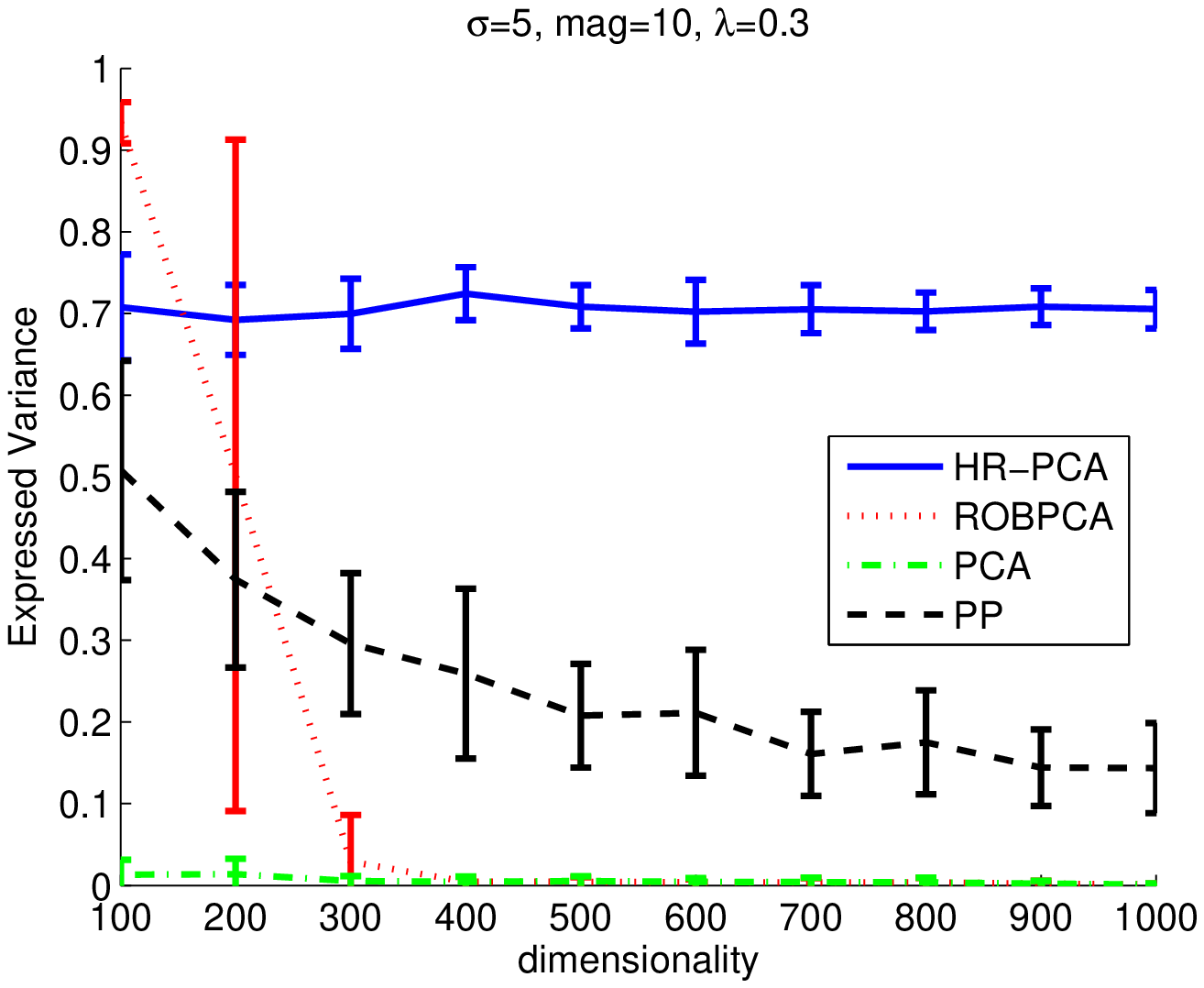}\\
  (c) $\lambda=0.25$ & (d) $\lambda=0.3$ \end{tabular}
  \caption{Performance vs dimensionality.}\label{fig.dim}
\end{center}
\end{figure}

\begin{figure}[htbp!]
\begin{center}
\begin{tabular}{ccc}
  \includegraphics[height=4.5cm,
  width=0.33\linewidth]{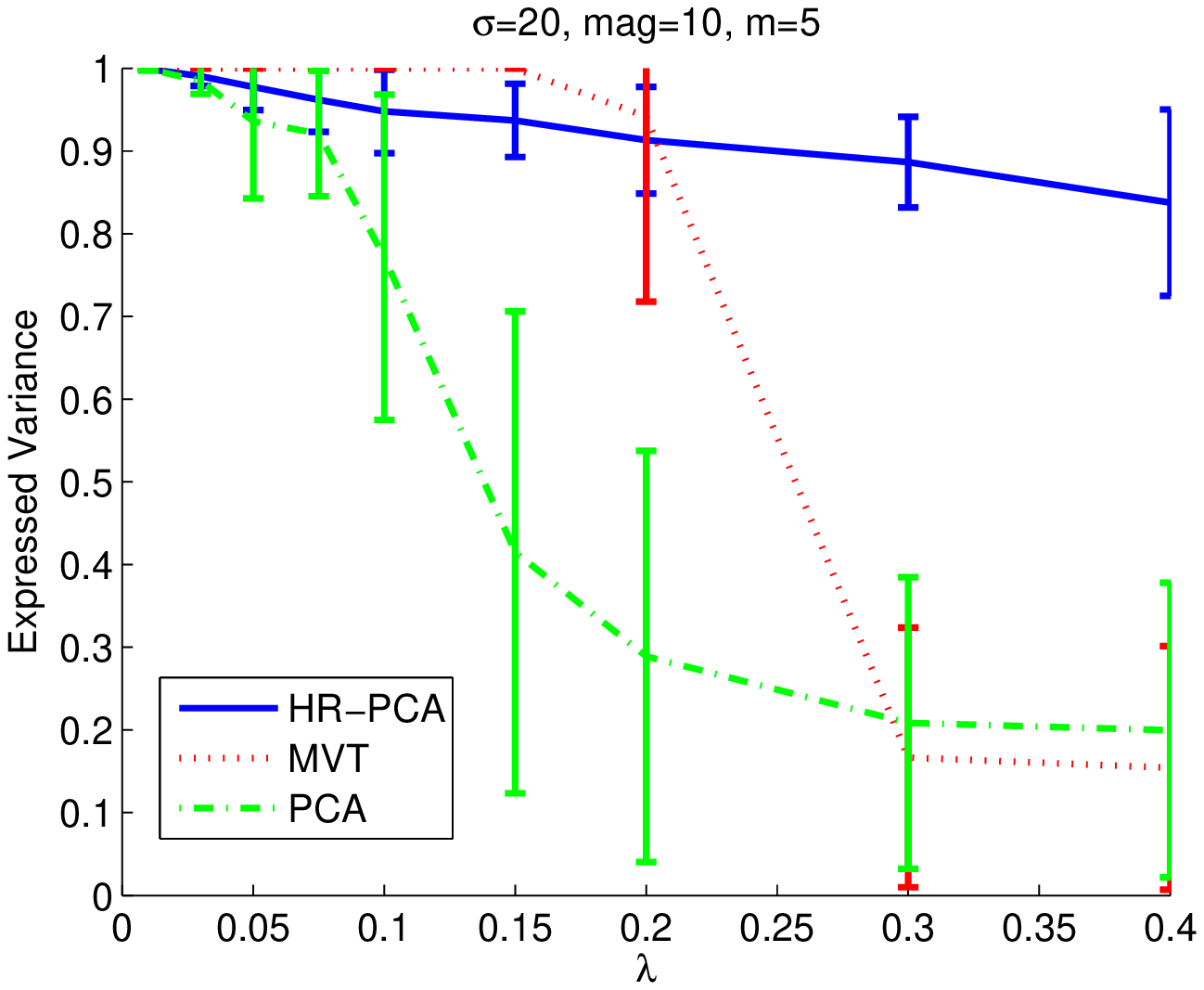}
  & \includegraphics[height=4.5cm,
  width=0.33\linewidth]{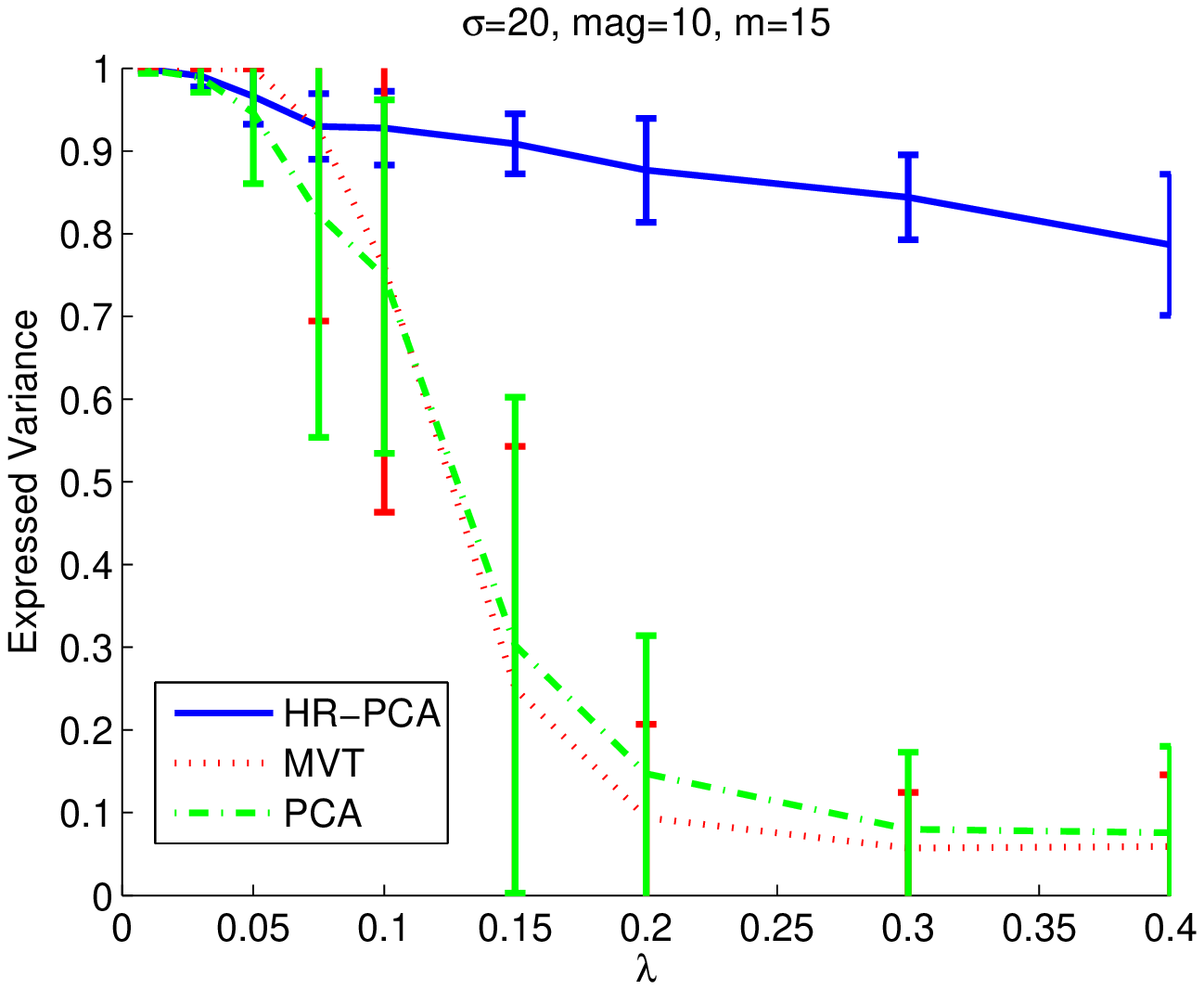}
&  \includegraphics[height=4.5cm,
  width=0.33\linewidth]{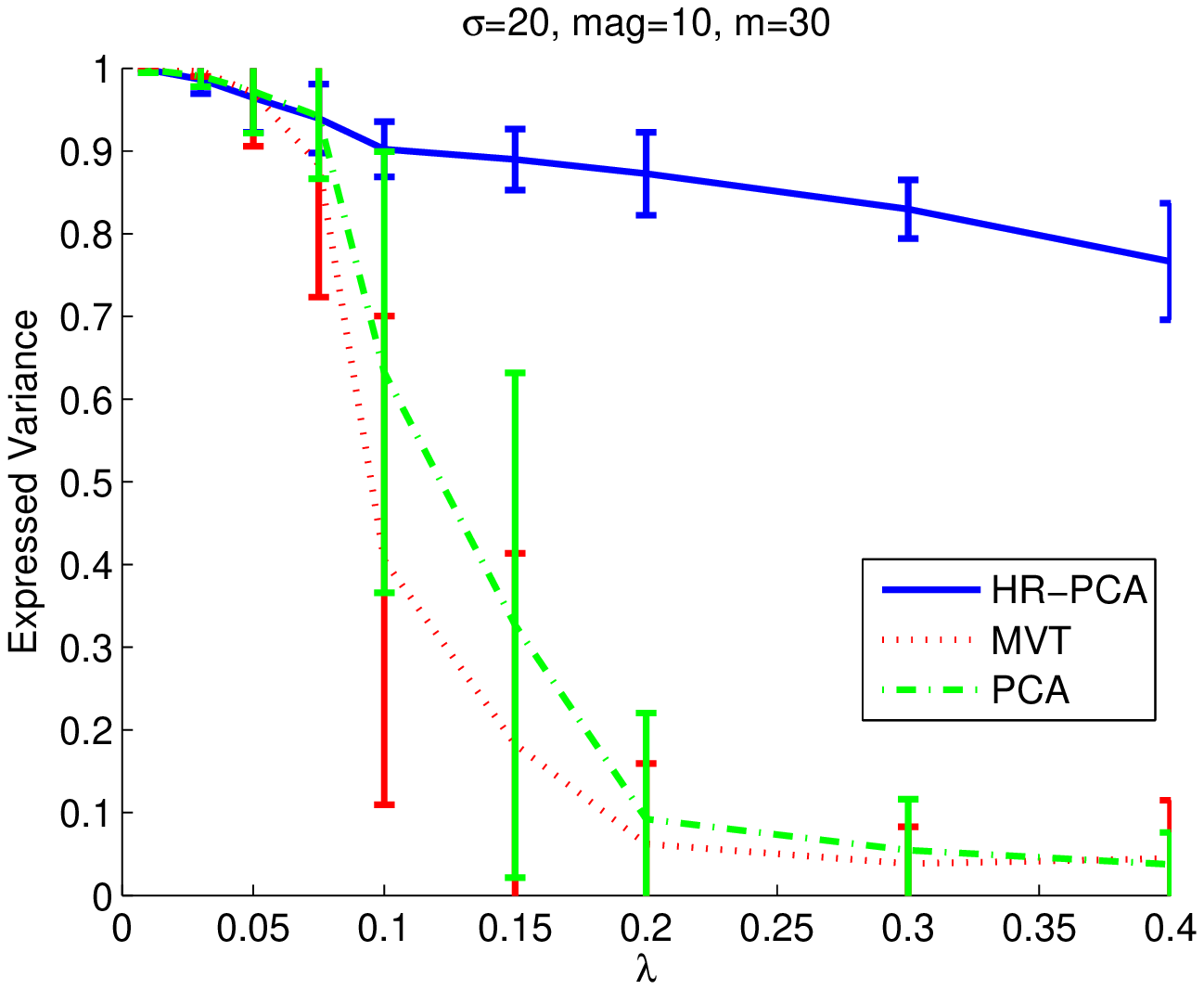}\\
  (a) $n=100, m=5$ & (b) $n=100, m=15$ & (c) $n=100, m=30$ \end{tabular}
  \caption{Performance of HR-PCA vs MVT for $m\ll n$.}\label{fig.mvt}
\end{center}
\end{figure}
Figure~\ref{fig.mvt} shows that the performance of MVT depends on
the dimensionality $m$. Indeed, the breakdown property of MVT is
roughly $1/m$ as  predicted by the theoretical analysis, which makes
MVT less attractive in the high-dimensional regime.

A similar numerical study for $d=3$ is also performed, where the
outliers are generated on $3$ random chosen lines. The results are
reported in Figure~\ref{fig.3d}. The same trends as in the $d=1$
case are observed, although the performance gap between different
strategies are smaller, because the effect of outliers are decreased
since they are on $3$ directions.
\begin{figure}[htbp!]
\begin{center}
\begin{tabular}{cc}
  \includegraphics[height=4.5cm,
  width=0.48\linewidth]{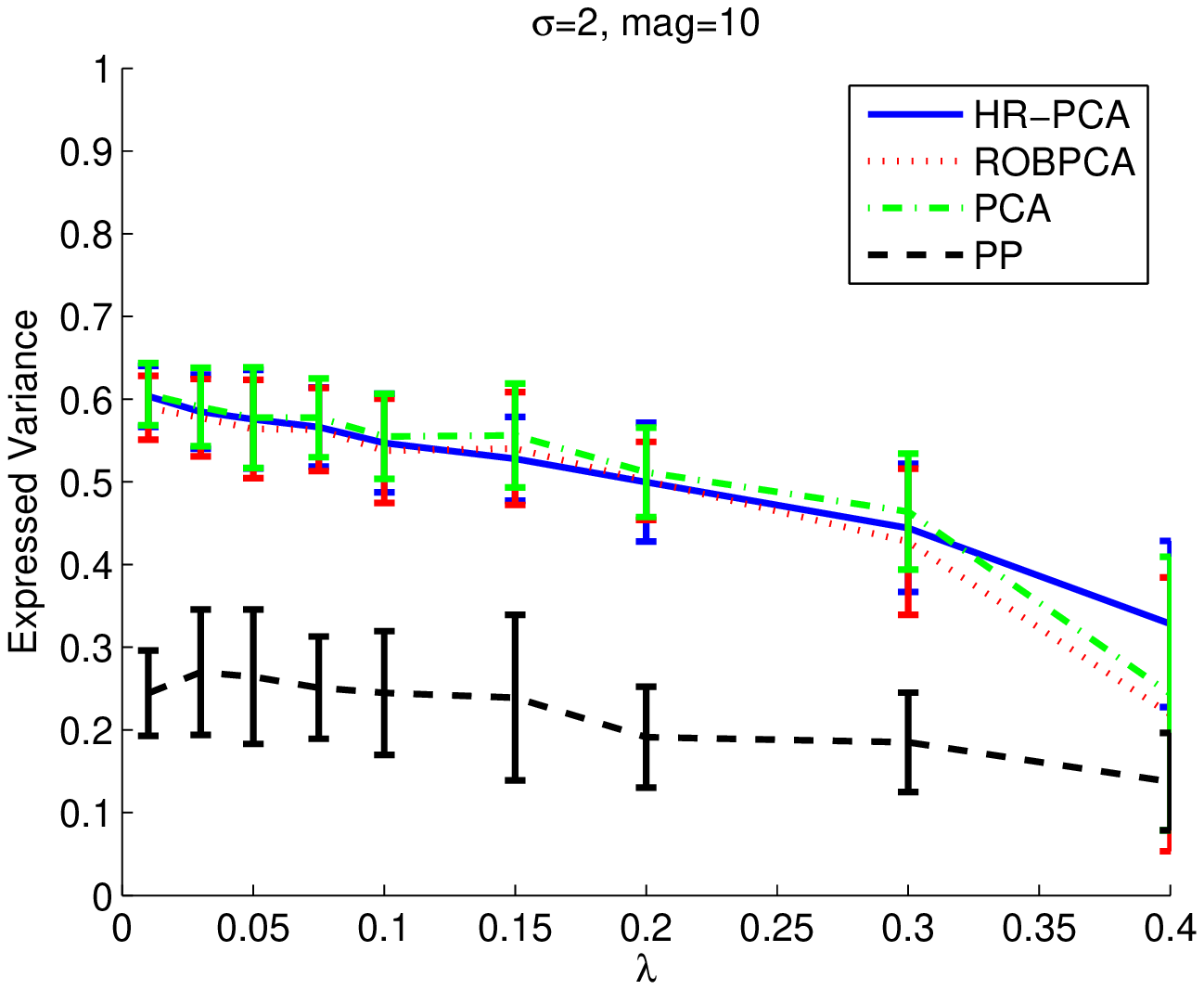}
  & \includegraphics[height=4.5cm,
  width=0.48\linewidth]{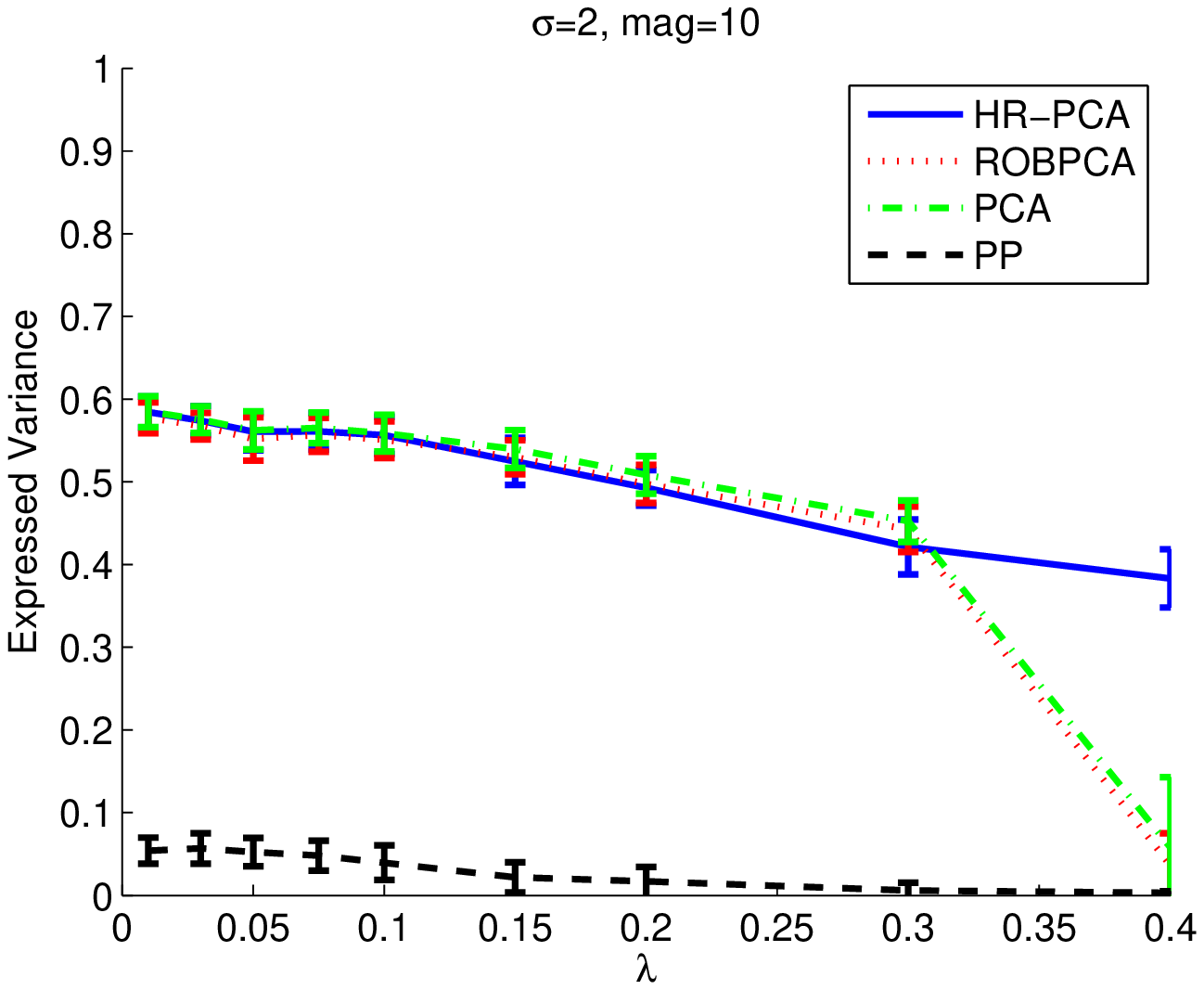}\\
  \includegraphics[height=4.5cm,
  width=0.48\linewidth]{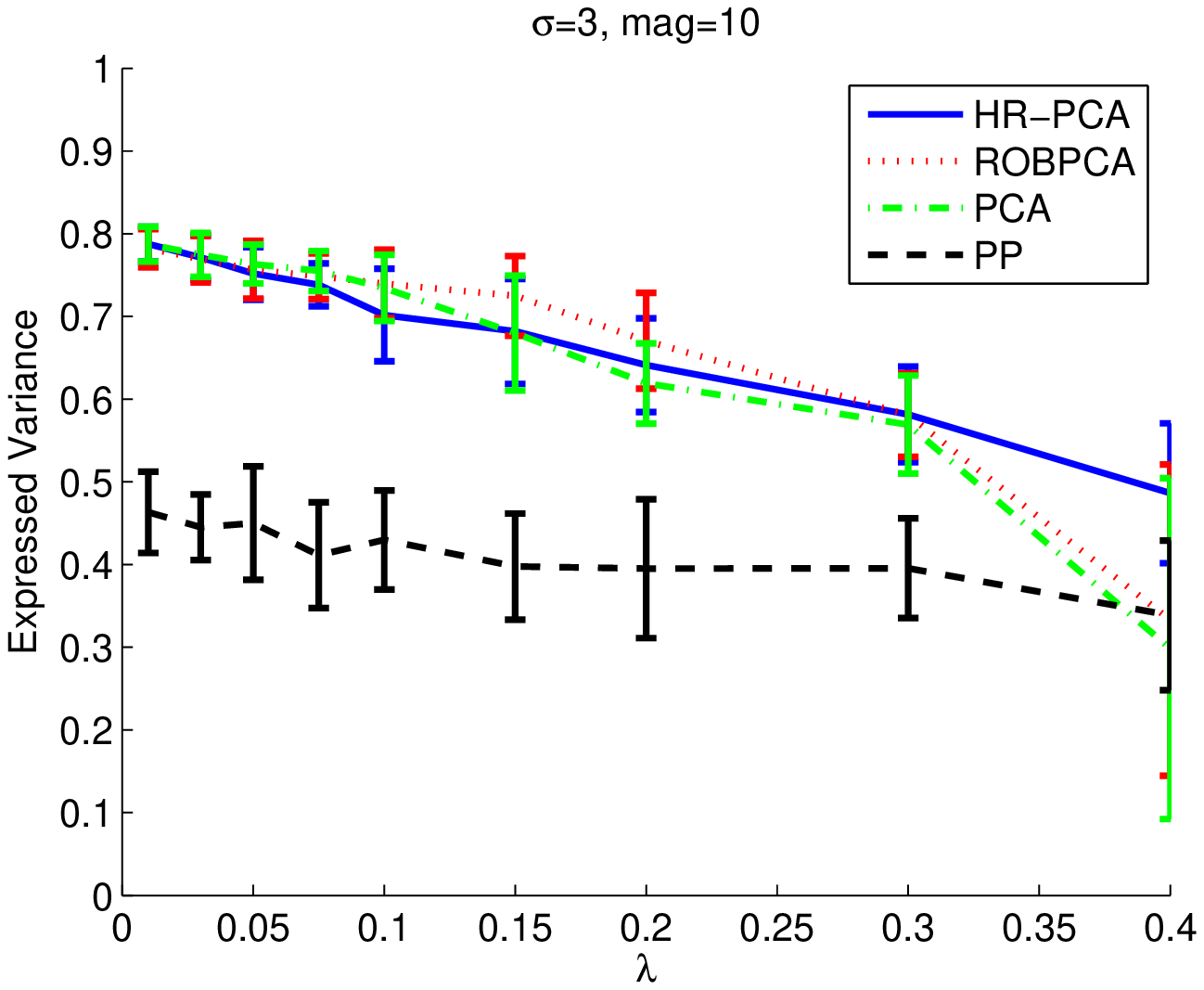} &
  \includegraphics[height=4.5cm,
  width=0.48\linewidth]{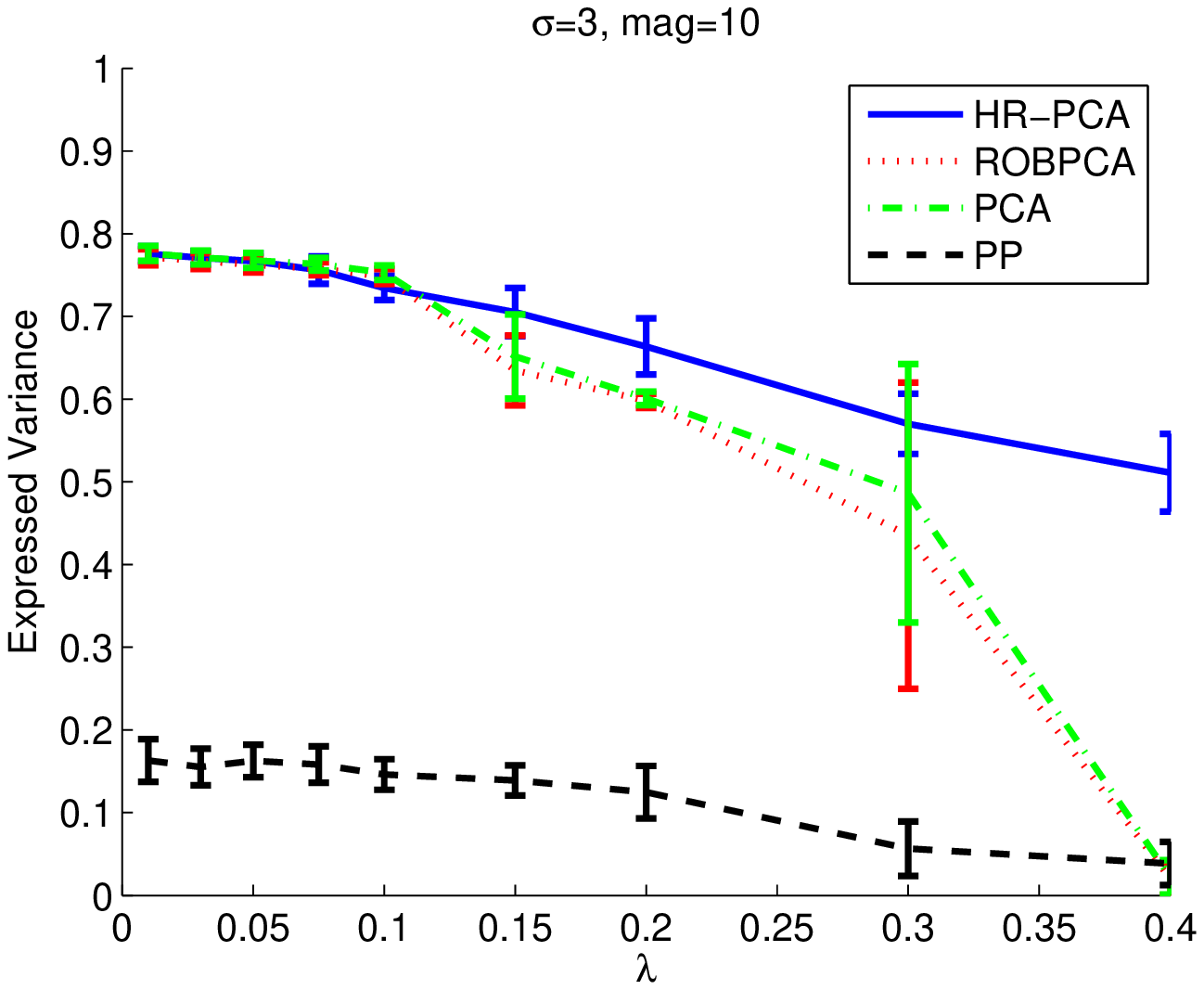} \\
    \includegraphics[height=4.5cm,
  width=0.48\linewidth]{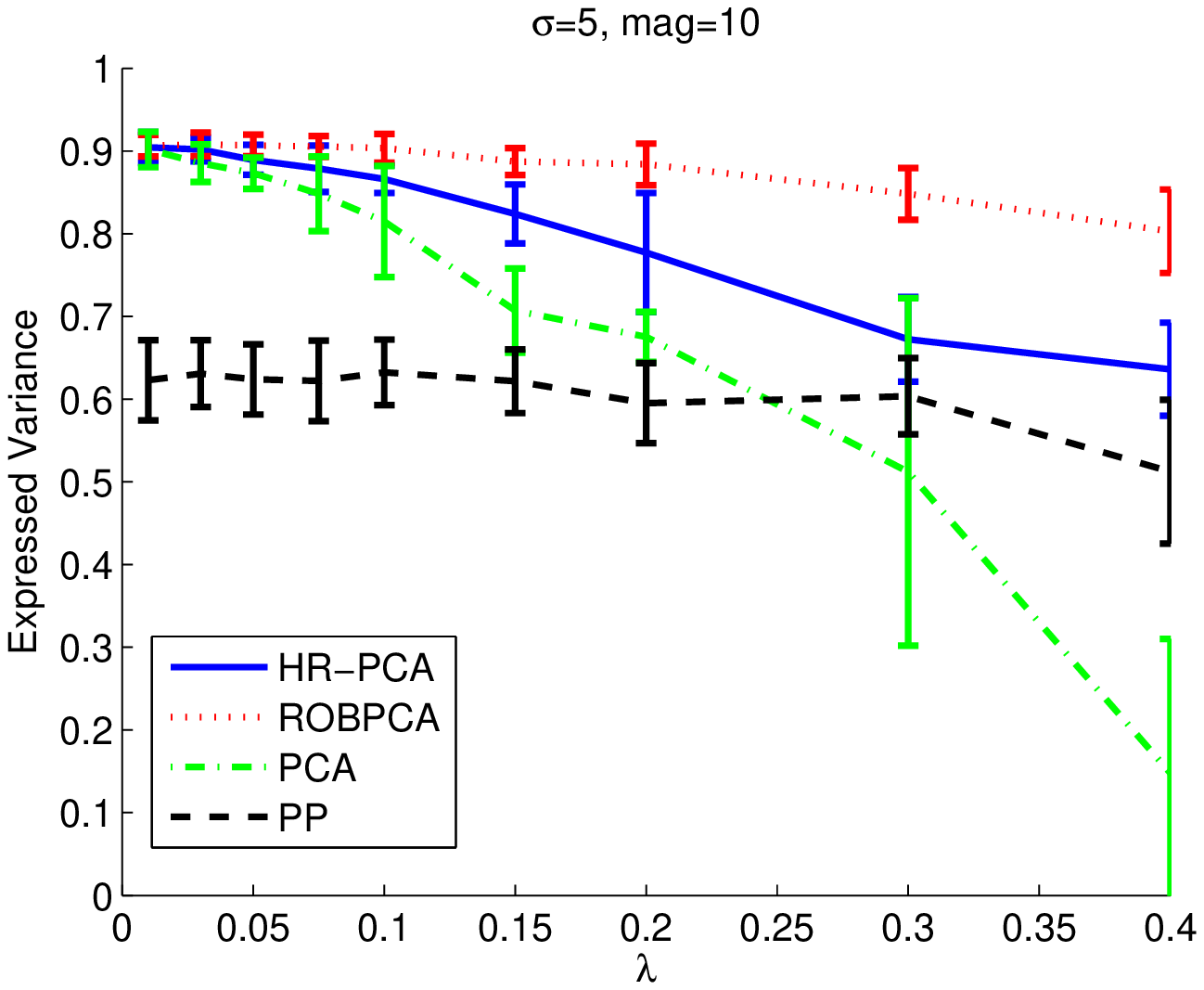}
  &      \includegraphics[height=4.5cm,
  width=0.48\linewidth]{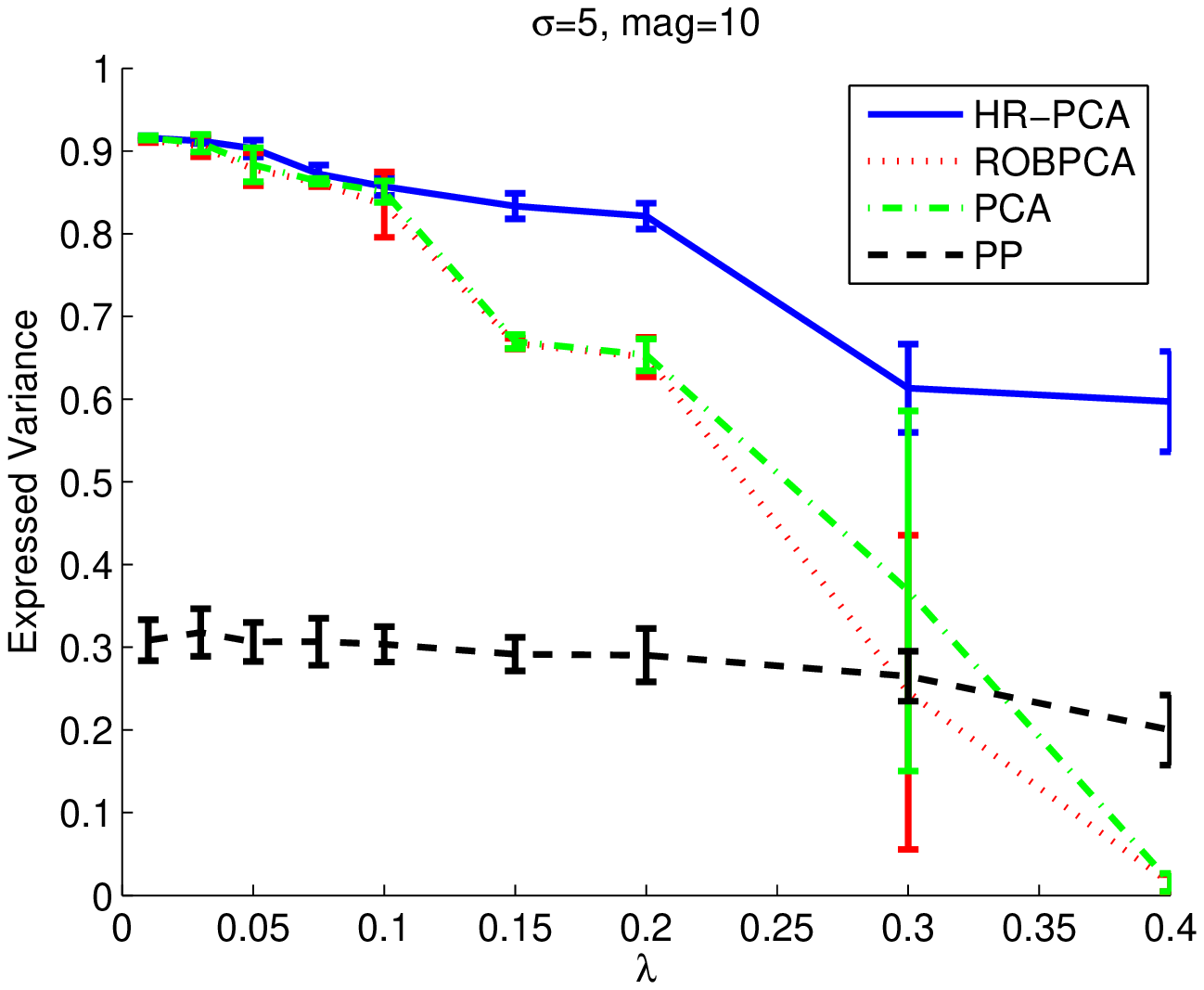}\\
  \includegraphics[height=4.5cm,
  width=0.48\linewidth]{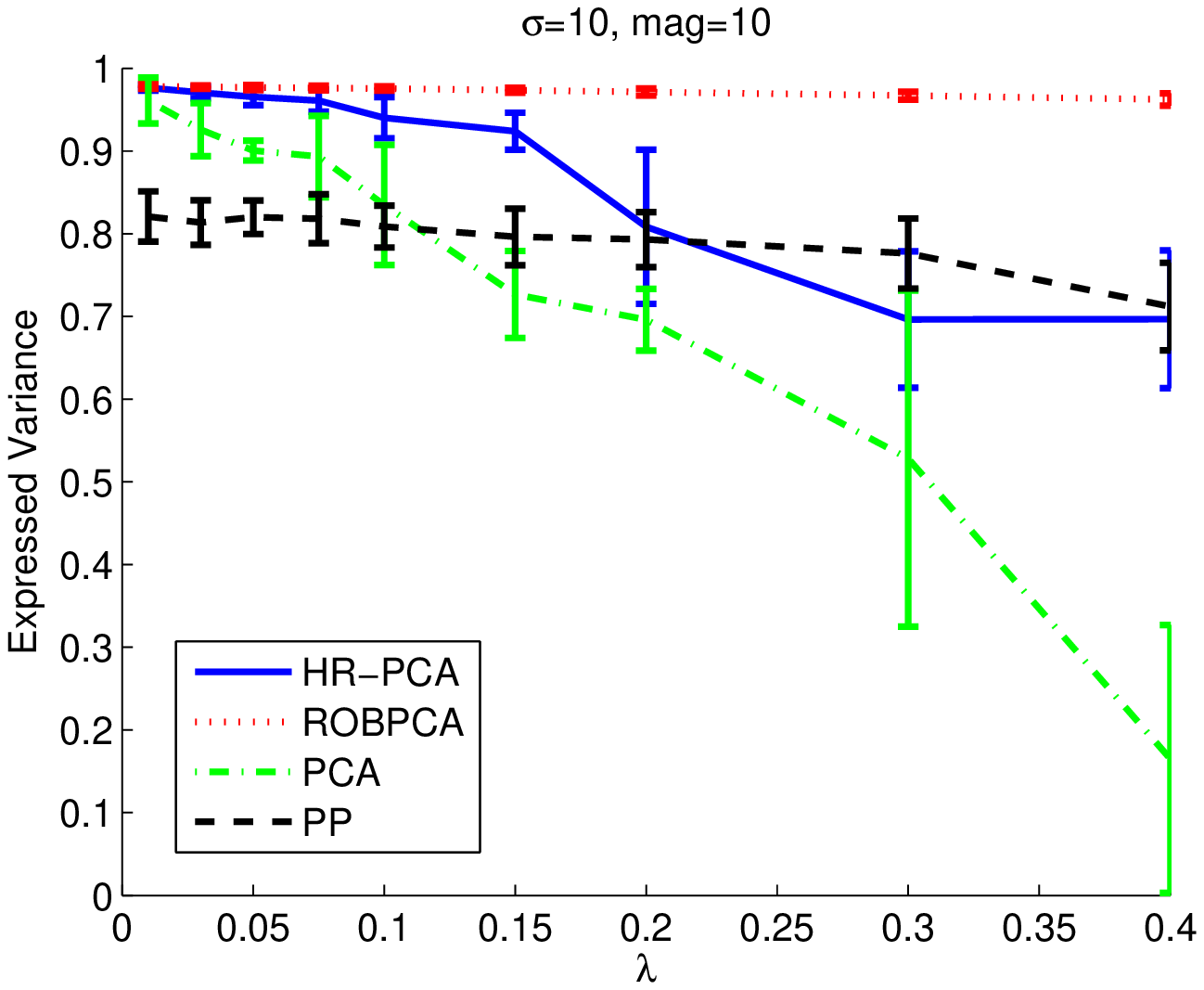}
  & \includegraphics[height=4.5cm,
  width=0.48\linewidth]{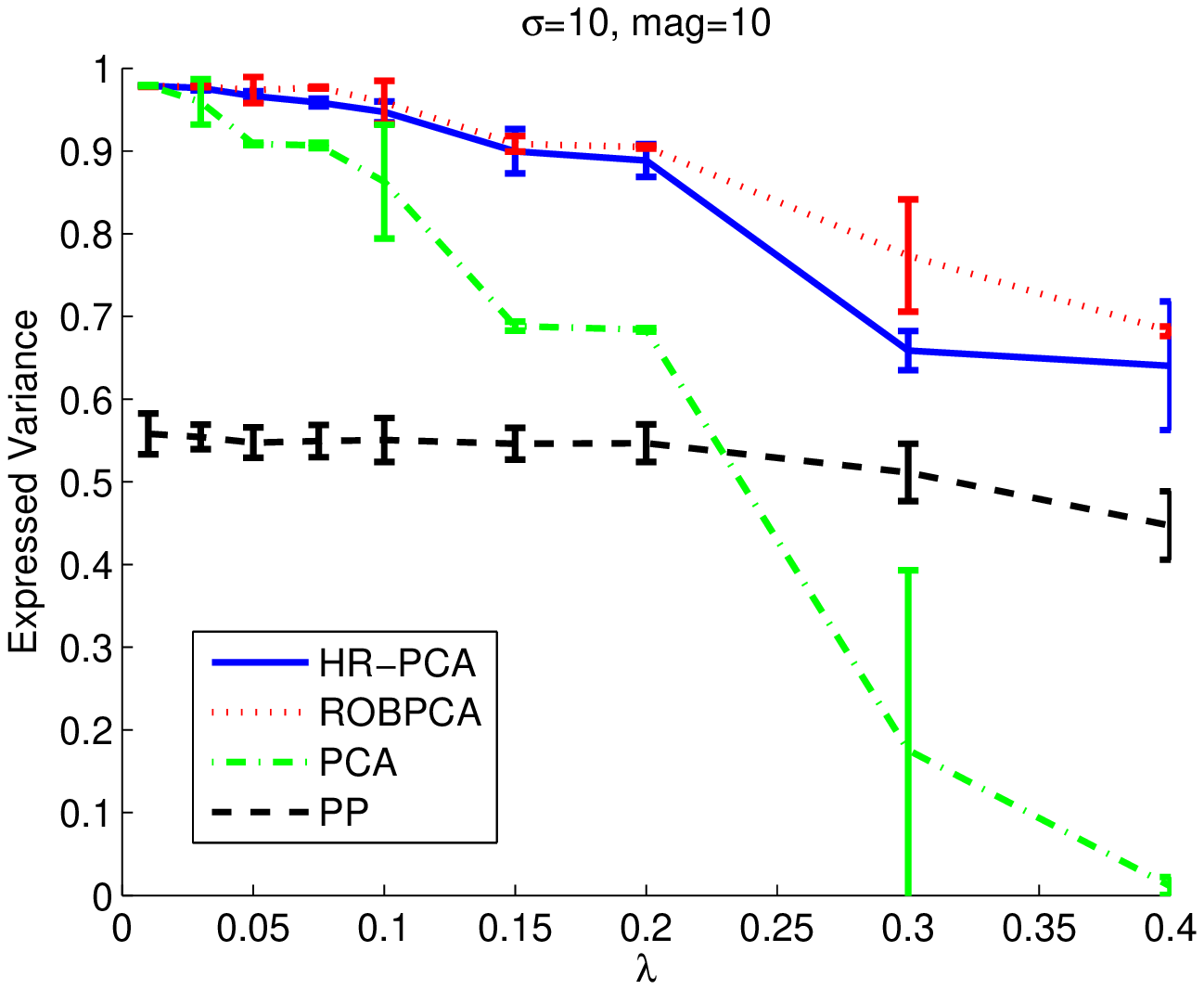}\\
  \includegraphics[height=4.5cm,
  width=0.48\linewidth]{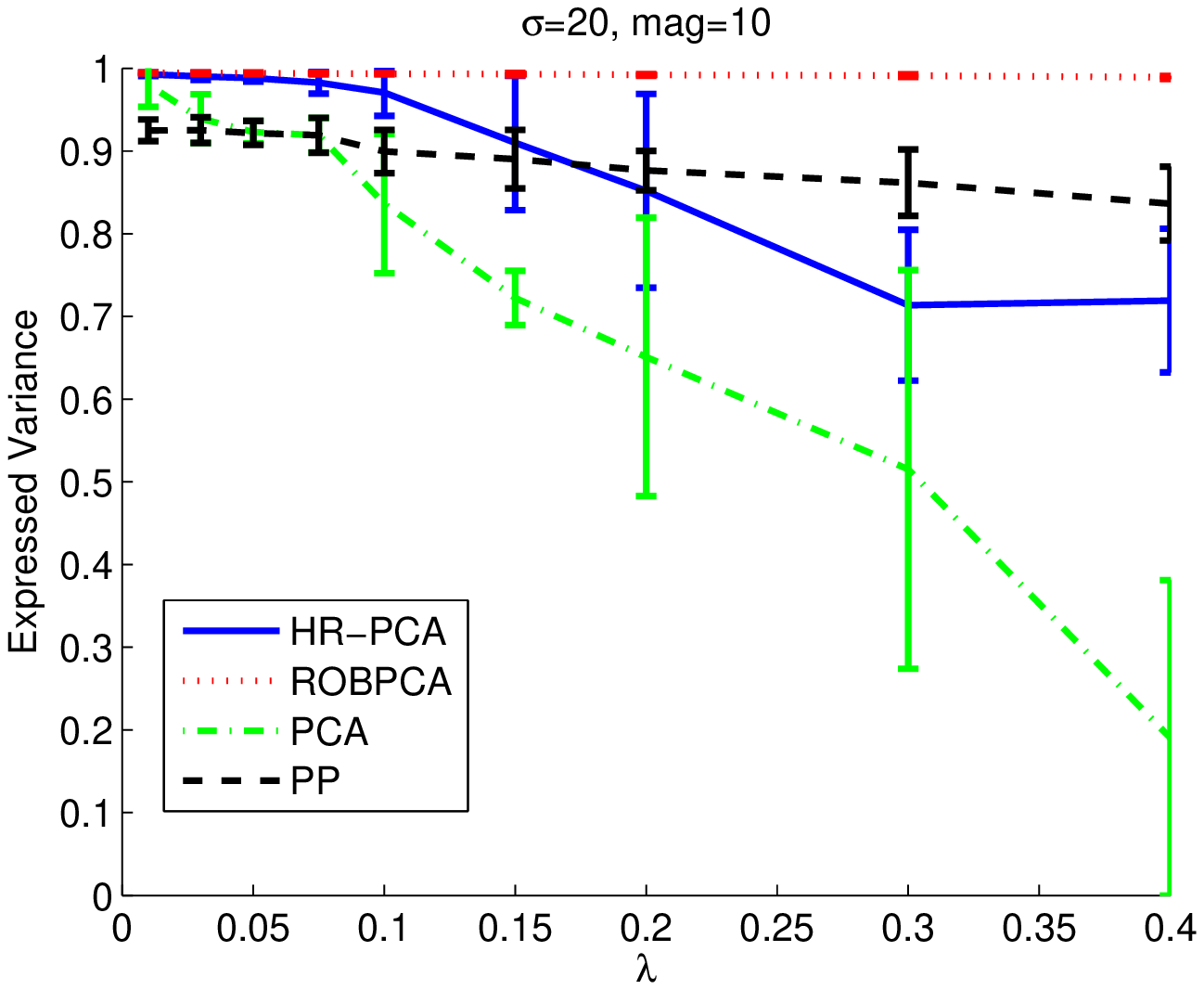}
  & \includegraphics[height=4.5cm,
  width=0.48\linewidth]{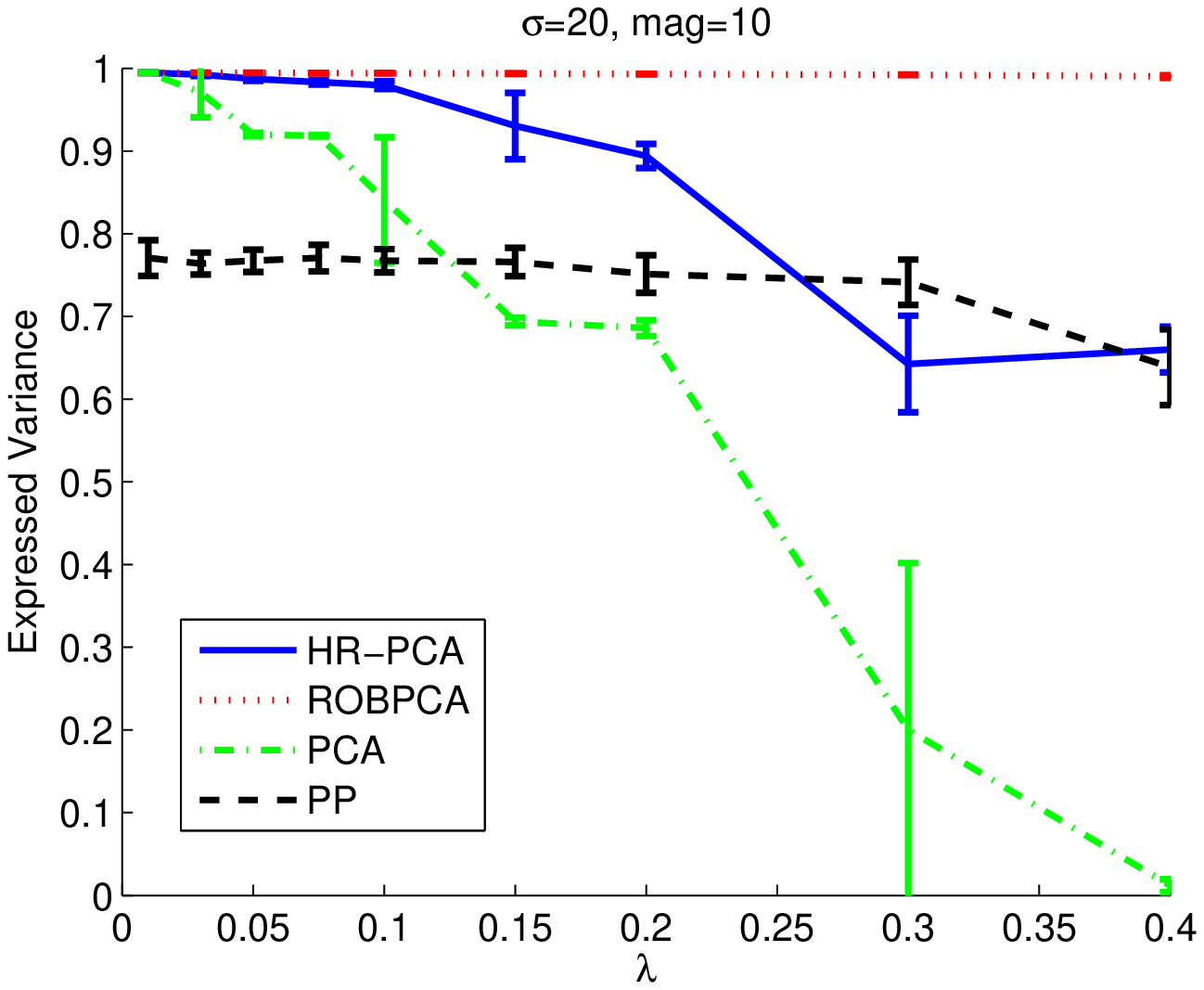}\\
  (a) $n=m=100$ & (b) $n=m=1000$  \end{tabular}
  \caption{Performance of HR-PCA vs ROBPCA, PP, PCA ($d=3$).}\label{fig.3d}
\end{center}
\end{figure}

\section{Concluding Remarks}\label{sec.conclusion}
In this paper, we investigated the dimensionality-reduction problem
in the case where the number and the dimensionality of samples are
of the same magnitude, and a constant fraction of the points are
arbitrarily corrupted (perhaps maliciously so). We proposed a
High-dimensional Robust Principal Component Analysis algorithm that
is tractable, robust to corrupted points, easily kernelizable and
asymptotically optimal. The algorithm  iteratively finds a set of
PCs using standard PCA and subsequently remove a point randomly with
a probability proportional to its expressed variance. We provided
both theoretical guarantees and favorable simulation results about
the performance of the proposed algorithm.

To the best of our knowledge, previous efforts to extend existing
robust PCA algorithms into the high-dimensional case remain
unsuccessful. Such algorithms are designed for low dimensional data
sets where the observations significantly outnumber the variables of
each dimension. When applied to high-dimensional data sets, they
either lose statistical consistency due to lack of sufficient
observations, or become highly intractable. This motivates our work
of proposing a new robust PCA algorithm that takes into account the
inherent difficulty in analyzing high-dimensional data.

{\small
\bibliographystyle{unsrt}
\bibliography{Phd1}}

\appendix
\section{The Details from Section \ref{sec.proof}}
In this appendix, we provide some of the details omitted in Section \ref{sec.proof}.

\subsection{Proof of Lemma~\ref{lem.onedirectionpartialvarbound}}
{\bf Lemma~\ref{lem.onedirectionpartialvarbound}} Given $\delta\in
[0,1]$, $\hat{c} \in \mathbb{R}^+$, $\hat{m}, m\in \mathbb{N}$
satisfying $\hat{m}<m$. Let $a_1,\cdots, a_m$ be i.i.d. samples
drawn from a probability measure $\mu$ supported on $\mathbb{R}^+$
and has a density function. Assume that $\mathbb{E}(a)=1$ and
$\frac{1}{m}\sum_{i=1}^ma_i\leq 1+\hat{c}$. Then with probability at
least $1-\delta$ we have
$$
\sup_{\overline{m}\leq \hat{m}}|\frac{1}{m}\sum_{i=1}^{\overline{m}}
a_{(i)}- \int_{0}^{\mu^{-1}({\overline{m}}/{m})}a d\mu|\leq
\frac{(2+\hat{c})m}{m-\hat{m}}\sqrt{\frac{8(2\log
m+1+\log\frac{8}{\delta})}{m}},
$$
where $\mu^{-1}(x)\triangleq \min\{z| \mu(a\leq z)\geq x\}$.

\begin{proof}In Section~\ref{sec.proof}, using VC dimension
argument, we showed that
\begin{equation}\label{equ.uniformboundong}
\mathrm{Pr}\big(\sup_{e \geq 0}|\frac{1}{m}\sum_{i=1}^m
g_e(a_i)-\mathbb{E}g_e(a)|\geq \epsilon_0\big)\leq 4\exp(2 \log
m+1-m\epsilon_0^2/8)=\frac{\delta}{2};
\end{equation}
and
\begin{equation}\label{equ.uniformboundonf}\mathrm{Pr}\big(\sup_{e\in [0, (1+c)m/(m-\hat{m})]
}|\frac{1}{m}\sum_{i=1}^m f_e(a_i)-\mathbb{E}f_e(a)|\geq
\epsilon\big)\leq 4\exp\left(2 \log m+1-
\frac{\epsilon^2(m-\hat{m})^2}{8(1+\hat{c})^2m}
\right)=\frac{\delta}{2}.
\end{equation}
To complete the proof, define $h(\cdot): [0,1]\rightarrow
\mathbb{R}^+$ as $h(x)=\int_{0}^{\mu^{-1}(x)} ad\mu$. Since $\mu$ is
supported on $\mathbb{R}^+$, by Markov inequality we have for any
$n<m$, $$a_{(n)}\leq a_{(n+1)} \leq \frac{(1+\hat{c})m}{m-n},$$ due
to $\mathbb{E}(a)=1$ and $\frac{1}{m}\sum_{i=1}^ma_i\leq 1+\hat{c}$.
Similarly, by Markov inequality we have for any $d,\epsilon \in
[0,1]$ such that $d+\epsilon <1$, the following holds:
\begin{equation}\label{equ.proofoflemma2}
\frac{h(d+\epsilon)-h(d)}{(d+\epsilon)-d}\leq
\frac{h(1)-h(d)}{1-d},$$ which implies $$ h(d+\epsilon)-h(d)\leq
\frac{\epsilon}{1-d}.
\end{equation}

Let $e_n=a_{(n)}$ for $n \leq m$, we have
\begin{equation*}
\begin{split}
&\sup_{\overline{m}\leq \hat{m}}\left|\frac{1}{m}\sum_{i=1}^{\overline{m}} a_{(i)}-h(\overline{m}/m)\right|\\
\leq & \sup_{\overline{m}\leq
\hat{m}}\left\{\big|\frac{1}{m}\sum_{i=1}^{\overline{m}} a_{(i)}-
\mathbb{E} f_{a_{(\overline{m})}}(a)\big|+
\big|\mathbb{E}f_{a_{(\overline{m})}}(a)-h(\overline{m}/m)\big|\right\}\\
= & \sup_{\overline{m}\leq
\hat{m}}\left\{\big|\frac{1}{m}\sum_{i=1}^m
f_{e_{\overline{m}}}\big(a_{(i)}\big)- \mathbb{E}
f_{e_{\overline{m}}}(a)\big|+
\big|\mathbb{E}f_{e_{\overline{m}}}(a)-h(\overline{m}/m)\big|\right\}\\\leq
& \sup_{\overline{m}\leq \hat{m}}|\frac{1}{m}\sum_{i=1}^{m}
f_{e_{\overline{m}}}(a_{(i)})- \mathbb{E}
f_{e_{\overline{m}}}(a)|+\sup_{\overline{m}'\leq \hat{m}}
|\mathbb{E}f_{e_{\overline{m}'}}(a)-h(\overline{m}'/m)|.
\end{split}\end{equation*} With probability at least $1-\delta/2$, the first
term is upper bounded by $\epsilon$ due to Inequality
(\ref{equ.uniformboundonf}). To bound the second term, we note that
from Inequality (\ref{equ.uniformboundong}), with probability at
least $1-\delta/2$ the following holds
\[\sup_{\overline{m}'\leq \hat{m}}\big|\overline{m}'/m- \mu([0,e_{\overline{m}'}])\big|=\sup_{\overline{m}'\leq \hat{m}}\big|\frac{1}{m}\sum_{i=1}^m g_{e_{\overline{m}'}}(a_i)-
\mathbb{E}g_{e_{\overline{m}'}}(a)\big|\leq \epsilon_0.
\]This is equivalent to with probability $1-\gamma/2$,
 $\mu^{-1}(\overline{m}'/m -\epsilon_0)\leq e_{\overline{m}'}\leq
\mu^{-1}(\overline{m}'/m+\epsilon_0)$ holds uniformly for all
$\overline{m}'\leq \hat{m}$. Note that this further implies
\begin{equation*}\begin{split} &\sup_{\overline{m}'\leq \hat{m}}|\mathbb{E}f_{e_{\overline{m}'}}(a)-h(\overline{m}'/m)|\\
\leq& \sup_{\overline{m}'\leq
\hat{m}}\max\big[h(\overline{m}'/m+\epsilon_0)-h(\overline{m}'/m),
h(\overline{m}'/m))-h(\overline{m}'/m-\epsilon_0)\big]\\ \leq&
\sup_{\overline{m}'\leq \hat{m}}
\frac{m\epsilon_0}{m-\overline{m}'}=\frac{m\epsilon_0}{m-\hat{m}},\end{split}\end{equation*}where
the last inequality follows from~(\ref{equ.proofoflemma2}). This
 bounds the second term. Summing up the two terms
proves the lemma.
\end{proof}

\subsection{Proof of Theorem \ref{thm.incomvar-d}}

{\bf Theorem \ref{thm.incomvar-d}}: If $\sup_{\mathbf{w}\in
\mathcal{S}_d} \big|\frac{1}{t}\sum_{i=1}^t
(\mathbf{w}^\top\mathbf{x}_i)^2-1\big|\leq \hat{c}$, then
\begin{equation*}\begin{split} &\mathrm{Pr}\left\{\sup_{\mathbf{w}\in \mathcal{S}_d,
\overline{t}\leq \hat{t}} \big|\frac{1}t
\sum_{i=1}^{\overline{t}}|\mathbf{w}^\top
\mathbf{x}|_{(i)}^2-\mathcal{V}\left(\frac{\overline{t}}{t}\right)\big|\geq
\epsilon\right\}\\ &\leq \max\left[\frac{8e
t^26^d(1+\hat{c})^{d/2}}{\epsilon^{d/2}},\frac{8e
t^224^d(1+\hat{c})^d
t^{d/2}}{\epsilon^{d}(t-\hat{t})^{d/2}}\right]\exp\left(-\frac{\epsilon^2
(1-\hat{t}/t)^2 t}{32
(2+\hat{c})^2}\right).\end{split}\end{equation*}

\begin{proof}
In Section \ref{sec.proof}, we cover $\mathcal{S}_d$ with a finite
$\epsilon$-net, and prove a uniform bound on this finite set,
showing
$$
\mathrm{Pr}\left\{\sup_{\mathbf{w} \in \hat{\mathcal{S}}_d, \overline{t}\leq \hat{t}}\left|\frac{1}{t}
\sum_{i=1}^{\overline{t}}|\mathbf{w}^\top
\mathbf{x}|_{(i)}^2-\mathcal{V}\big(\frac{\overline{t}}{t}\big)\right|\geq
\epsilon/2\right\}\leq  \frac{8e t^2 3^d}{\delta^d}
\exp\left(-\frac{(1-\hat{t}/t)^2\epsilon^2  t}{32
(2+\hat{c})^2}\right).
$$
We have left to relate the uniform bound on $\mathcal{S}_d$
with the uniform bound on this finite set.



%

For any $\mathbf{w}, \mathbf{w}_1\in \mathcal{S}_d$ such
that $\|\mathbf{w}-\mathbf{w}_1\|\leq \delta$ and $\overline{t}\leq
\hat{t}$, we have
\begin{equation}\label{equ.uniformconvergesignal}
\begin{split}&
\big|\frac{1}{t}\sum_{i=1}^{\overline{t}}|{\mathbf{w}}^\top\mathbf{x}|_{(i)}^2-\frac{1}{t}
\sum_{i=1}^{\hat{t}}|{\mathbf{w}_1}^\top\mathbf{x}|_{(i)}^2
\big|\\
\leq &\max\Big(\big|\frac{1}{t} \sum_{i=1}^{\overline{t}}[
(\mathbf{w}^\top \hat{\mathbf{x}}_i)^2-(\mathbf{w}_1^\top
\hat{\mathbf{x}}_i)^2]\big|, \, \big|\frac{1}{t}
\sum_{i=1}^{\overline{t}}[ (\mathbf{w}^\top
\bar{\mathbf{x}}_i)^2-(\mathbf{w}_1^\top
\bar{\mathbf{x}}_i)^2]\big|\Big),
\end{split}
\end{equation}
where $(\hat{\mathbf{x}}_1,\cdots, \hat{\mathbf{x}}_t)$ and
$(\bar{\mathbf{x}}_1,\cdots, \bar{\mathbf{x}}_t)$ are permutations
of $(\mathbf{x}_1,\cdots, \mathbf{x}_t)$ such that
$|\mathbf{w}^\top\hat{\mathbf{x}}_i|$ and
$|\mathbf{w}_1^\top\bar{\mathbf{x}}_i|$ are non-decreasing with $i$.

To bound the right hand side of~(\ref{equ.uniformconvergesignal}),
we note that
\begin{equation}\label{equ.convergenceepsilonnet}
\begin{split}
&\left|\frac{1}{t} \sum_{i=1}^{\overline{t}}[ (\mathbf{w}^\top
\hat{\mathbf{x}}_i)^2-(\mathbf{w}_1^\top
\hat{\mathbf{x}}_i)^2]\right|= \left|\frac{1}{t}
\sum_{i=1}^{\overline{t}}[ (\mathbf{w}^\top
\hat{\mathbf{x}}_i)^2-((\mathbf{w}_1-\mathbf{w}+\mathbf{w})^\top
\hat{\mathbf{x}}_i)^2]\right|\\
= &
\frac{1}{t}\left|-\sum_{i=1}^{\overline{t}}[(\mathbf{w}_1-\mathbf{w})^\top\hat{\mathbf{x}}_i]^2
+2\sum_{i=1}^{\overline{t}}\big\{
[(\mathbf{w}_1-\mathbf{w})^\top\hat{\mathbf{x}}_i][\mathbf{w}^\top
\hat{\mathbf{x}}_i]\big\} \right|\\
\leq &\max_{\mathbf{v}\in \mathcal{S}_d} \delta^2
\frac{1}{t}\sum_{i=1}^{\overline{t}} \mathbf{v}^\top
\hat{\mathbf{x}}_i\hat{\mathbf{x}}_i^\top \mathbf{v}+2\delta
\max_{\mathbf{v}'\in \mathcal{S}_d}
(\frac{1}{t}\sum_{i=1}^{\overline{t}}|\mathbf{v}'^\top\hat{\mathbf{x}}_i|)\cdot|\mathbf{w}^\top\hat{\mathbf{x}}_{\hat{t}}|.
\end{split}
\end{equation}
Here the inequality holds because $\|\mathbf{w}-\mathbf{w}_1\|\leq
\delta$, and $|\mathbf{w}^\top \hat{\mathbf{x}}_i|$ is
non-decreasing with $i$.

Note that for all $\mathbf{v}, \mathbf{v}' \in \mathcal{S}_d$, we
have
\begin{equation*}\begin{split} (I)\quad & \max_{\mathbf{v}\in
\mathcal{S}_d} \frac{1}{t}\sum_{i=1}^{\overline{t}} \mathbf{v}^\top
\hat{\mathbf{x}}_i\hat{\mathbf{x}}_i^\top \mathbf{v}\leq
\max_{\mathbf{v}\in \mathcal{S}_d} \frac{1}{t}\sum_{i=1}^{t}
\mathbf{v}^\top \hat{\mathbf{x}}_i\hat{\mathbf{x}}_i^\top \mathbf{v}
\leq 1+\hat{c};\\
(II)\quad
&\frac{1}{t}\sum_{i=1}^{\overline{t}}|\mathbf{v}'^\top\hat{\mathbf{x}}_i|
\leq \frac{1}{t}\sum_{i=1}^{t}|\mathbf{v}'^\top\hat{\mathbf{x}}_i|
\leq\sqrt{\frac{1}{t}\sum_{i=1}^{t}|\mathbf{v}'^\top\hat{\mathbf{x}}_i|^2}\leq
\sqrt{1+\hat{c}};\\
(III)\quad &
\sum_{i=\hat{t}+1}^t|\mathbf{w}^\top\hat{\mathbf{x}}_i|^2\leq
\sum_{i=1}^t|\mathbf{w}^\top\hat{\mathbf{x}}_i|^2 \leq t(1+c); \quad
\stackrel{(a)}{\Rightarrow} \quad |\mathbf{w}^\top
\hat{\mathbf{x}}_{\hat{t}}|\leq
\sqrt{\frac{t(1+\hat{c})}{t-\hat{t}}}.\end{split}\end{equation*}Here,
$(a)$ holds because $|\mathbf{w}^\top \hat{\mathbf{x}}_i|$ is
non-decreasing with $i$. Substituting it back to the right hand side
of~(\ref{equ.convergenceepsilonnet}) we have
\[\left|\frac{1}{t} \sum_{i=1}^{\overline{t}}[ (\mathbf{w}^\top
\hat{\mathbf{x}}_i)^2-(\mathbf{w}_1^\top
\hat{\mathbf{x}}_i)^2]\right| \leq
(1+\hat{c})\delta^2+2(1+\hat{c})\delta\sqrt{\frac{t}{t-\hat{t}}}\leq\epsilon/2.\]

Similarly we have
\begin{equation*}\begin{split}&\left|\frac{1}{t}
\sum_{i=1}^{\overline{t}}\big[ (\mathbf{w}^\top
\bar{\mathbf{x}}_i)^2-(\mathbf{w}_1^\top
\bar{\mathbf{x}}_i)^2\big]\right|
=\left|\frac{1}{t}\sum_{i=1}^{\overline{t}} \big[((\mathbf{w}_1
+\mathbf{w}-\mathbf{w}_1)^\top \bar{\mathbf{x}}_i)^2
-(\mathbf{w}_1^\top \bar{\mathbf{x}}_i)^2\big]\right|\\
=& \frac{1}{t}\left|\sum_{i=1}^{\overline{t}}
[(\mathbf{w}-\mathbf{w}_1)^\top \bar{\mathbf{x}}]^2
-2\sum_{i=1}^{\overline{t}}\big\{[(\mathbf{w}^\top
-\mathbf{w}_1)^\top \bar{\mathbf{x}}_i] [\mathbf{w}_1^\top
\bar{\mathbf{x}}_i]\big\}\right|\\
\leq &\max_{\mathbf{v}\in \mathcal{S}_d} \delta^2
\frac{1}{t}\sum_{i=1}^{\overline{t}} \mathbf{v}^\top
\bar{\mathbf{x}}_i\bar{\mathbf{x}}_i^\top \mathbf{v}+2\delta
\max_{\mathbf{v}'\in \mathcal{S}_d}
(\frac{1}{t}\sum_{i=1}^{\overline{t}}|\mathbf{v}'^\top\bar{\mathbf{x}}_i|)\cdot|\mathbf{w}_1^\top\bar{\mathbf{x}}_{\hat{t}}|,
\end{split}\end{equation*}where the last inequality follows from
that $|\mathbf{w}_1^\top\bar{\mathbf{x}}_{\hat{t}}|$ is
non-decreasing with $i$. Note that the non-decreasing property also
leads to
\[|\mathbf{w}_1^\top
\bar{\mathbf{x}}_{\hat{t}}|\leq
\sqrt{\frac{t(1+\hat{c})}{t-\hat{t}}},\] which implies that
\[\big|\frac{1}{t} \sum_{i=1}^{\overline{t}}[ (\mathbf{w}^\top
\bar{\mathbf{x}}_i)^2-(\mathbf{w}_1^\top \bar{\mathbf{x}}_i)^2]\big|
\leq \epsilon/2,\] and consequently
\[\big|\frac{1}{t}\sum_{i=1}^{\overline{t}}|{\mathbf{w}}^\top\mathbf{x}|_{(i)}^2-\frac{1}{t}
\sum_{i=1}^{\overline{t}}|{\mathbf{w}_1}^\top\mathbf{x}|_{(i)}^2
\big| \leq \epsilon/2.\]

Thus,
\begin{equation*}\begin{split}&\mathrm{Pr}\left\{\sup_{\mathbf{w} \in
\mathcal{S}_d, \overline{t}\leq \hat{t}}\left|\frac{1}{t}
\sum_{i=1}^{\overline{t}}|\mathbf{w}^\top
\mathbf{x}|_{(i)}^2-\mathcal{V}\left(\frac{\overline{t}}{t}\right)\right|\geq
\epsilon\right\}\\
\leq & \mathrm{Pr}\left\{\sup_{\mathbf{w}_1 \in \hat{\mathcal{S}}_d,
\overline{t}\leq \hat{t}}\left|\frac{1}{t}
\sum_{i=1}^{\overline{t}}|\mathbf{w}_1^\top
\mathbf{x}|_{(i)}^2-\mathcal{V}\left(\frac{\overline{t}}{t}\right)\right|\geq
\epsilon/2\right\}
\\\leq &  8e t^2\frac{3^d}{\delta^d}
\exp\left(-\frac{\epsilon^2 (1-\hat{t}/t)^2 t}{32 (2+\hat{c})^2}\right)\\
=&\max \left[8e t^2\frac{3^d}{\delta_1^d}
\exp\left(-\frac{\epsilon^2 (1-\hat{t}/t)^2 t}{32
(2+\hat{c})^2}\right),\, 8e t^2\frac{3^d}{\delta_2^d}
\exp\left(-\frac{\epsilon^2
(1-\hat{t}/t)^2 t}{32 (2+\hat{c})^2}\right)\right]\\
=&\max\left[\frac{8e
6^d(1+\hat{c})^{d/2}t^2}{\epsilon^{d/2}},\frac{8e 24^d(1+\hat{c})^d
t^2}{\epsilon^{d}(1-\hat{t}/t)^{d/2}}\right]\exp\left(-\frac{\epsilon^2
(1-\hat{t}/t)^2 t}{32 (2+\hat{c})^2}\right).
\end{split}\end{equation*}
The first inequality holds because there exists $\mathbf{w}_1\in
\hat{\mathcal{S}}_d$ such that $\|\mathbf{w}-\mathbf{w}_1\|\leq
\delta$, which implies
$\big|\frac{1}{t}\sum_{i=1}^{\overline{t}}|{\mathbf{w}}^\top\mathbf{x}|_{(i)}^2-\frac{1}{t}
\sum_{i=1}^{\overline{t}}|{\mathbf{w}_1}^\top\mathbf{x}|_{(i)}^2
\big| \leq \epsilon/2$.

\end{proof}

\subsection{Proof of Corollary \ref{cor.step1c} and Lemma \ref{lem.inproofofsimplysignal}}
{\bf Corollary \ref{cor.step1c}} If $\sup_{\mathbf{w}\in \mathcal{S}_d} \big|\frac{1}{t}\sum_{i=1}^t
(\mathbf{w}^\top\mathbf{x}_i)^2-1\big|\leq \hat{c}$, then with
probability $1-\gamma$
$$
\sup_{\mathbf{w}\in \mathcal{S}_d,
\overline{t}\leq \hat{t}} \big|\frac{1}t
\sum_{i=1}^{\overline{t}}|\mathbf{w}^\top
\mathbf{x}|_{(i)}^2-\mathcal{V}\left(\frac{\overline{t}}{t}\right)\big|\leq
\epsilon_0,
$$
where \begin{equation*}\begin{split}\epsilon_0 =&
\sqrt{\frac{32(2+\hat{c})^2\big\{\max[\frac{d+4}{2}\log t
+\log{\frac{1}{\gamma}}+ \log(16e
6^d)+\frac{d}{2}\log(1+\hat{c}),\,(1-\hat{t}/t)^2]\big\}}{t(1-\hat{t}/t)^2}}\\&\,+
\sqrt{\frac{32(2+\hat{c})^2\big\{\max[(d+2)\log t
+\log{\frac{1}{\gamma}}+ \log(16e^2
24^d)+d\log(1+\hat{c})-\frac{d}{2}\log(1-\hat{t}/t),\,(1-\hat{t}/t)^2]\big\}}{t(1-\hat{t}/t)^2}}.
\end{split}\end{equation*}

\begin{proof} The proof of the corollary requires Lemma \ref{lem.inproofofsimplysignal}.

{\bf Lemma \ref{lem.inproofofsimplysignal}} For any $C_1, C_2, d', t \geq 0$, and $0<\gamma<1$.
Let
$$
\epsilon=\sqrt{\frac{\max(d'\log
t-\log(\gamma/C_1),C_2)}{tC_2}},
$$
then
$$
C_1\epsilon^{-d'}\exp(-C_2\epsilon^2 t) \leq \gamma.
$$

\begin{proof} Note that
\begin{equation*}\begin{split}
&-C_2\epsilon^2 t -d'\log \epsilon\\
=& -\max(d'\log t-\log(\gamma/C_1),C_2) -d' \log [\max(d'\log
t-\log(\gamma/C_1)/C_2,1)] +d'\log t.
\end{split}\end{equation*}
It is easy to see that the r.h.s is upper-bounded by
$\log(\gamma/C_1)$ if $d'\log t-\log(\gamma/C_1) \geq C_2$. If
$d'\log t-\log(\gamma/C_1) < C_2$, then the r.h.s equals
$-C_2+d'\log t$ which is again upper-bounded by $\log(\gamma/C_1)$
due to $d'\log t-\log(\gamma/C_1) < C_2$. Thus, we have
\[-C_2\epsilon^2 t -d'\log \epsilon \leq \log(\gamma/C_1),\] which
is equivalent to
\[C_1\epsilon^{-d'}\exp(-C_2\epsilon^2 t) \leq \gamma.\]
\end{proof}
Now to prove the corollary: let
\begin{equation*}\begin{split}\epsilon_1&=\sqrt{\frac{32(2+\hat{c})^2\big\{\max[\frac{d+4}{2}\log
t +\log{\frac{1}{\gamma}}+ \log(16e
6^d)+\frac{d}{2}\log(1+\hat{c}),\,(1-\hat{t}/t)^2]\big\}}{t(1-\hat{t}/t)^2}}\\
\epsilon_2&=\sqrt{\frac{32(2+\hat{c})^2\big\{\max[(d+2)\log t
+\log{\frac{1}{\gamma}}+ \log(16e^2
24^d)+d\log(1+\hat{c})-\frac{d}{2}\log(1-\hat{t}/t),\,(1-\hat{t}/t)^2]\big\}}{t(1-\hat{t}/t)^2}}.
\end{split}\end{equation*}
By Lemma~\ref{lem.inproofofsimplysignal}, we have
\begin{equation*}\begin{split}&
\frac{8e
6^d(1+\hat{c})^{d/2}t^2}{\epsilon_1^{d/2}}\exp\left(-\frac{\epsilon_1^2
(1-\hat{t}/t)^2 t}{32 (2+\hat{c})^2}\right)\leq \gamma/2;\\&
\frac{8e 24^d(1+\hat{c})^d
t^2}{\epsilon_2^{d}(1-\hat{t}/t)^{d/2}}\exp\left(-\frac{\epsilon_2^2
(1-\hat{t}/t)^2 t}{32 (2+\hat{c})^2}\right) \leq \gamma/2.
\end{split}\end{equation*}
By Theorem~\ref{thm.incomvar-d} we have
\begin{equation*}\begin{split} &\mathrm{Pr}\left\{\sup_{\mathbf{w}\in \mathcal{S}_d,
\overline{t}\leq \hat{t}} \big|\frac{1}t
\sum_{i=1}^{\overline{t}}|\mathbf{w}^\top
\mathbf{x}|_{(i)}^2-\mathcal{V}\left(\frac{\overline{t}}{t}\right)\big|\geq
\epsilon_0\right\}\\ &\leq \max\left[\frac{8e
t^26^d(1+\hat{c})^{d/2}}{\epsilon_0^{d/2}},\frac{8e
t^224^d(1+\hat{c})^d
t^{d/2}}{\epsilon_0^{d}(t-\hat{t})^{d/2}}\right]\exp\left(-\frac{\epsilon_0^2
(1-\hat{t}/t)^2 t}{32
(2+\hat{c})^2}\right)\\
&\leq \left[\frac{8e
t^26^d(1+\hat{c})^{d/2}}{\epsilon_0^{d/2}}+\frac{8e
t^224^d(1+\hat{c})^d
t^{d/2}}{\epsilon_0^{d}(t-\hat{t})^{d/2}}\right]\exp\left(-\frac{\epsilon_0^2
(1-\hat{t}/t)^2 t}{32 (2+\hat{c})^2}\right)\\&\leq \frac{8e
t^26^d(1+\hat{c})^{d/2}}{\epsilon_1^{d/2}}\exp\left(-\frac{\epsilon_1^2
(1-\hat{t}/t)^2 t}{32 (2+\hat{c})^2}\right)+\frac{8e
t^224^d(1+\hat{c})^d
t^{d/2}}{\epsilon_2^{d}(t-\hat{t})^{d/2}}\exp\left(-\frac{\epsilon_2^2
(1-\hat{t}/t)^2 t}{32 (2+\hat{c})^2}\right)\\&\leq
\gamma.\end{split}\end{equation*}The third inequality holds because
$\epsilon_1,\epsilon_2 \leq \epsilon_0$.
\end{proof}

\subsection{Proof of Theorem \ref{thm.Es} and Lemma \ref{le.supermartingale}}

Recall the statement of Theorem \ref{thm.Es}:

{\bf Theorem \ref{thm.Es}} With probability at least $1-\gamma$,
$\bigcup_{s=1}^{s_0}\mathcal{E}(s)$ is true.  Here
$$
s_0\triangleq(1+\epsilon)\frac{(1+\kappa)\lambda n}{\kappa};\quad\epsilon=\frac{16(1+\kappa)\log(1/\gamma)}{\kappa \lambda n}+4\sqrt{\frac{(1+\kappa)\log(1/\gamma)}{\kappa \lambda n}}.
$$
Recall that we defined the random variable $X_s$ as follows:
Let $T=\min\{s|\mathcal{E}(s)\mbox{ is true}\}$. Note that since
$\mathcal{E}(s)\in \mathcal{F}_{s-1}$, we have $\{T > s\}\in
\mathcal{F}_{s-1}$. Then define:
$$
X_s=\left\{\begin{array}{ll}  |\mathcal{O}(T-1)|+\frac{\kappa(T-1)}{1+\kappa},&\mbox{if} \,\, T \leq s; \\
|\mathcal{O}(s)|+\frac{\kappa s}{1+\kappa}, & \mbox{if} \,\, T>s.
\end{array}\right.
$$
The proof of the above theorem depends on first showing that the random variable, $X_s$, is a supermartingale.

{\bf Lemma \ref{le.supermartingale}}. $\{X_s, \mathcal{F}_s\}$ is a supermartingale.

\begin{proof} Observe that $X_s \in \mathcal{F}_s$. We next show that
$\mathbb{E}(X_s|\mathcal{F}_{s-1})\leq X_{s-1}$ by enumerating the
following three cases:

Case 1, $T> s$: Thus we have $\mathcal{E}^c(s)$ is true. By
Lemma~\ref{thm.probabilitytoremoveanoutlier},
\[\mathbb{E}(X_s-X_{s-1}|\mathcal{F}_{s-1}) = \mathbb{E}\left(\mathcal{O}(s)-\mathcal{O}(s-1)+\frac{\kappa}{1+\kappa}\Big|\mathcal{F}_{s-1}\right)=\frac{\kappa}{1+\kappa}-\mathrm{Pr}\left(\overline{r}(s)\in \mathcal{O}(s-1)\right)<0.\]

Case 2, $T=s$: By definition of $X_s$ we have
$X_s=\mathcal{O}(s-1)+\kappa(s-1)/(1+\kappa)=X_{s-1}$.

Case 3, $T<s$: Since both $T$ and $s$ are integer, we have $T\leq
s-1$. Thus, $X_{s-1}=\mathcal{O}(T-1)+\kappa(T-1)/(1+\kappa)=X_s$.

Combining all three cases shows that
$\mathbb{E}(X_s|\mathcal{F}_{s-1})\leq X_{s-1}$, which proves the
lemma.
\end{proof}

Next, we prove Theorem~\ref{thm.Es}.

\begin{proof}
Note that
\begin{equation}\label{equ.proofofEs}\mathrm{Pr}\left(\bigcap_{s=1}^{s_0}
\mathcal{E}(s)^c\right)=\mathrm{Pr}\left(T>s_0\right)\leq
\mathrm{Pr}\left(X_{s_0}\geq \frac{\kappa
s_0}{1+\kappa}\right)=\mathrm{Pr}\left(X_{s_0}\geq
(1+\epsilon)\lambda n\right),\end{equation} where the inequality is
due to $|\mathcal{O}(s)|$ being non-negative.

Let $y_s\triangleq X_s-X_{s-1}$, where recall that $X_0=\lambda n$.
Consider the following sequence:
$$
y'_s\triangleq y_s-\mathbb{E}(y_s|y_1, \cdots, y_{s-1}).
$$
Observe that $\{y'_s\}$ is a martingale difference process w.r.t.
$\{\mathcal{F}_s\}$. Since $\{X_s\}$ is a supermartingale,
$\mathbb{E}(y_s|y_1, \cdots, y_{s-1})\leq 0$ a.s. Therefore, the
following holds a.s.,
\begin{equation}\label{equ.martigaleproof}
X_s-X_0=\sum_{i=1}^s y_i  =\sum_{i=1}^s y'_i +\sum_{i=1}^s
\mathbb{E}(y_i|y_1, \cdots, y_{i-1}) \leq \sum_{i=1}^s y'_i .
\end{equation}
By definition, $|y_s|\leq 1$, and hence $|y_s'|\leq 2$. Now apply
Azuma's inequality
\begin{equation*}\begin{split}&\mathrm{Pr}(X_{s_0}\geq (1+\epsilon)\lambda
n)\\&\leq \mathrm{Pr}( (\sum_{i=1}^{s_0} y'_i)\geq \epsilon\lambda
n)\\
&\leq \exp(-(\epsilon\lambda n)^2/8s_0)
\\&=\exp\left(-\frac{(\epsilon\lambda
n)^2\kappa}{8(1+\epsilon)(1+\kappa)\lambda n}\right)
\\&\leq \exp\left(-\frac{(\epsilon\lambda
n)^2\kappa}{8(1+\epsilon)(1+\kappa)\lambda n}\right)\\
&\leq \max\left( \exp\left(-\frac{\epsilon^2 \lambda n
\kappa}{16(1+\kappa)}\right),\, \exp\left(-\frac{\epsilon \lambda n
\kappa}{16(1+\kappa)}\right)\right).\end{split}\end{equation*} We
claim that the right-hand-side is upper bounded by $\gamma$. This is
because:
\[\epsilon \geq \sqrt{\frac{16(1+\kappa)\log (1/\gamma)}{\kappa \lambda n}}; \quad \Rightarrow \quad \exp\left(-\frac{\epsilon^2 \lambda n
\kappa}{16(1+\kappa)}\right) \leq \gamma;\] and
\[\epsilon \geq \frac{16(1+\kappa)\log (1/\gamma)}{\kappa \lambda n}; \quad \Rightarrow \quad \exp\left(-\frac{\epsilon \lambda n
\kappa}{16(1+\kappa)}\right) \leq \gamma;\]

Substitute into~(\ref{equ.proofofEs}), the theorem follows.
\end{proof}

\subsection{Proof of Lemmas \ref{lem.step6hsvsoverlineh} and \ref{lem.step6hsvshstar} and Theorems \ref{thm.step3ugly} and \ref{thm.main}}

We now prove all the intermediate results used in Section \ref{ssec.step3}.

{\bf Lemma \ref{lem.step6hsvsoverlineh}}. If $\mathcal{E}(s)$ is true for some $s\leq
s_0$, and there exists $\epsilon_1, \epsilon_2, \overline{c}$ such
that
\begin{equation*}\begin{split}
(I)\quad & \sup_{\mathbf{w}\in \mathcal{S}_d} \big|\frac{1}{t}
\sum_{i=1}^{t-s_0}|\mathbf{w}^\top
\mathbf{x}|_{(i)}^2-\mathcal{V}(\frac{t-s_0}{t})\big|\leq
\epsilon_1\\(II)\quad & \sup_{\mathbf{w}\in \mathcal{S}_d}
\big|\frac{1}{t} \sum_{i=1}^{t}|\mathbf{w}^\top
\mathbf{x}_i|^2-1\big|\leq \epsilon_2\\
(III)\quad &\sup_{\mathbf{w}\in\mathcal{S}_m} \frac{1}t
\sum_{i=1}^{t}|\mathbf{w}^\top \mathbf{n}_{i}|^2\leq
\overline{c},\end{split}\end{equation*} then
\[\frac{1}{1+\kappa}\left[(1-\epsilon_1)\mathcal{V}\left(\frac{t-s_0}{t}\right)\overline{H}
 -2\sqrt{(1+\epsilon_2)\overline{c}d \overline{H}}\right]\leq
(1+\epsilon_2)H_s
+2\sqrt{(1+\epsilon_2)\overline{c}dH_s}+\overline{c}.\]

\begin{proof}
If $\mathcal{E}(s)$ is true, then we have
$$
\sum_{j=1}^d \sum_{\mathbf{z}_i\in
\mathcal{Z}(s-1)} (\mathbf{w}_j(s)^{\top} \mathbf{z}_i)^2 \geq
\frac{1}{\kappa} \sum_{j=1}^d \sum_{\mathbf{o}_i\in
\mathcal{O}(s-1)} (\mathbf{w}_j(s)^{\top} \mathbf{o}_i)^2.
$$
Recall that $\mathcal{Y}(s-1) =\mathcal{Z}(s-1)\bigcup \mathcal{O}(s-1)$,
and that $\mathcal{Z}(s-1)$ and $\mathcal{O}(s-1)$ are disjoint.
 We thus have
\begin{equation}\label{equ.estrue}\frac{1}{1+\kappa} \sum_{j=1}^d \sum_{\mathbf{y}_i\in
\mathcal{Y}(s-1)} (\mathbf{w}_j(s)^{\top} \mathbf{y}_i)^2
\leq\sum_{j=1}^d \sum_{\mathbf{z}_i\in \mathcal{Z}(s-1)}
(\mathbf{w}_j(s)^{\top} \mathbf{z}_i)^2.\end{equation} Since
$\mathbf{w}_1(s),\cdots, \mathbf{w}_d(s)$ are the solution of the
$s^{th}$ stage, the following holds by definition of the algorithm
\begin{equation}\label{equ.s-stage-optim}\sum_{j=1}^d \sum_{\mathbf{y}_i\in \mathcal{Y}(s-1)}
(\overline{\mathbf{w}}_j^{\top} \mathbf{y}_i)^2 \leq \sum_{j=1}^d
\sum_{\mathbf{y}_i\in \mathcal{Y}(s-1)} (\mathbf{w}_j(s)^{\top}
\mathbf{y}_i)^2.\end{equation} Further note that by
$\mathcal{Z}(s-1) \subseteq \mathcal{Y}(s-1)$ and
$\mathcal{Z}(s-1)\subseteq \mathcal{Z}$, we have
\[\sum_{j=1}^d \sum_{\mathbf{z}_i\in
\mathcal{Z}(s-1)} (\overline{\mathbf{w}}_j^{\top}
\mathbf{z}_i)^2\leq \sum_{j=1}^d \sum_{\mathbf{y}_i\in
\mathcal{Y}(s-1)} (\overline{\mathbf{w}}_j^{\top} \mathbf{y}_i)^2,\]
and
\[\sum_{j=1}^d \sum_{\mathbf{z}_i\in \mathcal{Z}(s-1)}
(\mathbf{w}_j(s)^{\top} \mathbf{z}_i)^2\leq \sum_{j=1}^d
\sum_{\mathbf{z}_i\in \mathcal{Z}} (\mathbf{w}_j(s)^{\top}
\mathbf{z}_i)^2=\sum_{j=1}^d \sum_{i=1}^t (\mathbf{w}_j(s)^{\top}
\mathbf{z}_i)^2.\] Substituting them into~(\ref{equ.estrue}) and
~(\ref{equ.s-stage-optim}) we have
\[\frac{1}{1+\kappa}\sum_{j=1}^d \sum_{\mathbf{z}_i\in
\mathcal{Z}(s-1)} (\overline{\mathbf{w}}_j^{\top} \mathbf{z}_i)^2
\leq \sum_{j=1}^d \sum_{i=1}^t (\mathbf{w}_j(s)^{\top}
\mathbf{z}_i)^2.\] Note that $|\mathcal{Z}(s-1)| \geq t-(s-1) \geq
t-s_0$, hence for all $j=1,\cdots, d$,
\begin{equation*}\begin{split}
&\sum_{i=1}^{t-s_0}\big|\overline{\mathbf{w}}_j
\mathbf{z}\big|_{(i)}^2 \leq
\sum_{i=1}^{|\mathcal{Z}(s-1)|}\big|\overline{\mathbf{w}}_j
\mathbf{z}\big|_{(i)}^2
 \leq  \sum_{\mathbf{z}_i\in
\mathcal{Z}(s-1)}(\overline{\mathbf{w}}_j \mathbf{z}_i)^2,
\end{split}
\end{equation*}
which in turn implies
\[\frac{1}{1+\kappa}\sum_{i=1}^{t-s_0}\big|\overline{\mathbf{w}}_j
\mathbf{z}\big|_{(i)}^2 \leq \sum_{j=1}^d \sum_{i=1}^t
(\mathbf{w}_j(s)^{\top} \mathbf{z}_i)^2.\] By
Corollary~\ref{cor.step1dgeneralt} and Corollary~\ref{cor.step1dtis1}
we conclude
\[\frac{1}{1+\kappa}\left[(1-\epsilon_1)\mathcal{V}\left(\frac{t-s_0}{t}\right)\overline{H}
 -2\sqrt{(1+\epsilon_2)\overline{c}d \overline{H}}\right]\leq
(1+\epsilon_2)H_s
+2\sqrt{(1+\epsilon_2)\overline{c}dH_s}+\overline{c}.\]
\end{proof}

{\bf Lemma \ref{lem.step6hsvshstar}}. Fix a $\hat{t} \leq t$. If $\sum_{j=1}^d\overline{V}_{\hat{t}}(\mathbf{w}_j)
\geq \sum_{j=1}^d\overline{V}_{\hat{t}}(\mathbf{w}'_j)$,  and there
exists $\epsilon_1, \epsilon_2, \overline{c}$ such that
\begin{equation*}\begin{split}
(I)\quad & \sup_{\mathbf{w}\in \mathcal{S}_d} \big|\frac{1}{t}
\sum_{i=1}^{\hat{t}}|\mathbf{w}^\top
\mathbf{x}|_{(i)}^2-\mathcal{V}(\frac{\hat{t}}{t})\big|\leq
\epsilon_1,\\(II)\quad & \sup_{\mathbf{w}\in \mathcal{S}_d}
\big|\frac{1}{t} \sum_{i=1}^{\hat{t}-\frac{\lambda
t}{1-\lambda}}|\mathbf{w}^\top
\mathbf{x}|_{(i)}^2-\mathcal{V}\left(\frac{\hat{t}}{t}-\frac{\lambda}{1-\lambda}\right)\big|\leq
\epsilon_1,\\(III)\quad & \sup_{\mathbf{w}\in \mathcal{S}_d}
\big|\frac{1}{t} \sum_{i=1}^{t}|\mathbf{w}^\top
\mathbf{x}_i|^2-1\big|\leq \epsilon_2,\\
(IV)\quad &\sup_{\mathbf{w}\in\mathcal{S}_m} \frac{1}t
\sum_{i=1}^{t}|\mathbf{w}^\top \mathbf{n}_{i}|^2\leq
\overline{c},\end{split}\end{equation*} then
\begin{equation*}\begin{split}&(1-\epsilon_1)\mathcal{V}\left(\frac{\hat{t}}{t}-\frac{\lambda}{1-\lambda}\right)H(\mathbf{w}_1'\cdots, \mathbf{w}_d')
 -2\sqrt{(1+\epsilon_2)\overline{c}d H(\mathbf{w}_1'\cdots,
\mathbf{w}_d')}\\\leq & (1+\epsilon_1)H(\mathbf{w}_1\cdots,
\mathbf{w}_d)\mathcal{V}\left(\frac{\hat{t}}{t}\right)
+2\sqrt{(1+\epsilon_2)\overline{c}dH(\mathbf{w}_1\cdots,
\mathbf{w}_d)}+\overline{c}.\end{split}\end{equation*}

\begin{proof} Recall that $\overline{V}_{\hat{t}}(\mathbf{w})=\frac{1}{n}\sum_{i=1}^{\hat{t}} |\mathbf{w}^\top
\mathbf{y}|_{(i)}^2$. Since $\mathcal{Y}\subset \mathcal{Z}$ and
$|\mathcal{Z}\backslash\mathcal{Y}| =\lambda n=\lambda
t/(1-\lambda)$, we have
\[\sum_{i=1}^{\hat{t}-\frac{\lambda
t}{1-\lambda}} |\mathbf{w}^\top \mathbf{z}|_{(i)}^2 \leq
\sum_{i=1}^{\hat{t}} |\mathbf{w}^\top \mathbf{y}|_{(i)}^2 \leq
\sum_{i=1}^{\hat{t}} |\mathbf{w}^\top \mathbf{z}|_{(i)}^2.\] By
assumption $\sum_{j=1}^d\overline{V}_{\hat{t}}(\mathbf{w}_j) \geq
\sum_{j=1}^d\overline{V}_{\hat{t}}(\mathbf{w}'_j)$, we have
\[\sum_{j=1}^d\sum_{i=1}^{\hat{t}-\frac{\lambda
t}{1-\lambda}}|\mathbf{w}_j'^\top \mathbf{y}|_{(i)}^2 \leq
\sum_{j=1}^d\sum_{i=1}^{\hat{t}}|\mathbf{w}_j^\top
\mathbf{y}|_{(i)}^2.\] By Corollary~\ref{cor.step1dgeneralt} and
Corollary~\ref{cor.step1dtis1} we conclude
\begin{equation*}\begin{split}&(1-\epsilon_1)\mathcal{V}\left(\frac{\hat{t}}{t}-\frac{\lambda}{1-\lambda}\right)H(\mathbf{w}_1'\cdots, \mathbf{w}_d')
 -2\sqrt{(1+\epsilon_2)\overline{c}d H(\mathbf{w}_1'\cdots,
\mathbf{w}_d')}\\\leq & (1+\epsilon_1)H(\mathbf{w}_1\cdots,
\mathbf{w}_d)\mathcal{V}\left(\frac{\hat{t}}{t}\right)
+2\sqrt{(1+\epsilon_2)\overline{c}dH(\mathbf{w}_1\cdots,
\mathbf{w}_d)}+\overline{c}.\end{split}\end{equation*}
\end{proof}

{\bf Theorem \ref{thm.step3ugly}}. If $\bigcup_{s=1}^{s_0}\mathcal{E}(s)$ is true, and there
exists $\epsilon_1<1, \epsilon_2, \overline{c}$ such that
\begin{equation*}\begin{split}
(I)\quad & \sup_{\mathbf{w}\in \mathcal{S}_d} \big|\frac{1}{t}
\sum_{i=1}^{t-s_0}|\mathbf{w}^\top
\mathbf{x}|_{(i)}^2-\mathcal{V}(\frac{t-s_0}{t})\big|\leq
\epsilon_1\\
(II)\quad & \sup_{\mathbf{w}\in \mathcal{S}_d} \big|\frac{1}{t}
\sum_{i=1}^{\hat{t}}|\mathbf{w}^\top
\mathbf{x}|_{(i)}^2-\mathcal{V}(\frac{\hat{t}}{t})\big|\leq
\epsilon_1\\(III)\quad & \sup_{\mathbf{w}\in \mathcal{S}_d}
\big|\frac{1}{t} \sum_{i=1}^{\hat{t}-\frac{\lambda
t}{1-\lambda}}|\mathbf{w}^\top
\mathbf{x}|_{(i)}^2-\mathcal{V}\left(\frac{\hat{t}}{t}-\frac{\lambda
}{1-\lambda}\right)\big|\leq \epsilon_1\\(IV)\quad &
\sup_{\mathbf{w}\in \mathcal{S}_d} \big|\frac{1}{t}
\sum_{i=1}^{t}|\mathbf{w}^\top
\mathbf{x}_i|^2-1\big|\leq \epsilon_2\\
(V)\quad &\sup_{\mathbf{w}\in\mathcal{S}_m} \frac{1}t
\sum_{i=1}^{t}|\mathbf{w}^\top \mathbf{n}_{i}|^2\leq
\overline{c},\end{split}\end{equation*} then
\begin{equation}\begin{split}\label{equ.step6noprobability}&\frac{H^*}{\overline{H}} \geq \frac{(1-\epsilon_1)^2 \mathcal{V}\left(\frac{\hat{t}}{t}-\frac{\lambda}{1-\lambda}\right)\mathcal{V}\left(\frac{t-s_0}{t}\right)}{(1+\epsilon_1)(1+\epsilon_2)(1+\kappa)\mathcal{V}\left(\frac{\hat{t}}{t}\right)}\\
&\quad-\left[\frac{(2\kappa+4)(1-\epsilon_1)\mathcal{V}\left(\frac{\hat{t}}{t}-\frac{\lambda}{1-\lambda}\right)\sqrt{(1+\epsilon_2)\overline{c}d}+4(1+\kappa)(1+\epsilon_2)\sqrt{(1+\epsilon_2)\overline{c}d}}{(1+\epsilon_1)(1+\epsilon_2)(1+\kappa)\mathcal{V}\left(\frac{\hat{t}}{t}\right)}\right](\overline{H})^{-1/2}\\
&\quad -\left[\frac{(1-\epsilon_1)\mathcal{V}\left(\frac{\hat{t}}{t}-\frac{\lambda}{1-\lambda}\right)\overline{c}+(1+\epsilon_2)\overline{c}}{(1+\epsilon_1)(1+\epsilon_2)\mathcal{V}\left(\frac{\hat{t}}{t}\right)}\right](\overline{H})^{-1}.
\end{split}\end{equation}

\begin{proof} Since $\bigcup_{s=1}^{s_0}\mathcal{E}(s)$ is true,
there exists a $s'\leq s_0$ such that $\mathcal{E}(s')$ is true.
By Lemma~\ref{lem.step6hsvsoverlineh} we have
\[\frac{1}{1+\kappa}\left[(1-\epsilon_1)\mathcal{V}\left(\frac{t-s_0}{t}\right)\overline{H}
 -2\sqrt{(1+\epsilon_2)\overline{c}d \overline{H}}\right]\leq
(1+\epsilon_2)H_{s'}
+2\sqrt{(1+\epsilon_2)\overline{c}dH_{s'}}+\overline{c}.\]
By the definition of the algorithm, we have $\sum_{j=1}^d\overline{V}_{\hat{t}}(\mathbf{w}^*_j)
\geq \sum_{j=1}^d\overline{V}_{\hat{t}}(\mathbf{w}_j(s'))$, which by Lemma~\ref{lem.step6hsvshstar} implies
\begin{equation*}\begin{split}&(1-\epsilon_1)\mathcal{V}\left(\frac{\hat{t}}{t}-\frac{\lambda}{1-\lambda}\right)H_{s'}
 -2\sqrt{(1+\epsilon_2)\overline{c}d H_{s'}}\leq (1+\epsilon_1)H^*\mathcal{V}\left(\frac{\hat{t}}{t}\right)
+2\sqrt{(1+\epsilon_2)\overline{c}dH^*}+\overline{c}.\end{split}\end{equation*}
By definition, $H_{s'}, H^* \leq \overline{H}$. Thus we have
\begin{equation*}\begin{split}(I)\quad&\frac{1}{1+\kappa}\left[(1-\epsilon_1)\mathcal{V}\left(\frac{t-s_0}{t}\right)\overline{H}
 -2\sqrt{(1+\epsilon_2)\overline{c}d \overline{H}}\right]\leq
(1+\epsilon_2)H_{s'}
+2\sqrt{(1+\epsilon_2)\overline{c}d\overline{H}}+\overline{c};\\(II)\quad&(1-\epsilon_1)\mathcal{V}\left(\frac{\hat{t}}{t}-\frac{\lambda}{1-\lambda}\right)H_{s'}
 -2\sqrt{(1+\epsilon_2)\overline{c}d \overline{H}}\leq (1+\epsilon_1)\mathcal{V}\left(\frac{\hat{t}}{t}\right)H^*
+2\sqrt{(1+\epsilon_2)\overline{c}d\overline{H}}+\overline{c}.\end{split}\end{equation*}Rearrange the inequalities, we have
\begin{equation*}\begin{split}
&(I)\quad(1-\epsilon_1)\mathcal{V}\big(\frac{t-s_0}{t}\big)\overline{H}-(2\kappa+4)\sqrt{(1+\epsilon_2)\overline{c}d\overline{H}}-(1+\kappa)\overline{c}\leq (1+\kappa)(1+\epsilon_2)H_{s'};
\\ &(II)\quad (1-\epsilon_1) \mathcal{V}\left(\frac{\hat{t}}{t}-\frac{\lambda}{1-\lambda}\right) H_{s'} \leq (1+\epsilon_1)\mathcal{V}\left(\frac{\hat{t}}{t}\right) H^*+4\sqrt{(1+\epsilon_2)\overline{c}d\overline{H}}+\overline{c}.\end{split}\end{equation*}
Simplify the inequality. We get
\begin{equation*}\begin{split}&\frac{H^*}{\overline{H}} \geq \frac{(1-\epsilon_1)^2 \mathcal{V}\left(\frac{\hat{t}}{t}-\frac{\lambda}{1-\lambda}\right)\mathcal{V}\left(\frac{t-s_0}{t}\right)}{(1+\epsilon_1)(1+\epsilon_2)(1+\kappa)\mathcal{V}\left(\frac{\hat{t}}{t}\right)}\\
&\quad-\left[\frac{(2\kappa+4)(1-\epsilon_1)\mathcal{V}\left(\frac{\hat{t}}{t}-\frac{\lambda}{1-\lambda}\right)\sqrt{(1+\epsilon_2)\overline{c}d}+4(1+\kappa)(1+\epsilon_2)\sqrt{(1+\epsilon_2)\overline{c}d}}{(1+\epsilon_1)(1+\epsilon_2)(1+\kappa)\mathcal{V}\left(\frac{\hat{t}}{t}\right)}\right](\overline{H})^{-1/2}\\
&\quad -\left[\frac{(1-\epsilon_1)\mathcal{V}\left(\frac{\hat{t}}{t}-\frac{\lambda}{1-\lambda}\right)\overline{c}+(1+\epsilon_2)\overline{c}}{(1+\epsilon_1)(1+\epsilon_2)\mathcal{V}\left(\frac{\hat{t}}{t}\right)}\right](\overline{H})^{-1}.
\end{split}\end{equation*}
\end{proof}

{\bf Theorem \ref{thm.main}}. Let $\tau=\max(m/n, 1)$. There exists
a universal constant $c_0$ and a constant $C$ which can possible
depend on $\hat{t}/t$, $\lambda$, $d$, $\mu$ and $\kappa$, such that
for any $\gamma <1$, if $n/\log^4 n\geq \log^6(1/\gamma)$, then with
probability $1-\gamma$ the following holds
\begin{equation*}\begin{split}\frac{H^*}{\overline{H}}\geq \frac{
\mathcal{V}\left(\frac{\hat{t}}{t}
-\frac{\lambda}{1-\lambda}\right)\mathcal{V}\left(1-\frac{\lambda(1+\kappa)}{(1-\lambda)\kappa}\right)}
{(1+\kappa)\mathcal{V}\left(\frac{\hat{t}}{t}\right)}-\left[\frac{8\sqrt{c_0\tau
d}}{\mathcal{V}\left(\frac{\hat{t}}{t}\right)}\right](\overline{H})^{-1/2}
-\left[\frac{2c_0
\tau}{\mathcal{V}\left(\frac{\hat{t}}{t}\right)}\right](\overline{H})^{-1}
-C \frac{\log^2 n \log^3(1/\gamma)}{\sqrt{n}}.
\end{split}\end{equation*}

\begin{proof} We need to bound  all diminishing
terms in the r.h.s. of~(\ref{equ.step6noprobability}). We need to lower bound $\mathcal{V}((t-s_0)/{t})$ using the
following lemma.

\begin{lemma}
$$
\mathcal{V}\left(\frac{t-s_0}{t}\right) \geq
\mathcal{V}\left(1-\frac{\lambda(1+\kappa)}{(1-\lambda)\kappa}\right)-\epsilon,
$$
where $\epsilon\leq c \frac{\log(1/\gamma)}{n}+
c\sqrt{\frac{\log(1/\gamma)}{n}}$.
\end{lemma}
\begin{proof} Given $a^- < a^+ <1$, by the definition of
$\mathcal{V}$ we have
\[\frac{\mathcal{V}(a^+)-\mathcal{V}(a^-)}{a^+-a^-} \leq \frac{1-\mathcal{V}(a^+)}{1-a^+}.\]
Re-arranging, we have
\[\mathcal{V}(a^-)\geq \frac{1-a^-}{1-a^+}\mathcal{V}(a^+)-\frac{a^+-a^-}{1-a^+}\geq \mathcal{V}(a^+)-\frac{a^+-a^-}{1-a^+}.\]
Recall $s_0=(1+\epsilon)(1+\kappa)\lambda n/\kappa
=(1+\epsilon)(1+\kappa) \lambda t/(\kappa (1-\lambda))$. Let
$s'=(1+\kappa) \lambda t/(\kappa (1-\lambda))$. Take $a^+=t-s'$, and
$a^-=t-s_0$, the lemma follows.
\end{proof}

We also need the following two lemmas. The proofs are
straightforward.
\begin{lemma}\label{lem.step6notnecessary0} For any $0\leq \alpha_1, \alpha_2 \leq 1$ and $c>0$, we have
\[1-\alpha \leq 1/(1+\alpha);\quad (1-\alpha_1)(1-\alpha_2) \leq 1-(\alpha_1+\alpha_2);\quad \sqrt{c+\alpha_1} \leq \sqrt{c}+\alpha_1.\]
\end{lemma}
\begin{lemma}\label{lem.step6notnecessary} If
$n/\log^4 n \geq \log^6(1/\gamma)$, then
\[\max\left(\frac{\log(1/\gamma)}{n}, \sqrt{\frac{\log n}{n}}, \sqrt{\frac{\log (1/\gamma)}{n}}, \frac{\log^{2.5}n\log^{3.5}(1/\gamma)}{n}\right) \leq  \frac{\log^2 n \log^3(1/\gamma)}{\sqrt{n}} \leq 1.\]
\end{lemma}

Recall that with probability $1-\gamma$, $\overline{c}\leq
c_0\tau+c\frac{\log(1/\gamma)}{n}$ where $c_0$ is a universal
constant, and the constant $c$ depends on $\kappa$, $\hat{t}/t$,
$\lambda$, $d$ and $\mu$. We denote $\overline{c}-c_0 \tau$ by
$\epsilon_c$. Iteratively applying
Lemma~\ref{lem.step6notnecessary0},  we have the following holds
when $\epsilon_1, \epsilon_2, \epsilon_c <1$,
\begin{equation*}\begin{split}&\frac{H^*}{\overline{H}} \geq \frac{(1-\epsilon_1)^2 \mathcal{V}\left(\frac{\hat{t}}{t}-\frac{\lambda}{1-\lambda}\right)\mathcal{V}\left(\frac{t-s_0}{t}\right)}{(1+\epsilon_1)(1+\epsilon_2)(1+\kappa)\mathcal{V}\left(\frac{\hat{t}}{t}\right)}\\
&\quad-\left[\frac{(2\kappa+4)(1-\epsilon_1)\mathcal{V}\left(\frac{\hat{t}}{t}-\frac{\lambda}{1-\lambda}\right)\sqrt{(1+\epsilon_2)\overline{c}d}+4(1+\kappa)(1+\epsilon_2)\sqrt{(1+\epsilon_2)\overline{c}d}}{(1+\epsilon_1)(1+\epsilon_2)(1+\kappa)\mathcal{V}\left(\frac{\hat{t}}{t}\right)}\right](\overline{H})^{-1/2}\\
&\quad -\left[\frac{(1-\epsilon_1)\mathcal{V}
\left(\frac{\hat{t}}{t}-\frac{\lambda}{1-\lambda}\right)\overline{c}+(1+\epsilon_2)
\overline{c}}{(1+\epsilon_1)(1+\epsilon_2)
\mathcal{V}\left(\frac{\hat{t}}{t}\right)}\right](\overline{H})^{-1}\\
&\geq \frac{(1-\epsilon_1)^2 \mathcal{V}\left(\frac{\hat{t}}{t}
-\frac{\lambda}{1-\lambda}\right)\mathcal{V}\left(\frac{t-s'}{t}\right)}
{(1+\epsilon_1)(1+\epsilon_2)(1+\kappa)\mathcal{V}\left(\frac{\hat{t}}{t}\right)}-\epsilon-\left[\frac{4\sqrt{\overline{c}d}+4\sqrt{(1+\epsilon_2)\overline{c}d}}{\mathcal{V}\left(\frac{\hat{t}}{t}\right)}\right](\overline{H})^{-1/2}
-\left[\frac{2\overline{c}}{\mathcal{V}\left(\frac{\hat{t}}{t}\right)}\right](\overline{H})^{-1}\\
&\geq \frac{(1-\epsilon_1)^3(1-\epsilon_2)
\mathcal{V}\left(\frac{\hat{t}}{t}
-\frac{\lambda}{1-\lambda}\right)\mathcal{V}\left(\frac{t-s'}{t}\right)}
{(1+\kappa)\mathcal{V}\left(\frac{\hat{t}}{t}\right)}-\epsilon-\left[\frac{4\sqrt{\overline{c}d}+4\sqrt{(1+\epsilon_2)\overline{c}d}}{\mathcal{V}\left(\frac{\hat{t}}{t}\right)}\right](\overline{H})^{-1/2}
-\left[\frac{2\overline{c}}{\mathcal{V}\left(\frac{\hat{t}}{t}\right)}\right](\overline{H})^{-1}\\
&\geq \frac{(1- 15\max(\epsilon_1,\,\epsilon_2))
\mathcal{V}\left(\frac{\hat{t}}{t}
-\frac{\lambda}{1-\lambda}\right)\mathcal{V}\left(\frac{t-s'}{t}\right)}
{(1+\kappa)\mathcal{V}\left(\frac{\hat{t}}{t}\right)}-\epsilon\\&\quad-\left[\frac{4\sqrt{(c_0\tau+\epsilon_c)d}
+4(1+\epsilon_2)\sqrt{(c_0\tau+\epsilon_c)d}}{\mathcal{V}\left(\frac{\hat{t}}{t}\right)}\right](\overline{H})^{-1/2}
-\left[\frac{2(c_0\tau+\epsilon_c)}{\mathcal{V}\left(\frac{\hat{t}}{t}\right)}\right](\overline{H})^{-1}\\
&\geq \frac{(1- 15\max(\epsilon_1,\,\epsilon_2))
\mathcal{V}\left(\frac{\hat{t}}{t}
-\frac{\lambda}{1-\lambda}\right)\mathcal{V}\left(\frac{t-s'}{t}\right)}
{(1+\kappa)\mathcal{V}\left(\frac{\hat{t}}{t}\right)}-\epsilon\\&\quad-\left[\frac{4(\sqrt{c_0
\tau}+\epsilon_c)\sqrt{d}+4(1+\epsilon_2)(\sqrt{c_0
\tau}+\epsilon_c)\sqrt{d}}{\mathcal{V}\left(\frac{\hat{t}}{t}\right)}\right](\overline{H})^{-1/2}
-\left[\frac{2(c_0
\tau+\epsilon_c)}{\mathcal{V}\left(\frac{\hat{t}}{t}\right)}\right](\overline{H})^{-1}.
\end{split}\end{equation*}

Recall that  with probability $1-\gamma$,  $\epsilon_2\leq
c\frac{\log^2 n \log^3{1/\gamma}}{\sqrt{n}}$, $\epsilon\leq c
\frac{\log(1/\gamma)}{n}+ c\sqrt{\frac{\log(1/\gamma)}{n}}$,
$\overline{c}\leq c_0 \tau+c\frac{\log(1/\gamma)}{n}$. Furthermore,
$\epsilon_1 \leq c\sqrt{\frac{\log n+\log (1/\gamma)}{n}}
+c\frac{\log^{2.5}n\log^{3.5}(1/\gamma)}{n}$ if $\hat{t}/t=\eta<1$,
and $\epsilon_1 \leq \max(c\sqrt{\frac{\log n+\log (1/\gamma)}{n}}
+c\frac{\log^{2.5}n\log^{3.5}(1/\gamma)}{n}, \epsilon_2)$ if
$\hat{t}=t$. Here, $c_0$ is a universal constant, and the constant
$c$ depends on $\kappa$, $\eta$, $\lambda$, $d$ and $\mu$. Further
note by Lemma~\ref{lem.step6notnecessary} we can bound all
diminishing terms by $ \frac{\log^2 n \log^3(1/\gamma)}{\sqrt{n}}$.
Therefore, we have when $\epsilon_1, \epsilon_2, \epsilon_c <1$,
\begin{equation*}\begin{split}\frac{H^*}{\overline{H}} \geq \frac{
\mathcal{V}\left(\frac{\hat{t}}{t}
-\frac{\lambda}{1-\lambda}\right)\mathcal{V}\left(1-\frac{\lambda(1+\kappa)}{(1-\lambda)\kappa}\right)}
{(1+\kappa)\mathcal{V}\left(\frac{\hat{t}}{t}\right)}-\left[\frac{8\sqrt{c_0\tau
d}}{\mathcal{V}\left(\frac{\hat{t}}{t}\right)}\right](\overline{H})^{-1/2}
-\left[\frac{2c_0\tau}{\mathcal{V}\left(\frac{\hat{t}}{t}\right)}\right](\overline{H})^{-1}
-C_1 \frac{\log^2 n \log^3(1/\gamma)}{\sqrt{n}}.
\end{split}\end{equation*}
On the other hand, when $\max(\epsilon_1, \epsilon_2,
\epsilon_c)\geq 1$, since by Lemma~\ref{lem.step6notnecessary},
$\max(\epsilon_1, \epsilon_2, \epsilon_c) \leq C_2 \frac{\log^2 n
\log^3(1/\gamma)}{\sqrt{n}}$ for some constant $C_2$. Thus, $ C_2
\frac{\log^2 n \log^3(1/\gamma)}{\sqrt{n}} \geq 1$. Therefore, when
$\max(\epsilon_1, \epsilon_2, \epsilon_c)\geq 1$,
\begin{equation*}\begin{split}\frac{H^*}{\overline{H}} \geq 0 \geq \frac{
\mathcal{V}\left(\frac{\hat{t}}{t}
-\frac{\lambda}{1-\lambda}\right)\mathcal{V}\left(1-\frac{\lambda(1+\kappa)}{(1-\lambda)\kappa}\right)}
{(1+\kappa)\mathcal{V}\left(\frac{\hat{t}}{t}\right)}-\left[\frac{8\sqrt{c_0
\tau
d}}{\mathcal{V}\left(\frac{\hat{t}}{t}\right)}\right](\overline{H})^{-1/2}
-\left[\frac{2c_0
\tau}{\mathcal{V}\left(\frac{\hat{t}}{t}\right)}\right](\overline{H})^{-1}
-C_2 \frac{\log^2 n \log^3(1/\gamma)}{\sqrt{n}}.
\end{split}\end{equation*}
Let $C=\max(C_1, C_2)$, we proved the that
\begin{equation*}\begin{split}\frac{H^*}{\overline{H}}\geq \frac{
\mathcal{V}\left(\frac{\hat{t}}{t}
-\frac{\lambda}{1-\lambda}\right)\mathcal{V}\left(1-\frac{\lambda(1+\kappa)}{(1-\lambda)\kappa}\right)}
{(1+\kappa)\mathcal{V}\left(\frac{\hat{t}}{t}\right)}-\left[\frac{8\sqrt{c_0\tau
d}}{\mathcal{V}\left(\frac{\hat{t}}{t}\right)}\right](\overline{H})^{-1/2}
-\left[\frac{2c_0\tau}{\mathcal{V}\left(\frac{\hat{t}}{t}\right)}\right](\overline{H})^{-1}
-C \frac{\log^2 n \log^3(1/\gamma)}{\sqrt{n}}.
\end{split}\end{equation*}
\end{proof}
\end{document}